\documentclass{article}

\usepackage{microtype}
\usepackage{graphicx}
\usepackage{subfigure}
\usepackage{booktabs} % for professional tables
\usepackage{enumitem}
\usepackage{wrapfig}

\usepackage{hyperref}

\usepackage[accepted]{icml2024}

\usepackage{amsmath}
\usepackage{amssymb}
\usepackage{mathtools}
\usepackage{amsthm}

\usepackage[capitalize,noabbrev]{cleveref}

\theoremstyle{plain}
\newtheorem{theorem}{Theorem}[section]
\newtheorem{proposition}[theorem]{Proposition}
\newtheorem{lemma}[theorem]{Lemma}
\newtheorem{corollary}[theorem]{Corollary}
\theoremstyle{definition}
\newtheorem{definition}[theorem]{Definition}

\newtheorem{remark}[theorem]{Remark}

\usepackage{amsmath,amsfonts,bm}

\def\eqref#1{equation~\ref{#1}}
\def\1{\bm{1}}

\def\vb{{\bm{b}}}

\def\vh{{\bm{h}}}

\def\vp{{\bm{p}}}

\def\vu{{\bm{u}}}
\def\vv{{\bm{v}}}

\def\vx{{\bm{x}}}

\def\vz{{\bm{z}}}

\def\mA{{\bm{A}}}

\def\mD{{\bm{D}}}

\def\mH{{\bm{H}}}
\def\mI{{\bm{I}}}
\def\mJ{{\bm{J}}}

\def\mL{{\bm{L}}}
\def\mM{{\bm{M}}}

\def\mO{{\bm{O}}}
\def\mP{{\bm{P}}}

\def\mS{{\bm{S}}}
\def\mT{{\bm{T}}}

\def\mW{{\bm{W}}}

\DeclareMathAlphabet{\mathsfit}{\encodingdefault}{\sfdefault}{m}{sl}
\SetMathAlphabet{\mathsfit}{bold}{\encodingdefault}{\sfdefault}{bx}{n}
\newcommand{\tens}[1]{\bm{\mathsfit{#1}}}

\def\tP{{\tens{P}}}

\def\tX{{\tens{X}}}

\def\gG{{\mathcal{G}}}

\def\gM{{\mathcal{M}}}
\def\gN{{\mathcal{N}}}

\def\gP{{\mathcal{P}}}

\newcommand*{\ldblbrace}{\{\mskip-5mu\{}
\newcommand*{\rdblbrace}{\}\mskip-5mu\}}
\newcommand*{\tr}{\mathsf{tr}}
\newcommand*{\hash}{\mathsf{hash}}
\newcommand*{\atp}{\mathsf{atp}}
\newcommand*{\diag}{\mathsf{diag}}
\newcommand*{\twist}{\mathsf{twist}}
\newcommand*{\meta}{\mathsf{Meta}}

\usepackage[textsize=tiny]{todonotes}

\icmltitlerunning{On the Expressive Power of Spectral Invariant Graph Neural Networks}

\begin{document}

\twocolumn[
% \icmltitle{Enhancing and Analyzing the Expressivity of Spectrally-Enhanced GNNs with Eigenspace Projections}
\icmltitle{On the Expressive Power of Spectral Invariant Graph Neural Networks}




\begin{icmlauthorlist}
\icmlauthor{Bohang Zhang}{pku}
\icmlauthor{Lingxiao Zhao}{cmu}
\icmlauthor{Haggai Maron}{technion,nvidia}
\end{icmlauthorlist}

\icmlaffiliation{pku}{Peking University}
\icmlaffiliation{cmu}{Carnegie Mellon University}
\icmlaffiliation{technion}{Technion}
\icmlaffiliation{nvidia}{NVIDIA Research}

\icmlcorrespondingauthor{Bohang Zhang}{zhangbohang@pku.edu.cn}
\icmlcorrespondingauthor{Haggai Maron}{hmaron@nvidia.com}

\icmlkeywords{Machine Learning, ICML}

\vskip 0.3in
]



\printAffiliationsAndNotice{}  % leave blank if no need to mention equal contribution

\begin{abstract}
Incorporating spectral information to enhance Graph Neural Networks (GNNs) has shown promising results but raises a fundamental challenge due to the inherent ambiguity of eigenvectors. Various architectures have been proposed to address this ambiguity, referred to as spectral invariant architectures. Notable examples include GNNs and Graph Transformers that use spectral distances, spectral projection matrices, or other invariant spectral features.
However, the potential expressive power of these spectral invariant architectures remains largely unclear. The goal of this work is to gain a deep theoretical understanding of the expressive power obtainable when using spectral features. We first introduce a unified message-passing framework for designing spectral invariant GNNs, called Eigenspace Projection GNN (EPNN). A comprehensive analysis shows that EPNN essentially unifies all prior spectral invariant architectures, in that they are either strictly less expressive or equivalent to EPNN. A fine-grained expressiveness hierarchy among different architectures is also established. On the other hand, we prove that EPNN itself is bounded by a recently proposed class of Subgraph GNNs, implying that all these spectral invariant architectures are strictly less expressive than 3-WL. Finally, we discuss whether using spectral features can gain additional expressiveness when combined with more expressive GNNs.
\end{abstract}

\begin{figure*}[t]
    \centering
    \vspace{-5pt}
    \includegraphics[width=0.88\textwidth]{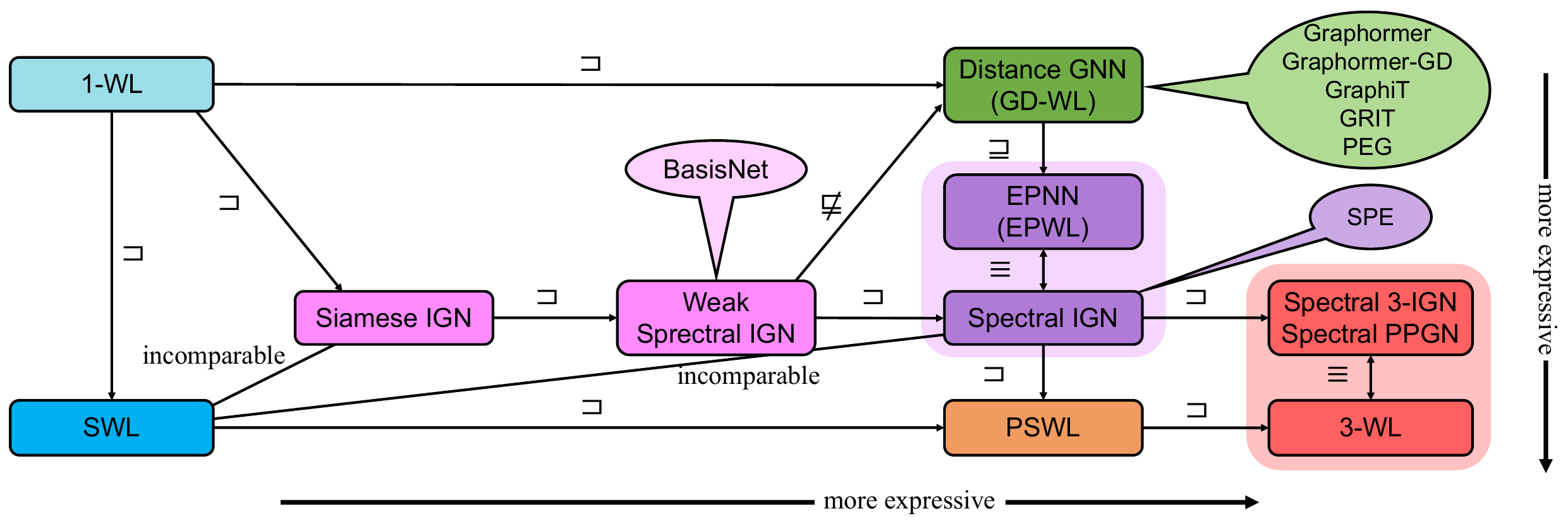}
    \vspace{-8pt}
    \caption{Expressive hierarchy for all GNN architectures studied in this paper. Here, the symbol ``$\equiv$'' means that the two GNNs being compared have the same expressive power; ``$\sqsupset$'' means that the latter GNN is \emph{strictly} more expressive than the former one; ``$\sqsupseteq$'' means that the latter GNN is either strict more expressive than or as expressive as the former one; ``$\not\sqsubseteq$'' means that the latter GNN is (strictly) not less expressive than the former one. Finally, ``incomparable'' means that either GNN is (strictly) not more expressive than the other. The dialog bubbles list literature architectures that can be seen as instantiations of the corresponding GNN class.}
    \label{fig:hierarchy}
    \vspace{-5pt}
\end{figure*}

\section{Introduction}

Recent works have demonstrated the promise of using spectral graph features, particularly the eigenvalues and eigenvectors of the graph Laplacian or functions thereof, as positional and structural encodings in Graph Neural Networks (GNNs) and Graph Transformers (GTs) \cite{dwivedi2020benchmarkgnns,dwivedi2020generalization,kreuzer2021rethinking, rampasek2022GPS,kim2022pure}. These spectral features encapsulate valuable information about graph connectivity, inter-node distances, node clustering patterns, and more. When using eigenvectors as inputs for machine learning models, a major challenge arises due to the inherent eigenespace symmetry \citep{lim2023sign} --- eigenvectors are not unique. Specifically, for any eigenvector $\vv$, $-\vv$ is also a valid eigenvector. The ambiguity becomes worse in the case of repeated eigenvalues; here, any orthogonal transformation of the basis vectors in a particular eigenspace yields alternative but equivalent input representations. 
%
% However, A main challenge when using these eigenvectors as input to machine learning models stems from the eigenespace symmetry \cite{lim2023sign} —  eigenvectors are not unique. Specifically, for any eigenvector $\vv$, $-\vv$ is also a valid eigenvector.
% In the simple spectrum case, where the eigenvalues are distinct, for any eigenvector $\vv$, $-\vv$ is also a valid eigenvector.
%This inherent non-uniqueness presents difficulties as there is a large variability in the representation of the same eigenbasis. Hence, accounting for eigenvector ambiguities is key to effectively leveraging spectral graph features \cite{lim2023sign,lim2023expressive}.

To address the ambiguity problem, a major line of recent works leverages \emph{invariant} features derived from eigenvectors and eigenvalues to design spectral invariant architectures. Popular choices for such features include eigenspace projection matrices\footnote{See \cref{sec:preliminary} for a formal definition of projection matrices.} \citep{lim2023sign,huang2024stability}, spectral node distances (e.g., those associated with random walks or graph diffusion) \citep{li2020distance,zhang2023rethinking,feldman2023weisfeiler}, or other invariant spectral characteristics \cite{wang2022equivariant}. All of these features can be easily integrated into GNNs to enhance edge features or function as relative positional encoding of GTs. However, on the theoretical side, while the expressive power of GNNs has been studied extensively \cite{xu2019powerful, maron2019provably, morris2021weisfeiler,geerts2022expressiveness, zhang2024beyond}, there remains little understanding of the important category represented by spectral invariant GNNs/GTs.

% Recent works have targeted the ambiguity problem mentioned above \cite{lim2023sign, wang2022equivariant, huang2024stability, ma2023graph}. The majority of these methods employ basis invariant features computed from the eigenvectors and eigenvalues. Popular choices for such features include eigenspace projection matrices \cite{lim2023sign}\footnote{for $G$ with n vetrices and a matrix $V\in \mathbb{R}^{n\times d}$ holding a basis for an eigenspace, these matrices are defined as $VV^T\in \mathbb{R}^{n\times n}$ and can be seen as edge features on a fully connected graph.} or spectral node distances \cite{wang2022equivariant}, such as commute time distance or PageRank distance, that can be readily used as edge features in GNNs and GTs. Unfortunately, while the expressive power of GNNs has been studied extensively \cite{xu2019powerful, maron2019provably, morris2021weisfeiler, zhang2024beyond}, there is little understanding of the expressive power of basis invariant GNNs and GTs. 

%\haggai{say what we know theoretically about bassinet?}
%\haggai{shall we discuss \cite{ma2023laplacian}?} 
%

\textbf{Current work.} The goal of this work is to gain deep insights into the expressive power of spectral invariant architectures and establish a complete expressiveness hierarchy. We begin by presenting \emph{Eigenspace Projection GNN} (EPNN), a novel GNN framework that unifies the study of all the aforementioned spectral invariant methods. EPNN is very simple: it encodes all spectral information for a node pair $(u,v)$ as a \emph{set} containing the values of all projection matrices on that node pair, along with the associated eigenvalues. It then computes and refines node representations using the spectral information as edge features within a standard message-passing framework on a fully connected graph. 

Our first theoretical result establishes a tight expressiveness upper bound for EPNN, showing that it is strictly less expressive than an important class of Subgraph GNNs proposed in \citet{zhang2023complete}, called PSWL. This observation is intriguing for two reasons. First, it connects spectral invariant GNNs and GTs with the seemingly unrelated research direction of Subgraph GNNs \citep{cotta2021reconstruction, bevilacqua2022equivariant, frasca2022understanding, qian2022ordered, zhao2022stars} — a line of research studying expressive GNNs from a structural and permutation symmetry perspective. Second, combined with recent results \citep{frasca2022understanding,zhang2023complete}, it implies that EPNN is strictly bounded by 3-WL. As an implication, bounding previously proposed spectral invariant methods by EPNN would readily indicate that they are all strictly less expressive than 3-WL. 

We then explore how EPNNs are related to GNNs/GTs that employ \emph{spectral distances} as positional encoding \cite{ying2021transformers,mialon2021graphit,zhang2023rethinking,ma2023graph,wang2022equivariant,li2020distance}. We prove that under the general framework proposed in \citet{zhang2023rethinking}, all commonly used spectral distances give rise to models with an expressive power bounded by EPNNs. This highlights an inherent expressiveness limitation of distance-based approaches in the literature. Moreover, our analysis underscores the crucial role of message-passing in enhancing the expressive power of spectral features.

Our next step aims to draw connections between EPNNs and two important spectral invariant architectures that utilize projection matrices, known as Basisnet \cite{lim2023sign} and SPE \cite{huang2024stability}. This is achieved by a novel symmetry analysis for eigenspace projections, which yields a theoretically-inspired architecture called Spectral IGN. Surprisingly, we prove that Spectral IGN is \emph{as expressive as} EPNN. On the other hand, SPE and BasisNet can be easily upper bounded by either Spectral IGN or its weaker variant.

Finally, we discuss the potential of using spectral features to boost the expressive power of higher-order GNNs. We show using the projection matrices alone does not provide any additional expressive power advantage when combined with highly expressive GNNs such as PPGN and $k$-IGN \cite{maron2019invariant,maron2019provably}. Nevertheless, we propose a possible solution towards further expressiveness gains: we hypothesize that stronger expressivity could be achieved through higher-order extensions of graph spectra, such as projection tensors.
Overall, our theoretical results characterize an expressiveness hierarchy across basis invariant GNNs, distance-based GNNs, GTs with spectral encoding, subgraph-based GNNs, and higher-order GNNs. The resulting hierarchy is illustrated in \cref{fig:hierarchy}.

\section{Related GNN Models}

\subsection{Spectrally-enhanced GNNs}
\label{sec:related_work_spectral_gnn}
In recent years, a multitude of research has emerged to develop spectrally-enhanced GNNs/GTs, integrating graph spectral information into either GNN node features or the subsequent message-passing process. These endeavors can be categorized into the following three groups.
% Depending on the specific implementations, these endeavors can be categorized into the following three groups.

\textbf{Laplacian eigenvectors as absolute positional encoding.} One way to design spectrally-enhanced GNNs involves encoding Laplacian eigenvectors. This approach treats each eigenvector as a 1-dimensional node feature and incorporates the top $k$ eigenvectors as a type of absolute positional encoding, which can be used to enhance any message-passing GNNs and GTs \citep{dwivedi2020benchmarkgnns,dwivedi2020generalization,kreuzer2021rethinking, rampasek2022GPS,maskey2022generalized,dwivedi2022graph,kim2022pure}. However, one main drawback of using Laplacian eigenvectors arises from the \emph{ambiguity} problem. 
% Specifically, for any eigenvector $\vv$, $-\vv$ is also an eigenvector. The situation becomes even worse if an eigenvalue has higher multiplicity, as infinitely many unit-norm eigenvectors can be chosen.
Such ambiguity creates severe issues regarding training instability and poor generalization \citep{wang2022equivariant}. While this problem can be partially mitigated through techniques like randomly flipping eigenvector signs or employing a canonization method \citep{ma2023laplacian}, it becomes much more complicated when eigenvalues have higher multiplicities.

\textbf{Spectral invariant architectures.} A better approach would be to design GNNs that are \emph{invariant} w.r.t. the choice of eigenvectors. For example, SignNet \citep{lim2023sign} transforms each eigenvector $\vv$ to $\phi(\vv)+\phi(-\vv)$ for some permutation equivariant function $\phi$, which guarantees the invariance when all eigenvalues have a multiplicity of 1. In case of higher multiplicity, BasisNet \citep{lim2023sign} achieves spectral invariance for the first time by utilizing the projection matrix. Specifically, given an eigenvalue $\lambda$ with multiplicity $k$, the projection matrix defined as $\sum_{i=1}^k \vv_i\vv_i^\top$ is invariant w.r.t. the choice of (unit) eigenvectors $\vv_1,\cdots,\vv_k$ as long as they form an orthogonal basis of the eigenspace associated with $\lambda$. Therefore, BasisNet simply feeds the projection matrix into a permutation equivariant model $\rho:\mathbb R^{n\times n}\to\mathbb R^n$ (e.g., 2-IGN \citep{maron2019invariant}) to generate spectral invariant node features $\rho(\sum_{i=1}^k \vv_i\vv_i^\top)$. The node features generated for different eigenspaces are concatenated together. While the authors proved that BasisNet can universally represent any graph functions when $\rho$ is universal (e.g., using $n$-IGN), the empirical performance is generally unsatisfactory when employing a practical model $\rho$ (i.e., 2-IGN). Recently, \citet{huang2024stability} further generalized BasisNet by proposing SPE, which performs a soft aggregation across different eigenspaces rather than a hard separation implemented in BasisNet. Specifically, let $\vv_1,\cdots,\vv_n$ be an orthogonal basis of (unit) eigenvectors associated with eigenvalues $\lambda_1,\cdots,\lambda_n$, respectively; then, each 1-dimensional node feature generated by SPE has the form $\rho(\sum_{i=1}^n \psi_j(\lambda_i)\vv_i \vv_i^\top)$, where $\psi_j:\mathbb R\to\mathbb R$ is a parameterized function associated with feature dimension $j$. The authors demonstrated that SPE can enhance the stability and generalization of GNNs, yielding much better empirical performance compared with BasisNet. 
% \haggai{add theoretical results on basisnet with 2-IGN from the signnet paper?}

% \haggai{Q: can we say something regarding the expressive power compared to signnet?}

\looseness=-1 \textbf{Spectral distances as invariant relative positional encoding.} In contrast to encoding Laplacian eigenvectors, an alternative approach to achieving spectral invariance involves utilizing (spectral) distances. Previous studies have identified various distances, spanning from the basic shortest path distance \citep{feng2022powerful,abboud2022shortest} to more advanced ones such as PageRank distance, resistance distance, and distances associated with random walks and graph diffusion \citep{li2020distance,zhang2023rethinking,mialon2021graphit,feldman2023weisfeiler}. Notably, all of these distances have a deep relation to the graph Laplacian while being more interpretable than eigenvectors and not suffering from ambiguity problems. The work of PEG \citep{wang2022equivariant} designed an invariant \emph{relative} positional encoding based on Laplacian eigenvectors, which can also be treated as a distance between nodes. Distances can be easily encoded in GNN models by either serving as edge features in message-passing aggregations \citep{wang2022equivariant,velingker2022affinity} or as relative positional encoding in Graph Transformers \citep{ying2021transformers,zhang2023complete,ma2023graph}.

\subsection{Expressive GNNs}

% \paragraph{Learning on graphs with spectral positional encoding.}

% Using Laplacian eigenvectors as positional encoding can enhance GNNs and Graph Transformers \citep{dwivedi2020benchmarkgnns,dwivedi2020generalization,kreuzer2021rethinking, rampasek2022GPS,maskey2022generalized}
% %
% Early works suggested addressing the sign ambiguity by data augmentation. Several works have addressed the sign ambiguity of these positional encodings by suggesting sign invariant \cite{lim2023sign} or equivariant \cite{lim2023expressive} networks. These methods have shown improvement over the basic method of processing the positional encoding directly. Less research was done on the more general basis invariant case. The first basis invariant architecture is Basisnet \cite{lim2023sign}. Specifically, instead of using an eigenbasis $V\in \mathbb{R}^{n\times k}$ (with $n$ nodes and $k$ eigenvectors), BasisNet takes the projection matrix $VV^T\in \mathbb{R}^{n\times n}$ for each eigenspace, which is invariant to orthogonal transformations of $V$. This matrix is processed as edge features in a $2$-IGN model \cite{maron2019invariant, maron2019provably}. However, this approach did not yield satisfactory empirical performance. \haggai{discuss the canonization  approach}
% %
% Another method for handling basis ambiguity uses spectral distances such as $d_{ij}=\|f_i-f_j\|$ (or functions thereof ) as input features \cite{wang2022equivariant}, where $f_i$ holds eigenvector values at node $i$. Experimentation with this approach was limited to link prediction tasks. \haggai{correct?}

The expressive power of GNNs has been studied in depth in the recent few years. Early works \cite{xu2019powerful,morris2019weisfeiler} have pointed out a fundamental limitation of GNNs by establishing an equivalence between message-passing neural networks and the 1-WL graph isomorphism test \citep{weisfeiler1968reduction}. To develop more expressive models, several studies leveraged high-dimensional variants of the WL test \citep{cai1992optimal,grohe2017descriptive}. Representative models include $k$-IGN \citep{maron2019invariant}, PPGN \citep{maron2019provably}, and $k$-GNN \citep{morris2019weisfeiler,morris2020weisfeiler}. However, these models suffer from severe computational costs and are generally not suitable in practice. Currently, one mainstream approach to designing simple, efficient, practical, and expressive architectures is the Subgraph GNNs \citep{cotta2021reconstruction,bevilacqua2022equivariant,bevilacqua2023efficient,you2021identity,zhang2021nested,zhao2022stars,kong2023mag}. In particular, the expressive power of Subgraph GNNs as well as their relation to the WL tests are well-understood in recent studies \citep{frasca2022understanding,qian2022ordered,zhang2023complete,zhang2024beyond}. These results will be used to analyze spectrally-enhanced GNNs in this paper.

\subsection{Expressive power of spectral invariant GNNs}

While spectrally-enhanced GNNs have been extensively studied in the literature, much less is known about their expressive power. \citet{balcilar2021analyzing,wang2022powerful} delved into the expressive power of specific spectral filtering GNNs, but their expressive power is inherently limited by 1-WL. Another line of works studied the expressive power of the \emph{raw} spectral invariants (e.g., projection matrices) in relation to the Weisfeiler-Lehman algorithms \cite{furer1995graph,furer2010power,rattan2023weisfeiler}. However, their analysis does not consider any aggregation or refinement procedures over spectral invariants, and thus, it does not provide explicit insights into the expressive power of the corresponding GNNs. \citet{lim2023sign} proposed a concrete spectral invariant GNN called BasisNet, but their expressiveness analysis still largely focuses on \emph{raw} eigenvectors and projection matrices. To our knowledge, none of the prior works addresses the crucial problem of whether/how the design of \emph{GNN layers} contributes to the model's expressiveness. In this paper, we will answer this question by showing that $(\mathrm{i})$ a suitable aggregation procedure can strictly improve the expressive power beyond raw spectral features, and $(\mathrm{ii})$ different aggregation schemes can lead to considerable variations in the models' expressiveness.
% We will thoroughly study this problem in our work.
% \bohang{Is this paragraph correct?}

% Spectral information have also been the boost the power of GNNs beyong the 1-WL test \citep{,}.

% We refer the interested reader to a recent survey \cite{morris2021weisfeiler}. The expressive power of GNNs that use spectral features is less explored. elaborate on . 

\section{Preliminaries}
\label{sec:preliminary}

\looseness=-1 We use $\{\ \}$ and $\ldblbrace\ \rdblbrace$ to denote sets and multisets, respectively. Given a (multi)set $S$, its cardinality is denoted as $|S|$. In this paper, we consider finite, undirected, simple graphs with no isolated vertices. Let $G=(V_G,E_G)$ be a graph with vertex set $V_G$ and edge set $E_G$, where each edge in $E_G$ is represented as a set $\{u,v\}\subset V_G$ of cardinality two. The \emph{neighbors} of a vertex $u\in V_G$ is denoted as $N_G(u)=\{v\in V_G:\{u,v\}\in E_G\}$, and the degree of $u$ is denoted as $\deg_G(u)=|N_G(u)|$. Given vertex pair $(u,v)\in V_G^2$, denote by $\atp_G(u,v)$ its atomic type, which encodes whether $u=v$, $\{u,v\}\in E_G$, or $u$ and $v$ are not adjacent. Given vertex tuple $\vu\in V_G^k$, the rooted graph $G^\vu$ is a graph obtained from $G$ by marking vertices $u_1,\cdots,u_k$ sequentially. We denote by $\gG$ the set of all graphs and by $\gG_k$ the set of all rooted graphs marking $k$ vertices. It follows that $\gG_0=\gG$.

\textbf{Graph invariant.} Two (rooted) graphs $G^\vu,H^\vv\in \gG_k$ are called \emph{isomorphic} (denoted by $G^\vu\simeq H^\vv$) if there is a bijection $f:V_G\to V_H$ such that $f(u_i)=v_i$ for all $i\in [k]$, and for all vertices $w_1,w_2\in V_G$, $\{w_1,w_2\}\in E_G$ iff $\{f(w_1),f(w_2)\}\in E_H$. A function $f$ defined on graphs $\gG_k$ is called a graph invariant if it is invariant under isomorphism, i.e., $f(G^\vu)=f(H^\vv)$ if $G^\vu\simeq H^\vv$. In the context of graph learning, any GNN that outputs a graph representation should be a graph invariant over $\gG_0$; similarly, any GNN that outputs a representation for each node/each pair of nodes should be a graph invariant over $\gG_1$/$\gG_2$, respectively.

\textbf{Graph vectors and matrices.} Any real-valued graph invariant $x$ defined over $\gG_1$ corresponds to a graph vector $\vx_G: V_G\to \mathbb R$ when restricting on a specific graph $G\in \gG$. Without ambiguity, we denote the elements in $\vx_G$ as $\vx_G(u)$ for each $u\in V_G$, which is equal to $x(G^u)$. Similarly, any real-valued graph invariant $M$ defined over $\gG_2$ corresponds to a graph matrix $\mM_G: V_G^2\to \mathbb R$ when restricting on $G\in \gG$, where element $\mM_G(u,v)$ equals to $M(G^{uv})$. For ease of reading, we will drop the subscript $G$ when there is no ambiguity of the graph used in context. One can generalize all basic linear algebras from classic vectors/matrices to those defined on graphs. For example, the matrix product is defined as $(\mM_1\mM_2)(u,v)=\sum_{w\in V_G}\mM_1(u,w)\mM_2(w,v)$. Several basic graph matrices include the adjacency matrix $\mA$, degree matrix $\mD$, Laplacian matrix $\mL:=\mD-\mA$, and normalized Laplacian matrix $\hat \mL:=\mD^{-1/2}\mL\mD^{-1/2}$. Note that all these matrices are symmetric.

\textbf{Graph spectra and projection.} Let $\mM$ be any symmetric graph matrix (e.g., $\mA$, $\mL$, or $\hat \mL$). The graph spectrum is the set of all eigenvalues of $\mM$, which is a graph invariant over $\gG_0$. In addition to eigenvalues, the spectral information of a graph also includes eigenvectors or eigenspaces, which contain much more fine-grained information. Unfortunately,
% since an eigenvector associated with the same eigenvalue is not unique, 
eigenvectors have inherent ambiguity and cannot serve as a valid graph invariant over $\gG_1$. Instead, we focus on the eigenspaces characterized by their projection matrices. Concretely, there is a \emph{unique} projection matrix $\mP_i$ for each eigenvalue $\lambda_i$, which can be obtained via the eigen-decomposition $\mM=\sum_{i\in[m]}\lambda_i\mP_i$, where $m$ is the number of different eigenvalues. It follows that these projection matrices are symmetric, idempotent ($\mP_i^2=\mP_i$), ``orthogonal'' ($\mP_i\mP_j=\mO$ for all $i\neq j$), and sum to identity ($\sum_{i\in[m]}\mP_i=\mI$). There is a close relation between projection matrix $\mP_i$ and any orthogonal basis of unit eigenvectors $\{\vz_{i,1},\cdots,\vz_{i,J_i}\}$ that spans the eigenspace associated with $\lambda_i$: specifically, $\mP_i=\sum_{j=1}^{J_i}\vz_{i,j}\vz_{i,j}^\top$. The projection matrices naturally define a graph invariant $\gP^\mM$ over $\gG_2$:
\begin{equation*}
    \gP^\mM_G(u,v):=\ldblbrace (\lambda_1,\mP_1(u,v)),\cdots,(\lambda_m,\mP_m(u,v))\rdblbrace.
\end{equation*}
We call $\gP^\mM$ the eigenspace projection invariant (associated with graph matrix $\mM$).
%\haggai{is there a good reason to use the graph invariants terms vs. equivaraince? I find it a bit more complicated. We can keep it as is I don't mind just want to understand the benefit}

\section{Eigenspace Projection Network}

This section introduces a simple GNN design paradigm based on the invariant $\gP^\mM$ defined above, called Eigenspace Projection GNN (EPNN). The idea of EPNN is very simple:  $\gP_G^\mM(u,v)$ essentially encodes the relation between vertices $u$ and $v$ in graph $G$ and can thus be treated as a form of ``edge feature''. In light of this, one can naturally define a message-passing GNN that updates the node representation of each vertex $u$ by iteratively aggregating the representations of other nodes $v$ along with edge features associated with $(u,v)$. Formally, consider a $K$-layer EPNN and denote by $\vh^{(l)}_G(u)$ the node representation of $u\in V_G$ computed by an EPNN after the $l$-th layer. Then, we can write the update rule of each EPNN layer as follows:
\begin{equation}
\label{eq:epnn_aggregation}
\begin{aligned}
    &\vh^{(l+1)}_G(u)=g^{(l+1)}(\vh^{(l)}_G(u),\\
    &\qquad\qquad\ldblbrace (\vh^{(l)}_G(v),\gP^\mM_G(u,v)):v\in V_G\rdblbrace),
\end{aligned}
% \vh^{\!(l+1)}_G(u)\!=\!g^{\!(l+1)}(\vh^{\!(l)}_G(u),\ldblbrace (\vh^{\!(l)}_G(v),\gP^\mM_G\!(u,v))\!:\!v\!\in\! V_G\rdblbrace),
\end{equation}
where all node representations $\vh_G^{(0)}(u)$ are the same at initialization. Here, $g^{(l+1)}$ can be any parameterized function representing the $(l+1)$-th layer. In practice, it can be implemented in various ways such as GIN-based aggregation \citep{xu2019powerful} or Graph Transformers \citep{ying2021transformers}, and we highlight that it is particularly suited for Graph Transformers as the graph becomes fully connected with this edge feature. Finally, after the $K$-th layer, a global pooling is performed over all vertices in the graph to obtain the graph representation $\mathsf{POOL}(\ldblbrace\vh^{(K)}_G(u):u\in V_G\rdblbrace)$.

EPNN is well-defined. First, the graph representation computed by an EPNN is permutation invariant w.r.t.  vertices, as $\gP^\mM$ is a graph invariant over $\gG_2$. Second, EPNN does not suffer from the eigenvector ambiguity problem, as $\gP_G^\mM$ is uniquely determined by graph $G$. Later, we will show that EPNN can serve as a simple yet unified framework for studying the expressive power of spectral invariant GNNs. 

% \textbf{Relations to Laplacian positional encoding.} While this paper mainly focuses on spectral \emph{invariant} GNNs, a brief discussion can be made regarding how the node features computed by an EPNN can recover non-spectral-invariant positional encoding such as the Laplacian eigenvectors (see \cref{sec:related_work_spectral_gnn}). Consider a simple setting where all eigenvalues $\lambda_1,\cdots, \lambda_m$ are different. In this case, given vertex $u\in V_G$, $\gP^\mM_G(u,u)$ can be equivalently written as $\ldblbrace (\lambda_1,|z_1(u)|^2),\cdots, (\lambda_m,|z_m(u)|^2)\rdblbrace$, where $\vz_i$ is any (unit) eigenvector associated with eigenvalue $\lambda_i$. We show in \cref{thm:epwl_first_layer_encode_Puu} that the node representation $\vh_G^{(1)}(u)$ after the first layer of an EPNN can encode $\gP^\mM_G(u,u)$, by which it captures a sign-invariant version of eigenvector information at $u$ (i.e., $(|z_1(u)|,\cdots,|z_m(u)|)$). This shows that node features computed by EPNN can capture a sign-invariant version of Laplacian positional encoding.

% Below, we will address this question from two perspectives. On one hand, we will establish a tight expressiveness bound on EPNN in terms of distinguishing non-isomorphic graphs by using the standard Weisfeiler-Lehman hierarchy. On the other hand, we will show in \cref{sec:distance_gnn,sec:spectral_ign} that EPNN essentially unifies a wide range of spectrally-enhanced GNNs raised in prior works, achieving the maximal expressive power among them.

\subsection{Expressive power of EPNN}
\label{sec:expressvieness_bound}

The central question we would like to study is: \emph{what is the expressive power of EPNN?} 
To formally study this question, this subsection introduces EPWL (Eigenspace Projection Weisfeiler-Lehman), an abstract color refinement algorithm tailored specifically for graph isomorphism test. Compared with \cref{eq:epnn_aggregation}, in EPWL the node representation $\vh^{(l+1)}_G(u)$ is replaced by a color $\chi^{(l+1)}_G(u)$, and the aggregation function $g^{(l+1)}$ is replaced by a perfect \emph{hash function} $\hash$, as presented below:
\begin{equation}
\label{eq:epwl}
    \chi^{(l+1)}_G(u)\!=\!\hash(\chi^{(l)}_G(u),\ldblbrace (\chi^{(l)}_G(v),\gP^\mM_G\!(u,v))\!:\!v\!\in\! V_G\rdblbrace).
\end{equation}
% $\stackrel{\text{EP}}{\sim}$ 
Initially, all node colors $\chi^{(0)}_G(u)$ are the same. For each iteration $l$, the color mapping $\chi_G^{(l)}$ induces an equivalence relation over vertex set $V_G$, and the relation gets \emph{refined} with the increase of $l$. Therefore, with a sufficiently large number of iterations $l\le |V_G|$, the relations get \emph{stable}. The graph representation is then defined to be the multiset of stable colors. EPWL distinguishes two non-isomorphic graphs iff the computed graph representations are different. We have the following result (which can be easily proved following standard techniques, see e.g., \citet{zhang2023complete}):
\begin{proposition}
    The expressive power of EPNN is bounded by EPWL in terms of graph isomorphism test. Moreover, with sufficient layers and proper functions $g^{(l)}$, EPNN can be as expressive as EPWL.
\end{proposition}

In subsequent analysis, we will bound the expressive power of EPWL by building relations to the standard Weisfeiler-Lehman hierarchy. First, it is easy to see that EPWL is \emph{lower bounded} by the classic 1-WL defined below:
\begin{equation}
    \tilde\chi^{(l+1)}_G(u)\!=\!\hash(\tilde\chi^{(l)}_G(u),\ldblbrace (\tilde\chi^{(l)}_G(v),\atp_G(u,v))\!:\!v\!\in\! V_G\rdblbrace),
\end{equation}
where $\tilde\chi^{(l)}_G(u)$ is the 1-WL color of $u$ in graph $G$ after $l$ iterations. This result follows from the fact that the EPWL color mapping $\chi^{(l+1)}$ always induces a \emph{finer} relation than the 1-WL color mapping $\tilde\chi^{(l+1)}$, as $\atp_G(u,v)$ is fully encoded in $\gP_G^M(u,v)$ (see \cref{thm:epwl_atp}). Besides, it is easy to give 1-WL indistinguishable graphs that can be distinguished via spectral information (see \cref{fig:1wl_indistinguishable}). Putting these together, we arrive at the following conclusion:
\begin{proposition}
\label{thm:epwl_1wl}
    For any graph matrix $\mM\in\{\mA,\mL,\hat\mL\}$, the corresponding EPWL is strictly more expressive than 1-WL in distinguishing non-isomorphic graphs.
\end{proposition}

On the other hand, the question becomes more intriguing when studying the \emph{upper bound} of EPWL. Below, we will approach the problem by building fundamental connections between EPWL and an important class of expressive GNNs known as Subgraph GNNs. The basic form of Subgraph GNN is very simple: given a graph $G$, it treats $G$ as a set of rooted graphs $\ldblbrace G^u:u\in V_G\rdblbrace$ (known as node marking), independently runs 1-WL for each $G^u$, and finally merges their graph representations. We call the above algorithm SWL, and the refinement rule can be formally written as
\begin{equation}
\begin{aligned}
    &\chi^{\mathsf{S},(l+1)}_G(u,v)=\hash(\chi^{\mathsf{S},(l)}_G(u,v),\\
    &\qquad\ldblbrace (\chi^{\mathsf{S},(l)}_G(u,w),\atp_G(v,w)):w\in V_G\rdblbrace).
\end{aligned}
\end{equation}
where $\chi^{\mathsf{S},(l)}_G(u,v)$ is the SWL color of vertex $v$ in graph $G^u$ after $l$ iterations, and the initial color $\chi^{\mathsf{S},(0)}_G(u,v)=\mathbb I[u=v]$ distinguishes the marked vertex $u$ in $G^u$. Recently, \citet{zhang2023complete,frasca2022understanding} significantly generalized Subgraph GNNs by enabling interactions among different subgraphs and built a complete design space. Among them, \citet{zhang2023complete} proposed the PSWL algorithm, which adds a cross-graph aggregation to SWL as shown below:
\begin{equation}
\begin{aligned}
    &\chi^{\mathsf{PS},(l+1)}_G(u,v)=\hash(\chi^{\mathsf{PS},(l)}_G(u,v),\chi^{\mathsf{PS},(l)}_G(v,v),\\
    &\qquad\ldblbrace (\chi^{\mathsf{PS},(l)}_G(u,w),\atp_G(v,w)):w\in V_G\rdblbrace),
\end{aligned}
\end{equation}
where $\chi^{\mathsf{PS},(l)}_G(u,v)$ is the PSWL color of $(u,v)\in V_G^2$ after $l$ iterations. We now present our main result, which reveals a fundamental connection between PSWL and EPWL:

\begin{theorem}
\label{thm:epwl_pswl}
    For any graph matrix $\mM\in\{\mA,\mL,\hat\mL\}$, the expressive power of EPWL is strictly bounded by PSWL in distinguishing non-isomorphic graphs.
\end{theorem}
The proof of \cref{thm:epwl_pswl} is deferred to \cref{sec:proof_epwl_sswl}, which is based on the recent graph theory result established by \citet{rattan2023weisfeiler}. Specifically, given any graph $G$ and vertices $u,v\in V_G$, each projection element $\mP_i(u,v)$ in $\gP^\mM$ is determined by the SWL stable color $\chi_G^\mathsf{S}(u,v)$ for any symmetric ``equitable'' matrix $\mM$ defined in \citet{rattan2023weisfeiler}, and the eigenvalues are also determined by the SWL graph representation. Notably, all matrices studied in this paper (e.g., $\mA,\mL,\hat\mL$) are equitable. Based on this result, one may guess that EPWL can be bounded by SWL. However, we show this is actually not the case when further taking the message-passing aggregation into account. The key technical contribution in our proof is to relate the refinement procedure in EPWL to the additional cross-graph aggregation $\chi^{\mathsf{PS},(l)}_G(v,v)$ in PSWL. To this end, we show the stable color $\chi^{\mathsf{PS}}_G(u,u)$ is strictly finer than the stable color $\chi_G(u)$, thus concluding the proof.

\begin{remark}
\label{remark:epwl_pswl}
    Based on \cref{thm:epwl_pswl}, one can also bound the expressiveness of EPWL by other popular GNNs in literature, such as SSWL \citep{zhang2023complete}, Local 2-GNN \citep{morris2020weisfeiler,zhang2024beyond}, $\mathsf{ReIGN(2)}$ \citep{frasca2022understanding}, ESAN \citep{bevilacqua2022equivariant}, and GNN-AK \citep{zhao2022stars}, as all these architectures are more expressive than PSWL \citep{zhang2023complete}. However, EPWL is incomparable to the vanilla SWL, where we give counterexamples in \cref{sec:proof_counterexample}.
\end{remark}

The significance of \cref{thm:epwl_pswl} is twofold. First, it reveals a surprising relation between GNNs augmented with spectral information and the ones grounded in structural message-passing, which represents two seemingly unrelated research directions. Our result thus offers insights into how previously proposed expressive GNNs can encode spectral information. Second, \cref{thm:epwl_pswl} points out a fundamental limitation of EPNN. Indeed, combined with the result that PSWL is strictly bounded by 3-WL \citep{zhang2023complete,zhang2024beyond}, we obtain the concluding corollary:

\begin{corollary}
\label{thm:epwl_3wl}
    For any graph matrix $\mM\in\{\mA,\mL,\hat\mL\}$, the expressive power of EPWL is strictly bounded by 3-WL.
\end{corollary}

\section{Distance GNNs and Graph Transformers}
\label{sec:distance_gnn}

In this section, we will show that EPNN unifies all distance-based spectral invariant GNNs. Here, we adopt the framework proposed in \citet{zhang2023rethinking}, known as Generalized Distance (GD) GNNs. The aggregation formula of the corresponding GD-WL can be written as follows:
\begin{equation}
\begin{aligned}
    \chi^{\mathsf{D},(l+1)}_G(u)=&\hash(\chi^{\mathsf{D},(l)}_G(u),\\
    &\ldblbrace (\chi^{\mathsf{D},(l)}_G(v),d_G(u,v)):v\in V_G\rdblbrace).
\end{aligned}
\end{equation}
where $\chi^{\mathsf{D},(l)}_G(u)$ is the GD-WL color of vertex $u\in V_G$ after $l$ iterations, and $d$ can be any valid distance metric. By choosing different distances, GD-WL incorporates various prior works listed below.
\begin{itemize}[topsep=0pt,leftmargin=20pt]
    \setlength{\itemsep}{0pt}
    \item \textbf{Shorest-path distance (SPD).} This is the most basic distance metric and has been extensively used in designing expressive GNNs \citep[e.g., ][]{li2020distance,ying2021transformers,abboud2022shortest,feng2022powerful}.
    \item \textbf{Resistance distance (RD).} It is defined to be the effective resistance between two nodes when treating the graph as an electrical network where each edge has a resistance of $1\Omega$. Recently, \citet{zhang2023rethinking} showed that incorporating RD can significantly improve the expressive power of GNNs for biconnectivity problems such as identifying cut vertices/edges. Besides, RD has been extensively studied in other areas in the GNN community, such as alleviating oversquashing problems \citep{arnaiz2022diffwire}.
    \item \textbf{Distances based on random walk.} The hitting-time distance (HTD) between two vertices $u$ and $v$ is defined as the expected number of steps in a random walk starting from $u$ and reaching $v$ for the first time, which is an asymmetric distance. Instead, the commute-time distance (CTD) is defined as the expected number of steps for a round-trip starting at $u$ to reach $v$ and then return to $u$, which is a symmetrized version of HTD. These distances are fundamental in graph theory and have been used to develop/understand expressive GNNs \citep{velingker2022affinity,zhang2023complete}.
    \item \textbf{PageRank distance (PRD).} Given weight sequence $\gamma_0,\gamma_1,\cdots$, the PageRank distance between vertices $u$ and $v$ is defined as $\sum_{k=0}^\infty \gamma_k\mW^k(u,v)$, where $\mW=\mD^{-1}\mA$ is the random walk probability matrix. It is a generalization of the $p$-step landing probability distance, which corresponds to setting $\gamma_p=1$ and $\gamma_k=0$ for all $k\neq p$. \citet{li2020distance} first proposed to use PRD-WL to boost the expressive power of GNNs.
    \item \textbf{Other distances.} We also study the (normalized) diffusion distance \citep{coifman2006diffusion} and the biharmonic distance \citep{lipman2010biharmonic}. Due to space limit, please refer to \cref{sec:proof_distance} for more details.
\end{itemize}

Our main result is stated as follows:
\begin{theorem}
\label{thm:distance}
    For any distance listed above, the expressive power of GD-WL is upper bounded by EPWL with the normalized graph Laplacian matrix $\hat\mL$.
\end{theorem}

The proof of \cref{thm:epwl_pswl} is highly technical and is deferred to \cref{sec:proof_distance}, with several important remarks made as follows. For the case of SPD, the proof is based on the key finding that $\gP^{\hat\mL}_G(u,v)$ determines the shortest path distance between $u$ and $v$ for any graph $G$ and vertices $u,v\in V_G$. Unfortunately, this property does not transfer to other distances listed above. For general distances, the reason why EPWL is still more expressive lies in the entire \emph{message-passing process} (or color refinement procedure). Concretely, the refinement continuously enriches the information embedded in node colors $\chi^{(l)}_G(v)$, so that the tuple $(\chi^{(l)}_G(u),\chi^{(l)}_G(v),\gP^{\hat\mL}_G(u,v))$ eventually encompasses sufficient information to determine any distance $d_G(u,v)$ (although $\gP^{\hat\mL}_G(u,v)$ alone may not determine it). Our proof thus emphasizes the critical role of message-passing aggregation in enhancing the expressiveness of spectral information. Note that this is also justified in the proof of \cref{thm:epwl_pswl}, where the message-passing process boosts the expressive power of EPWL beyond SWL. Moreover, we emphasize that these distances listed above cannot be well-encoded when using weaker message-passing aggregations, as will be elucidated in \cref{sec:basisnet}.

\textbf{Implications.} \cref{thm:distance} has a series of consequences. First, it implies that all the power of distance information is possessed by EPWL. As an example, we immediately have the following corollary based on the relation between distance and biconnectivity of a graph established in \citet{zhang2023rethinking}, which significantly extends the classic result that Laplacian spectrum encodes graph connectivity \cite{brouwer2011spectra}. 
\begin{corollary}
\label{thm:biconnectivity}
    EPWL is fully expressive for encoding graph biconnectivity properties, such as identifying cut vertices and cut edges, determining the number of biconnected components, and distinguishing graphs with non-isomorphic block cut-vertex trees and block cut-edge trees.
\end{corollary}

\begin{remark}
    \cref{thm:biconnectivity} offers a novel understanding of the work by \citet{zhang2023rethinking} on why ESAN can encode graph biconnectivity, thereby thoroughly unifying their analysis. Essentially, this is just because ESAN is more powerful than EPWL (\cref{remark:epwl_pswl}), and EPWL itself is already capable of encoding both SPD and RD.
\end{remark}

\begin{remark}
    Combining \cref{thm:epwl_pswl,thm:distance} resolves an open question posed by \citet{zhang2023complete}, confirming that PSWL can encode resistance distance.
\end{remark}

As a second implication, combined with \cref{thm:epwl_3wl}, \cref{thm:distance} highlights a fundamental limitation of all distance-based GNNs as stated below:

\begin{corollary}
    For any distance defined above, the expressive power of GD-WL is strictly bounded by 3-WL.
\end{corollary}

\textbf{Distances as positional encoding.} Distances have also found extensive application as positional encoding in GNNs and GTs. For instance, Graphormer \citep{ying2021transformers}, Graphormer-GD \citep{zhang2023rethinking}, and GraphiT \citep{mialon2021graphit} employ various distances as relative positional encoding in Transformer's attention layers. Similarly, GRIT \citep{ma2023graph} employs multi-dimensional PRD and a novel attention mechanism to further boost the performance of GTs. The positional encoding devised in PEG \citep{wang2022equivariant} can also be viewed as a function of a distance, with the architecture representing an instantiation of GD-WL. These architectures are analyzed in \cref{sec:proof_other_architectures}, where we have the following concluding corollary:
\begin{corollary}
    The expressive power of Graphormer, Graphormer-GD, GraphiT, GRIT, and PEG are all bounded by EPWL and strictly less expressive than 3-WL.
\end{corollary}

% Consequently, we conclude that all these architectures are bounded by EPWL in terms of expressive power and are inherently less expressive than 3-WL. 

\section{Spectral Invariant Graph Network}
\label{sec:spectral_ign}

To gain an in-depth understanding of the expressive power of EPNN, in this section we will switch our attention to a more principled perspective by studying how to model spectral invariant GNNs based on the \emph{symmetry} of projection matrices. Consider a graph $G$ with vertex set $V_G=\{1,\cdots,n\}$. We can group all spectral information of $G$ into a 4-dimensional tensor $\tP\in \mathbb R^{m\times n\times n\times 2}$, where $\tP_{i,u,v,1}=\lambda_i$ encodes the $i$-th eigenvalue and $\tP_{i,u,v,2}=\mP_i(u,v)$ encodes the $(u,v)$-th element of the projection matrix $\mP_i$ associated with $\lambda_i$. From this representation, one can easily figure out the symmetry group associated with $\tP$, which is the \emph{product group} $S_m\times S_n$ of two permutation groups $S_m$ and $S_n$ representing two independent symmetries --- eigenspace symmetry and graph symmetry. For any element $(\sigma,\tau)\in S_m\times S_n$, it acts on $\tP$ in the following way:
\begin{equation}
    [(\sigma,\tau)\cdot \tP]_{i,u,v,j}=\tP_{\sigma^{-1}(i),\tau^{-1}(u),\tau^{-1}(v),j}.
\end{equation}
We would like to design a GNN model $f$ that is invariant under $S_m\times S_n$, i.e., $f((\sigma,\tau)\cdot \tP)=f(\tP)$ for all $(\sigma,\tau)\in S_m\times S_n$. Interestingly, this setting precisely corresponds to a specific case of the DSS framework proposed in \citet{maron2020learning,bevilacqua2022equivariant}. This offers a theoretically inspired approach to designing powerful invariant architectures, as we will detail below.

\subsection{Spectral IGN}
\label{sec:spectral_ign_standard}

In the DSS framework, the neural network $f$ is formed by stacking a series of equivariant linear layers $L^{(i)}$ interleaved with elementwise non-linear activation $\phi$, culminating in the final pooling layer:
\begin{equation}
\label{eq:spectral_ign}
    f=M \circ \phi \circ g \circ\phi \circ L^{(K)} \circ \phi \circ \cdots \circ  \phi \circ L^{(1)},
\end{equation}
where each $L^{(l)}:\mathbb R^{m\times n\times n\times d_{l-1}}\to \mathbb R^{m\times n\times n\times d_l}$ is equivariant w.r.t. $S_m\times S_n$, i.e., for all $(\sigma,\tau)\in S_m\times S_n$,
$$L^{(l)}((\sigma,\tau)\cdot \tX)=(\sigma,\tau)\cdot L^{(l)}(\tX)\quad \forall \tX\in\mathbb R^{m\times n\times n\times d_{l-1}};$$
$g:\mathbb R^{m\times n\times n\times d_K}\to \mathbb R^{d_K}$ is an invariant pooling layer (e.g., average pooling or max pooling), and $M:\mathbb R^{d_{K}}\to \mathbb R^{d_{K+1}}$ is a multi-layer perceptron. Here, the key question lies in designing linear layers $L^{(l)}$ equivariant to the product group $S_m\times S_n$. \citet{maron2020learning} theoretically showed that $L^{(l)}$ can be decomposed in the following way: 
\begin{equation}
\label{eq:spectral_ign_layer}
    [L^{(l)}(\tX)]_{i}=\tilde L^{(l)}_1(\tX_{i})+\tilde L^{(l)}_2\left(\sum_{i\in[m]}\tX_{i}\right),
\end{equation}
where $\tilde L^{(l)}_1,\tilde L^{(l)}_2:\mathbb R^{n\times n\times d_{l-1}}\to \mathbb R^{n\times n\times d_{l}}$ are two linear functions equivariant to the graph symmetry modeled by $S_n$. This decomposition is significant as the design space of equivariant linear layers for graphs has been fully characterized in \citet{maron2019invariant}, known as 2-IGN. We thus call our model (\cref{eq:spectral_ign,eq:spectral_ign_layer}) Spectral IGN.

The central question we would like to study is: what is the expressive power of Spectral IGN? Surprisingly, we have the following main result:
% we find there is a close relation between Spectral IGN and EPNN, as stated in the main theorem below:
\begin{theorem}
\label{thm:spectral_ign}
    The expressive power of Spectral IGN is bounded by EPWL. Moreover, with sufficient layers and proper network parameters, Spectral IGN is as expressive as EPWL in distinguishing non-isomorphic graphs.
\end{theorem}

The proof of \cref{thm:spectral_ign} is given in \cref{sec:proof_spectral_ign}. It offers an interesting and alternative view for theoretically understanding the EPNN designing framework and justifying its expressiveness. Moreover, the connection to Spectral IGN allows us to bridge EPNN with important architectural variants as we will discuss in the next subsection.

\begin{remark}
\label{remark:pooling_decompose}
    We remark that there is a variant of Spectral IGN, where the pooling layer $g$ is decomposed into two pooling layers via $g=g^{(2)}\circ\phi\circ g^{(1)}$ with $g^{(1)}:\mathbb R^{m\times n\times n\times d_K}\to \mathbb R^{n\times n\times d_K}$ and $g^{(2)}:\mathbb R^{n\times n\times d_K}\to \mathbb R^{d_K}$ pooling layers for symmetry groups $S_m$ and $S_n$, respectively. This variant has the same expressive power, and \cref{thm:spectral_ign} still holds.
\end{remark}

\subsection{Siamese IGN}
\label{sec:basisnet}

Let us consider an interesting variant of Spectral IGN, dubbed Siamese IGN \citep{maron2020learning}, where the network processes each eigenspace $\tP_i$ ($i\in[m]$) \emph{independently} using a ``Siamese'' 2-IGN without aggregating over different eigenspaces. This corresponds to replacing \cref{eq:spectral_ign_layer} by $[L^{(l)}(\tX)]_i=\tilde L^{(l)}(\tX_i)$ with $\tilde L^{(l)}:\mathbb R^{n\times n\times d_{l-1}}\to \mathbb R^{n\times n\times d_{l}}$ a 2-IGN layer. Siamese IGN is interesting because it is a special type of Spectral IGN that is invariant to a \emph{strictly larger} group formed by the wreath product $S_m \wr S_n$ \citep[see][]{maron2020learning,wang2020equivariant}. Moreover, it is closely related to a prior architecture called BasisNet \citep[see \cref{sec:proof_basisnet} for a detailed description]{lim2023sign}, as BasisNet also processes each eigenspace independently without interaction. To elaborate on this connection, consider a variant of Siamese IGN, dubbed Weak Spectral IGN, which is obtained from Siamese IGN by decomposing the pooling layer according to \cref{remark:pooling_decompose}. Note that unlike Siamese IGN, Weak Spectral IGN is no longer invariant to group $S_m \wr S_n$. We have the following result:
\begin{proposition}
\label{thm:basisnet_basic}
    Let the top graph encoder used in BasisNet be any 1-layer message-passing GNN. Then, the expressive power of BasisNet is bounded by Weak Spectral IGN.
\end{proposition}

On the other hand, a more fundamental question lies in the relation between Siamese IGN, Weak Spectral IGN, and Spectral IGN. Our main result states that there are strict expressivity gaps between them:
\begin{theorem}
\label{thm:siamese_ign_spectral_ign}
    Siamese IGN is strictly less expressive than Weak Spectral IGN. Moreover, Weak Spectral IGN is strictly less expressive than Spectral IGN.
\end{theorem}

\textbf{Discussions with BasisNet.} $(\mathrm{i})$ Combined with previous results, one can prove that \emph{EPNN is strictly more expressive than BasisNet}\footnote{This holds for any message-passing-based top graph encoder with arbitrary layers by following the construction in \cref{thm:siamese_ign_spectral_ign}.}. This result is particularly striking, as EPNN only stores node representations while BasisNet stores a representation for each 3-tuple $(i,u,v)\in [m]\times V_G\times V_G$, whose memory complexity scales like $O(|V_G|^3)$. $(\mathrm{ii})$ Combined with \cref{thm:epwl_3wl}, we conclude that the expressive power of BasisNet is also strictly bounded by 3-WL.

\textbf{Discussions with SPE.} We next turn to the SPE architecture \citep{huang2024stability}. Surprisingly, while SPE was originally designed to improve the stability and generalization of spectral invariant GNNs (see \cref{sec:related_work_spectral_gnn}), we found the soft aggregation across different eigenspaces simultaneously enhances the network's \emph{expressive power}. Indeed, we have:

\begin{proposition}
\label{thm:spe}
    When 2-IGN is used to generate node features in SPE and the top graph encoder is a message-passing GNN, the expressive power of the whole SPE architecture is as expressive as Spectral IGN.
\end{proposition}
% \vspace{-4pt}

\cref{thm:spe} theoretically justifies the design of SPE. Combined with previous results, we conclude that SPE is strictly more expressive than BasisNet when using the same 2-IGN backbone, while being strictly bounded by 3-WL.

\textbf{Delving more into the gap.} We remark that the gap between Siamese IGN and Spectral IGN is not just theoretical; it also reveals significant limitations of the siamese design in practical aspects. Specifically, we identify that both Siamese IGN and Weak Spectral IGN \emph{cannot} fully encode any graph distance listed in \cref{sec:distance_gnn} (even the basic SPD), as stated in the following theorem:
\begin{theorem}
\label{thm:siamese_ign_distance}
    For any distance listed in \cref{sec:distance_gnn}, there exist two non-isomorphic graphs which GD-WL can distinguish but Weak Spectral IGN (applied to matrix $\hat\mL$) cannot.
\end{theorem}
% \vspace{-4pt}

On the other hand, we have proved that EPNN applied to matrix $\hat\mL$ is more powerful than GD-WL. This contrast reveals the crucial role of allowing interaction between eigenspaces for enhancing model's expressiveness.
% Tthus conclude that disabling interactions between eigenspaces (like BasisNet) substantially hurts the expressiveness.

\subsection{Extending to higher-order spectral invariant GNNs}
\label{sec:higher-order-gnn}

The DSS framework presented in \cref{sec:spectral_ign_standard} is quite general. In principle, any $S_n$-equivariant graph layer $E$ can be used to build a GNN model $f$ invariant to $S_m\times S_n$. This can be achieved by making $\tilde L^{(l)}_1,\tilde L^{(l)}_2$ in \cref{eq:spectral_ign_layer} two instantiations of $E$. In this subsection, we will study higher-order spectral invariant GNNs where the used graph encoders are beyond 2-IGN. We consider two standard settings for choosing highly expressive graph encoders: the $k$-IGN and the $k$-order Folklore GNN \cite{maron2019invariant,maron2019provably,azizian2021expressive}. We call the resulting models Spectral $k$-IGN and Spectral $k$-FGNN, respectively. Unfortunately, our results are negative for all of these higher-order spectral invariant GNNs:

\begin{proposition}
\label{thm:higher-order-gnn}
    For all $k>2$, Spectral $k$-IGN is as expressive as $k$-WL. Similarly, for all $k\ge 2$, Spectral $k$-FGNN is as expressive as $k$-FWL.
\end{proposition}

We give a proof in \cref{sec:proof_other_architectures}. Combined with the results that $k$-IGN is already as expressive as $k$-WL and $k$-FGNN is already as expressive as $k$-FWL \citep{maron2019provably,azizian2021expressive,geerts2022expressiveness}, we conclude that commonly-used spectral information does not help when combined with highly powerful GNN designs.

\textbf{Discussions on higher-order spectral features.} The above negative result further inspires us to think about the following question: is it still possible to use spectral information to enhance the expressive power of higher-order GNNs? Here, we offer some possible directions towards this goal. The crux here is to use \emph{higher-order} spectral features. Specifically, all the spectral information considered in previous sections (e.g., distance or projection matrices) is at most 2-dimensional. Can we generalize these spectral features into multi-dimensional tensors? This is indeed possible: for example, a simple approach is to use symmetric powers of a graph (also called the \emph{token graph}), which has been widely studied in literature \cite{audenaert2007symmetric,alzaga2010spectra,barghi2009non,monroy2012token}. The $k$-th symmetric power of graph $G$, denoted by $G^{\{ k\}}$, is the graph formed by vertex set $V_{G^{\{k\}}}:=\{S\subset V_G:|S|=k\}$ and edge set $E_{G^{\{k\}}}:=\{\{S_1,S_2\}:S_1,S_2\in V_{G^{\{k\}}}, S_1\triangle S_2\in E_G\}$. Here, each element in $V_{G^{\{k\}}}$ is a multiset of cardinality $k$, and two multisets are connected if their symmetric difference is an edge. In this way, one can easily define higher-order spectral information $(\lambda_i,\tP_i)\in\mathbb R\times \mathbb R^{n^{2k}}$ based on $\gP_{G^{\{k\}}}^\mM=\{(\lambda_i,\mP_i)\}_{i=1}^m$ (the eigenspace projection invariant associated with the $k$-th token graph), e.g., by setting $\tP_i(u_1,\cdots,u_k,v_1,\cdots,v_k)=\mP_i(\ldblbrace u_1,\cdots,u_k\rdblbrace,\ldblbrace v_1,\cdots,v_k\rdblbrace)$. The higher-order spectral information can then serve as initial features of any higher-order GNN that computes representations for each vertex tuple, such as $2k$-IGN.

Several works have pointed out the strong expressive power of higher-order spectral features. \citet{audenaert2007symmetric} verified that the spectra of the 3rd symmetric power are already not less expressive than 3-WL. Moreover, \citet{alzaga2010spectra,barghi2009non} upper bounds the expressive power of the spectra of the $k$-th symmetric power by $2k$-FWL. These results imply that using higher-order projection tensors is a promising approach to further boosting the expressive power of higher-order GNNs. We leave the corresponding architectural design and expressiveness analysis as an open direction for future study.

% Finally, by integrating all the previous results, we are able to build a comprehensive expressiveness hierarchy unifying all architectures studied in this paper. 

\section{Experiments}

In this section, we empirically evaluate the expressive power of various GNN architectures studied in this paper. We adopt the BREC benchmark \cite{wang2023towards}, a comprehensive dataset for comparing the expressive power of GNNs. We focus on the following GNNs that are closely related to this paper: $(\mathrm{i})$~Graphormer \cite{ying2021transformers} (a distance-based GNN that uses SPD, see \cref{sec:distance_gnn}); $(\mathrm{ii})$~NGNN \cite{zhang2021nested} (a variant of subgraph GNN, see \cref{sec:expressvieness_bound}); $(\mathrm{ii})$~ESAN \cite{bevilacqua2022equivariant} (an advanced subgraph GNN that adds cross-graph aggregations, see \cref{sec:expressvieness_bound}); $(\mathrm{iv})$~PPGN \cite{maron2019provably} (a higher-order GNN, see \cref{sec:higher-order-gnn}); $(\mathrm{v})$~EPNN (this paper). We follow the same setup as in \citet{wang2023towards} in both training and evaluation. For all baseline GNNs, the reported numbers are directly borrowed from \citet{wang2023towards}; For EPNN, we run the model 10 times with different seeds and report the average performance\footnote{Our code can be found in the following github repo:\\ \href{https://github.com/LingxiaoShawn/EPNN-Experiments}{{https://github.com/LingxiaoShawn/EPNN-Experiments}}}.

\begin{table}[ht]
    \centering
    \small
    \setlength{\tabcolsep}{4pt}
    \vspace{-2pt}
    \caption{Empirical performance of different GNNs on BREC.}
    \label{tab:results}
    \vspace{2pt}
    \begin{tabular}{cc|cccc|c}
        \hline
        Model & WL class & Basic & Reg & Ext & CFI & Total \\ \hline
        Graphormer & SPD-WL	& 26.7 & 10.0 & 41.0 & 10.0 & 19.8 \\
        NGNN & SWL & 98.3 & 34.3 & 59.0 & 0 & 41.5 \\
        ESAN & GSWL & 96.7 & 34.3 & 100.0 & 15.0 & 55.2 \\
        PPGN & 3-WL & 100.0 & 35.7 & 100.0 & 23.0 & 58.2 \\ \hline
        EPNN & EPWL & 100.0 & 35.7 & 100.0 & 5.0 & 53.8\\ \hline
    \end{tabular}
    \vspace{-8pt}
\end{table}

The results are presented in \cref{tab:results}. From these results, one can see that the empirical performance of EPNN matches its theoretical expressivity in our established hierarchy. Concretely, EPNN performs much better than Graphormer (SPD-WL) and NGNN (SWL), while underperforming ESAN (GS-WL) and PPGN (3-WL).

\section{Conclusion}
This paper investigates the expressive power of spectral invariant GNNs and related models. It establishes an expressiveness hierarchy between current models using a unifying framework we propose. We also draw a surprising connection to a recently proposed class of highly expressive GNNs, demonstrating that specific instances from this class (e.g., PSWL) upper bound all spectral invariant GNNs. This implies spectral invariant GNNs are strictly less expressive than the 3-WL test. Furthermore, we show spectral projection features and spectral distances do not provide additional expressivity benefits when combined with more powerful high-order architectures. We give a graphical illustration of all of these results in \cref{fig:hierarchy}.

\textbf{Open questions.} There are still several promising directions that are not fully explored in this paper. First, an interesting question lies in how the choice of graph matrix $\mM$ affects the expressive power of spectral invariant GNNs. We suspect that using (normalized) Laplacian matrix can be more beneficial than using the adjacency matrix, and the former is \emph{strictly} more expressive. We have given implicit evidence showing that EPNN with matrix $\hat\mL$ can encode all distances studied in this paper; however, we were unable to demonstrate the same result for other graph matrices. Besides, another important open question is the expressive power of higher-order spectral features, such as the one obtained by using token graphs. It is still unknown about a tight lower/upper bound of their expressive power in relation to higher-order WL tests. Moreover, investigating the \emph{refinements} over higher-order spectral features and the corresponding GNNs could be a fantastic open direction.
% One of the implications of this work is that spectral invariant features such as distances and projection matrices are not the appropriate method for gaining expressive power above 3-WL.

% \bohang{I will add results for SPE, which can take an additional 10-20 lines}

\section*{Impact Statement}
This paper presents work whose goal is to advance the field of Machine Learning. There are many potential societal consequences of our work, none of which we feel must be specifically highlighted here.

\section*{Acknowledgements}
HM is the Robert J. Shillman Fellow and is supported by the Israel Science Foundation through a personal grant (ISF 264/23) and an equipment grant (ISF 532/23). BZ would like to thank Jingchu Gai for helpful discussions.

\bibliography{ref}
\bibliographystyle{icml2024}

\newpage
\appendix
\onecolumn

\section{Proofs}
\label{sec:proofs}

This section presents all the missing proofs in this paper. First, \cref{sec:proof_preliminary} defines basic concepts and introduces our proof technique that will be frequently used in subsequent analysis. Then, \cref{sec:proof_epwl_1wl} gives several basic results and proves \cref{thm:epwl_1wl}. The formal proofs of our main theorems, including \cref{thm:epwl_pswl,thm:distance,thm:spectral_ign,thm:basisnet_basic}, are presented in \cref{sec:proof_epwl_sswl,sec:proof_distance,sec:proof_spectral_ign,sec:proof_basisnet}, respectively. Discussion with other architectures, such as GRIT, PEG, Spectral PPGN, and Spectral $k$-IGN are presented in \cref{sec:proof_other_architectures}. Finally, \cref{sec:proof_counterexample} reveals the gaps between each pair of architectures in our paper, leading to the proofs of \cref{thm:siamese_ign_spectral_ign,thm:siamese_ign_distance}.

\subsection{Preliminary}
\label{sec:proof_preliminary}

We first introduce some basic terminologies and concepts for general color refinement algorithms.

\textbf{Color mapping.} Any graph invariant over $\gG_k$ is called a $k$-dimensional color mapping. We use $\gM_k$ to denote the family of all $k$-dimensional color mappings. For a color mapping $\chi\in\gM_k$, our interest lies not in the specific values of the function (say $\chi_G(\vu)$ for some $\vu\in V_G^k$), but in the equivalence relations among different values. Formally, each color mapping $\chi\in\gM_k$ defines an equivalence relation $\stackrel{\chi}{\sim}$ between rooted graphs $G^\vu$, $H^\vv$ marking $k$ vertices, where $G^\vu\stackrel{\chi}{\sim}H^\vv$ iff $\chi_G(\vu)=\chi_H(\vv)$. For any graph $G\in\gG$, the equivalence relation $\stackrel{\chi}{\sim}$ induces a partition $Q_G(\chi)$ over the set $V_G^k$. 

Given two color mappings $\chi_1,\chi_2\in\gM_k$, we say $\chi_1$ is equivalent to $\chi_2$, denoted as $\chi_1\equiv\chi_2$, if $G^\vu\stackrel{\chi_1}{\sim}H^\vv\iff G^\vu\stackrel{\chi_2}{\sim}H^\vv$ for all graphs $G,H\in\gG$ and vertices $\vu\in V_G^k$, $\vv\in V_H^k$. One can see that ``$\equiv$'' forms an equivalence relation over $\gM_k$. We say $\chi_1$ is \emph{finer} than $\chi_2$, denoted as $\chi_1\preceq\chi_2$, if $G^\vu\stackrel{\chi_1}{\sim}H^\vv\implies G^\vu\stackrel{\chi_2}{\sim}H^\vv$ for all graphs $G$, $H$ and vertices $\vu\in V_G^k$, $\vv\in V_H^k$. One can see that ``$\preceq$'' forms a partial relation on $\gM_k$. We say $\chi_1$ is \emph{strictly finer} than $\chi_2$, denoted as $\chi_1\prec\chi_2$, if $\chi_1\preceq\chi_2$ and $\chi_1\not\equiv\chi_2$. %Finally, we remark that in certain cases, we may only consider connected graphs. In such instances, we can modify the definition of relations ``$\equiv$'', ``$\preceq$'', and ``$\prec$'' accordingly to apply only to connected graphs.

\textbf{Color refinement.} A function $T:\gM_k\to\gM_{k'}$ that maps from one color mapping to another is called a color transformation. Throughout this paper, we assume that all color transformations are \emph{order-preserving}, i.e., for all $\chi_1,\chi_2\in\gM_k$, $T(\chi_1)\preceq T(\chi_2)$ if $\chi_1\preceq\chi_2$. An order-preserving color transformation $T:\gM_k\to\gM_{k}$ is further called a \emph{color refinement} if $T(\chi)\preceq \chi$ for all $\chi\in \gM_k$. For any color refinement $T$, we denote by $T^t$ the $t$-th function power of $T$, i.e., the function composition $T\circ \cdots\circ T$ with $t$ occurrences of $T$. Note that if $T$ is a color refinement, so is $T^t$ for all $t\ge 0$.

Given two color transformations $T_1,T_2:\gM_k\to\gM_{k'}$, we say $T_1$ is as expressive as $T_2$, denoted by $T_1\equiv T_2$, if $T_1(\chi)\equiv T_2(\chi)$ for all $\chi\in\gM_k$. We say $T_1$ is more expressive than $T_2$, denoted by $T_1\preceq T_2$, if $T_1(\chi)\preceq T_2(\chi)$ for all $\chi\in\gM_k$. We say $T_1$ is strictly more expressive than $T_2$, denoted by $T_1\prec T_2$,  if $T_1$ is more expressive than $T_2$ and not as expressive as $T_2$. As will be clear in our subsequent proofs, the expressive power of GNNs is compared through an examination of their color transformations.

For any color refinement $T:\gM_k\to\gM_k$, we define the corresponding stable refinement $T^\infty:\gM_k\to\gM_k$ as follows. For any $\chi\in\gM_k$, define the color mapping $T^\infty(\chi)$ such that $G^\vu\stackrel{T^\infty(\chi)}{\sim} H^\vv$ iff $G^\vu\stackrel{T^t(\chi)}{\sim} H^\vv$ where $t\ge 0$ is the minimum integer satisfying $Q_G(T^t(\chi))=Q_G(T^{t+1}(\chi))$ and $Q_H(T^t(\chi))=Q_H(T^{t+1}(\chi))$.
% $$[T^\infty(\chi)]_G(\vu)=\inf_{t:Q_G(T^t(\chi))=Q_G(T^{t+1}(\chi))} [T^t(\chi)]_G(\vu)\quad \forall \chi\in\gM_k,G^\vu\in\gG_k. $$
Note that $T^\infty$ is well-defined since $t$ always exists (one can see that $Q_G(T^t(\chi))=Q_G(T^{t+1}(\chi))$ and $Q_H(T^t(\chi))=Q_H(T^{t+1}(\chi))$ holds for all $t\ge\max(|V_G|^k,|V_H|^k)$ when $T$ is a color refinement), and it is easy to see that $T^\infty$ is a color refinement. We call $T^\infty$ stable because $(T\circ T^\infty)(\chi)\equiv T^\infty(\chi)$ holds for all $\chi\in\gM_k$, i.e., $T\circ T^\infty\equiv T^\infty$.

A color refinement algorithm $A$ is formed by the composition of a stable refinement  $T^\infty:\gM_k\to\gM_k$ with a color transformation $U:\gM_k\to\gM_0$, called the \emph{pooling} transformation. It can be formally written as $A:=U\circ T^\infty$. Below, we will derive several useful properties for general color refinement algorithms. These properties will be frequently used to compare the expressive power of different algorithms.

\begin{proposition}
\label{thm:refinement1}
    Let $T_1,T_2:\gM_{k_1}\to\gM_{k_2}$, $U_1,U_2:\gM_{k_2}\to\gM_{k_3}$ be color transformations. If $T_1\preceq T_2$ and $U_1\preceq U_2$, then $U_1\circ T_1 \preceq U_2\circ T_2$.
\end{proposition}
% \vspace{-7pt}
\begin{proof}
    Since $T_1\preceq T_2$, $T_1(\chi)\preceq T_2(\chi)$ holds for all $\chi\in\gM_{k_1}$. Since $U_2$ is order-preserving, $U_2(T_1(\chi))\preceq U_2(T_2(\chi))$ holds for all $\chi\in\gM_{k_1}$. Finally, since $U_1\preceq U_2$, $U_1(T_1(\chi))\preceq U_2(T_1(\chi))$ holds for all $\chi\in\gM_{k_1}$. Combining the above inequalities yields the desired result.
\end{proof}
\begin{proposition}
\label{thm:refinement2}
    Let $T_1:\gM_{k_1}\to\gM_{k_1}$ and $T_2:\gM_{k_2}\to\gM_{k_2}$ be color refinements, and let $U_1:\gM_{k_0}\to\gM_{k_1}$ and $U_2:\gM_{k_1}\to\gM_{k_2}$ be color transformations. If $T_2 \circ U_2\circ T_1^\infty\circ U_1\equiv U_2\circ T_1^\infty\circ U_1$, then $U_2\circ T_1^\infty\circ U_1\preceq T_2^\infty\circ U_2\circ U_1$.
\end{proposition}
% \vspace{-7pt}
\begin{proof}
    If $T_2 \circ U_2\circ T_1^\infty\circ U_1\equiv U_2\circ T_1^\infty\circ U_1$, then by definition of stable refinement, $T_2^\infty \circ U_2\circ T_1^\infty\circ U_1\equiv U_2\circ T_1^\infty\circ U_1$. Since $T_1^\infty$ is a refinement,  $U_2\circ T_1^\infty\circ U_1\preceq T_2^\infty\circ U_2\circ U_1$ according to \cref{thm:refinement1}. We thus obtain the desired result.
\end{proof}
\begin{corollary}
\label{thm:refinement3}
    Let $T_1,T_2:\gM_{k}\to\gM_{k}$ be color refinements. Then, $T_1\preceq T_2$ implies that $T_2\circ T_1^\infty\equiv T_1^\infty$.
\end{corollary}
\begin{proof}
    Since $T_1\preceq T_2$, $T_1\circ T_1^\infty\preceq T_2\circ T_1^\infty$ by \cref{thm:refinement1}. Namely, $T_1^\infty\preceq T_2\circ T_1^\infty$. On the other hand, $T_2\circ T_1^\infty\preceq T_1^\infty$ since $T_2$ is a color refinement. Combined the two directions yields the desired result.
\end{proof}

The above two propositions will play a crucial role in our subsequent proofs. Below, we give a simple example to illustrate how these results can be used to give a proof that a GNN model $M_1$ is more expressive than another model $M_2$. Suppose $T_1,T_2:\gM_k\to\gM_k$ are color refinements corresponding to one GNN layer of $M_1$ and $M_2$, respectively, and let $U:\gM_k\to\gM_0$ be the color transformation corresponding to the final pooling layer in $M_1$ and $M_2$. Then, the color refinement algorithms associated with $M_1$ and $M_2$ can be represented by $U\circ T_1^\infty$ and $U\circ T_2^\infty$, respectively. Concretely, denote by $\chi^{0}$ the initial color mapping in the two algorithms, e.g., the constant mapping where $\chi^0_G(\vu)$ is the same for all $G^\vu\in\gG_k$. It follows that the graph representation of a graph $G$ computed by the two algorithms is $[U(T_1(\chi^0))](G)$ and $[U(T_2(\chi^0))](G)$, respectively. Then, The statement ``$M_1$ is more expressive than $M_2$'' means that $[U(T_1(\chi^0))](G)=[U(T_1(\chi^0))](H)\implies [U(T_2(\chi^0))](G)=[U(T_2(\chi^0))](H)$ for all graphs $G,H\in\gG$.

To prove that $M_1$ is more expressive than $M_2$, it suffices to prove that $T_2\circ T_1^\infty$ is as expressive as $T_1^\infty$. Indeed, this is actually a simple consequence of \cref{thm:refinement1,thm:refinement2}. If $T_2\circ T_1^\infty$ is as expressive as $T_1^\infty$, then $T_1^\infty$ is more expressive than $T_2^\infty$ (\cref{thm:refinement2}), and thus $U\circ T_1^\infty$ is more expressive than $U\circ T_2^\infty$ (\cref{thm:refinement1}), yielding the desired result.

\subsection{Basic results}
\label{sec:proof_epwl_1wl}

This subsection proves several basic results for EPWL. We begin by proving that EPWL is strictly more expressive than 1-WL (\cref{thm:epwl_1wl}). We will first restate \cref{thm:epwl_1wl} using the color refinement terminologies defined in \cref{sec:proof_preliminary}. We need the following color transformations:
\begin{itemize}[topsep=0pt,leftmargin=20pt]
    \setlength{\itemsep}{0pt}
    \item \textbf{EPWL color refinement.} Define $T_{\mathsf{EP},\mM}:\gM_1\to\gM_1$ such that for any color mapping $\chi\in\gM_1$ and rooted graph $G^u$,
    \begin{equation}
        [T_{\mathsf{EP},\mM}(\chi)]_G(u)=\hash(\chi_G(u),\ldblbrace (\chi_G(v),\gP^\mM_G(u,v)):v\in V_G\rdblbrace).
    \end{equation}
    \item \textbf{1-WL color refinement.} Define $T_\mathsf{WL}:\gM_1\to\gM_1$ such that for any color mapping $\chi\in\gM_1$ and rooted graph $G^u$,
    \begin{equation}
    \label{eq:wl}
        [T_\mathsf{WL}(\chi)]_G(u)=\hash(\chi_G(u),\ldblbrace (\chi_G(v),\atp_G(u,v)):v\in V_G\rdblbrace).
    \end{equation}
    \item \textbf{Global pooling.} Define $T_\mathsf{GP}:\gM_1\to\gM_0$ such that for any color mapping $\chi\in\gM_1$ and graph $G$,
    \begin{equation}
        [T_\mathsf{GP}(\chi)](G)=\hash(\ldblbrace \chi_G(u):u\in V_G\rdblbrace).
    \end{equation}
\end{itemize}
Equipped with the above color transformations, \cref{thm:epwl_1wl} is equivalent to the following:
\begin{proposition}
    For any graph matrix $\mM\in\{\mA,\mL,\hat\mL\}$, $T_\mathsf{GP}\circ T_{\mathsf{EP},\mM}^\infty \preceq T_\mathsf{GP}\circ T_\mathsf{WL}^\infty$.
\end{proposition}

Based on \cref{thm:refinement1,thm:refinement2}, it suffices to prove that $T_{\mathsf{EP},\mM}\preceq T_\mathsf{WL}$.

\begin{lemma}
\label{thm:epwl_atp}
    For any graph matrix $\mM\in\{\mA,\mL,\hat\mL\}$, $\gP^\mM\preceq \atp$. Here, the atomic type operator $\atp\in\gM_2$ is regarded as a 2-dimensional color mapping. This readily implies that $T_{\mathsf{EP},\mM}\preceq T_\mathsf{WL}$.
\end{lemma}
\begin{proof}
    It suffices to prove that, for any two graphs $G,H\in\gG$ and vertices $u,v\in V_G$, $x,y\in V_H$, if $\gP^\mM_G(u,v)=\gP^\mM_H(x,y)$, then (a) $u=v\iff x=y$; (b) $\{u,v\}\in E_G\iff \{x,y\}\in E_H$.

    Item (a) simply follows from the fact that $u=v$ iff $\sum_{(\lambda,\mP(u,v))\in\gP^\mM_G(u,v)}\mP(u,v)=1$ (by definition of eigen-decomposition). Item (b) simply follows from the fact that $\{u,v\}\in E_G$ iff $\sum_{(\lambda,\mP(u,v))\in\gP^\mM_G(u,v)}\lambda\mP(u,v)\neq 0$ and $u\neq v$ (which holds for all matrices $\mM\in\{\mA,\mL,\hat\mL\}$).
\end{proof}

\begin{wrapfigure}{r}{0.32\textwidth}
  % \vspace{-20pt}
  \begin{center}
   \includegraphics[width=0.3\textwidth]{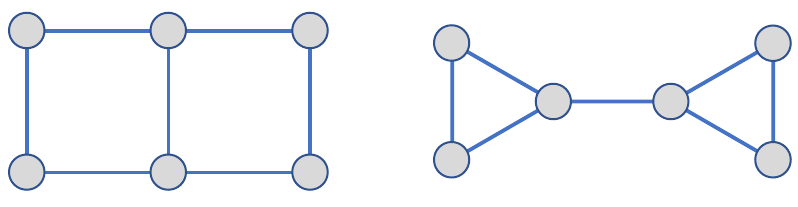}
  \end{center}
  \vspace{-10pt}
  \caption{A pair of counterexample graphs that are indistinguishable by 1-WL but can be distinguished via EPWL.}
  \label{fig:1wl_indistinguishable}
  \vspace{-20pt}
\end{wrapfigure}
\textbf{Strict sepatation.} It is easy to give a pair of counterexample graphs $G,H$ that are indistinguishable by 1-WL but can be distinguished via EPWL. We give such a pair of graphs in \cref{fig:1wl_indistinguishable}. One can check that the two graphs has difference set of eigenvalues no matter what graph matrix $\mM\in\{\mA,\mL,\hat\mL\}$ is used in EPWL.

We next show that EPWL is more expressive than spectral positional encoding using Laplacian eigenvectors. We will prove the following result:
\begin{proposition}
\label{thm:epwl_first_layer_encode_Puu}
    Given any graph matrix $\mM\in\{\mA,\mL,\hat\mL\}$ and any initial color mapping $\chi^0\in\gM$, let $\chi^1=T_{\mathsf{EP},\mM}(\chi^0)$ be the EPWL color mapping after the first iteration. Then, for any graphs $G,H\in\gG$ and vertices $u\in V_G,x\in V_H$, if $\chi^1_G(u)=\chi^1_H(v)$, then $\gP^\mM_G(u,u)=\gP^\mM_H(x,x)$.
\end{proposition}
\begin{proof}
    Let $G,H\in\gG$ and vertices $u\in V_G,x\in V_H$ satisfy that $\chi^1_G(u)=\chi^1_H(v)$. Then,
    \begin{equation}
    \label{eq:proof_epwl_first_layer_encode_Puu}
        \ldblbrace (\chi^0_G(v),\gP^\mM_G(u,v)):v\in V_G\rdblbrace=\ldblbrace (\chi^0_H(y),\gP^\mM_H(x,y)):y\in V_H\rdblbrace.
    \end{equation}
    Note that for any $u,v\in V_G$ and $x,y\in V_H$, if $\gP^\mM_G(u,v)=\gP^\mM_H(x,y)$, then $u=v\iff x=y$ (\cref{thm:epwl_atp}). Therefore, \cref{eq:proof_epwl_first_layer_encode_Puu} implies that $\gP^\mM_G(u,u)=\gP^\mM_H(x,x)$.
\end{proof}

\subsection{Proof of \cref{thm:epwl_pswl}}
\label{sec:proof_epwl_sswl}

This subsection aims to prove \cref{thm:epwl_pswl}. We will first restate \cref{thm:epwl_pswl} using the color refinement terminologies defined in \cref{sec:proof_preliminary}. We need the following color transformations:
\begin{itemize}[topsep=0pt,leftmargin=20pt]
    \setlength{\itemsep}{0pt}
    \item \textbf{EPWL color refinement.} Define $T_{\mathsf{EP},\mM}:\gM_1\to\gM_1$ such that for any color mapping $\chi\in\gM_1$ and rooted graph $G^u$,
    \begin{equation}
        [T_{\mathsf{EP},\mM}(\chi)]_G(u)=\hash(\chi_G(u),\ldblbrace (\chi_G(v),\gP^\mM_G(u,v)):v\in V_G\rdblbrace).
    \end{equation}
    \item \textbf{SWL color refinement.} Define $T_\mathsf{S}:\gM_2\to\gM_2$ such that for any color mapping $\chi\in\gM_2$ and rooted graph $G^{uv}$,
    \begin{equation}
        [T_\mathsf{S}(\chi)]_G(u,v)=\hash(\chi_G(u,v),\ldblbrace (\chi_G(u,w),\atp_G(v,w)):w\in V_G\rdblbrace).
    \end{equation}
    \item \textbf{PSWL color refinement.} Define $T_\mathsf{PS}:\gM_2\to\gM_2$ such that for any color mapping $\chi\in\gM_2$ and rooted graph $G^{uv}$,
    \begin{equation}
        [T_\mathsf{PS}(\chi)]_G(u,v)=\hash(\chi_G(u,v),\chi_G(v,v),\ldblbrace (\chi_G(u,w),\atp_G(v,w)):w\in V_G\rdblbrace).
    \end{equation}
    \item \textbf{Global refinement.} Define $T_\mathsf{Gu},T_\mathsf{Gv}:\gM_2\to\gM_2$ such that for any color mapping $\chi\in\gM_2$ and rooted graph $G^{uv}$,
    \begin{align}
        [T_\mathsf{Gu}(\chi)]_G(u,v)=\hash(\chi_G(u,v),\ldblbrace \chi_G(u,w):w\in V_G\rdblbrace),\\
        [T_\mathsf{Gv}(\chi)]_G(u,v)=\hash(\chi_G(u,v),\ldblbrace \chi_G(w,v):w\in V_G\rdblbrace).
    \end{align}
    \item \textbf{Diagonal refinement.} Define $T_\mathsf{Du},T_\mathsf{Dv}:\gM_2\to\gM_2$ such that for any color mapping $\chi\in\gM_2$ and rooted graph $G^{uv}$,
    \begin{align}
        [T_\mathsf{Du}(\chi)]_G(u,v)=\hash(\chi_G(u,v),\chi_G(u,u)),\\
        [T_\mathsf{Dv}(\chi)]_G(u,v)=\hash(\chi_G(u,v),\chi_G(v,v)).
    \end{align}
    \item \textbf{Node marking refinement.} Define $T_\mathsf{NM}:\gM_2\to\gM_2$ such that for any color mapping $\chi\in\gM_2$ and rooted graph $G^{uv}$,
    \begin{equation}
        [T_\mathsf{NM}(\chi)]_G(u,v)=\hash(\chi_G(u,v),\mathbb I[u=v]).
    \end{equation}
    \item \textbf{Lift transformation.} Define $T_\uparrow:\gM_1\to\gM_2$ such that for any color mapping $\chi\in\gM_1$ and rooted graph $G^{uv}$,
    \begin{equation}
        [T_\uparrow(\chi)]_G(u,v)=\chi_G(v).
    \end{equation}
    \item \textbf{Subgraph pooling.} Define $T_\mathsf{SP}:\gM_2\to\gM_1$ such that for any color mapping $\chi\in\gM_2$ and rooted graph $G^{u}$,
    \begin{equation}
        [T_\mathsf{SP}(\chi)]_G(u)=\hash(\ldblbrace \chi_G(u,v):v\in V_G\rdblbrace).
    \end{equation}
    \item \textbf{Global pooling.} Define $T_\mathsf{GP}:\gM_1\to\gM_0$ such that for any color mapping $\chi\in\gM_1$ and graph $G$,
    \begin{equation}
        [T_\mathsf{GP}(\chi)](G)=\hash(\ldblbrace \chi_G(u):u\in V_G\rdblbrace).
    \end{equation}
\end{itemize}

Note that $T_\mathsf{PS}\preceq T_\mathsf{S}$, $T_\mathsf{PS}\preceq T_\mathsf{Gu}$, and $T_\mathsf{PS}\preceq T_\mathsf{Dv}$. Equipped with the above color transformations, \cref{thm:epwl_pswl} is equivalent to the following:
\begin{theorem}
\label{thm:epwl_pswl_formal}
    For any graph matrix $\mM\in\{\mA,\mL,\hat\mL\}$, $T_\mathsf{GP}\circ T_\mathsf{SP} \circ T_\mathsf{PS}^\infty \circ T_\mathsf{NM}\circ T_\uparrow\preceq T_\mathsf{GP}\circ T_{\mathsf{EP},\mM}^\infty$.
\end{theorem}
We will decompose the proof into a set of lemmas. First, we leverage a recent breakthrough in graph theory established by \citet{rattan2023weisfeiler}. We restate their result in our context as follows.
\begin{definition}
\label{def:equitable_matrix}
    Define a family of graph matrices $\mathfrak{E}$ as follows, called \emph{equitable matrices}:
    \begin{enumerate}[label=\alph*),topsep=0pt,leftmargin=20pt]
    \setlength{\itemsep}{0pt}
        \item Base matrices: the identify matrix $\mI$, all-one matrix $\mJ$, adjacency matrix $\mA$, degree matrix $\mD$ are all equitable;
        \item Algebraic property: for any equitable matrices $\mM_1,\mM_2\in\mathfrak{E}$, $\mM_1+\mM_2,c\mM_1,\mM_1\mM_2\in\mathfrak{E}$, where $c\in\mathbb C$ can be any constant.
        \item Spectral property: for any $\mM\in\mathfrak{E}$ and $\lambda\in\mathbb C$, let $\mP^\mM_\lambda$ be the projection onto the eigenspace spanned by the eigenvectors of $\mM$ with eigenvalue $\lambda$ ($\mP^\mM_\lambda=\mO$ if $\lambda$ is not an eigenvalue of $\mM$). Then, $\mP^\mM_\lambda\in\mathfrak{E}$.
    \end{enumerate}
\end{definition}
Based on \cref{def:equitable_matrix}, we readily have the following proposition:
\begin{proposition}
    The Laplacian matrix $\mL$ and the normalized Laplacian matrix $\hat\mL$ are equitable.
\end{proposition}
\citet{rattan2023weisfeiler} proved the following main result:
\begin{theorem}
\label{thm:ratten}
    For any $\mM\in\mathfrak{E}$, $T_\mathsf{S}^\infty(T_\mathsf{NM}(\chi^\mathsf{C}))\preceq \mM$, where $\chi^\mathsf{C}\in\gM_2$ is the constant color mapping.
\end{theorem}
\begin{corollary}
\label{thm:ratten_corollary}
    For any symmetric equitable matrix $\mM\in\mathfrak{E}$ and initial color mapping $\chi^0\in\gM_1$, $(T_\mathsf{PS}^\infty\circ T_\mathsf{NM}\circ T_\uparrow)(\chi^0)\preceq \gP^\mM$.
\end{corollary}
\begin{proof}
    For any symmetric equitable matrix $\mM\in\mathfrak{E}$ and any $\lambda\in\mathbb R$, $\mP_\lambda^\mM\in\mathfrak{E}$ holds by \cref{def:equitable_matrix}(c). Therefore, \cref{thm:ratten} implies that $T_\mathsf{S}^\infty(T_\mathsf{NM}(\chi^\mathsf{C}))\preceq \mP^\mM_\lambda$, i.e., $T_\mathsf{PS}^\infty(T_\mathsf{NM}(\chi^\mathsf{C}))\preceq \mP^\mM_\lambda$. Next, note that $\mP^\mM_\lambda\neq \mO$ iff $\lambda$ is an eigenvalue of $\mM$. Therefore, the graph invariant $\bm\Delta$ over $\gG_2$ defined by $\bm\Delta_G(u,v)=\mathbb I[\lambda\text{ is an eigenvalue of }\mM]$ satisfies that $(T_\mathsf{Gu}\circ T_\mathsf{Dv}\circ T_\mathsf{Gu} \circ T_\mathsf{PS}^\infty\circ T_\mathsf{NM})(\chi^\mathsf{C}))\preceq \bm\Delta$. Since $T_\mathsf{Gu}\circ T_\mathsf{Dv}\circ T_\mathsf{Gu} \circ T_\mathsf{PS}^\infty\equiv T_\mathsf{PS}^\infty$ (based on the facts $T_\mathsf{PS}\preceq T_\mathsf{Gu}$ and $T_\mathsf{PS}\preceq T_\mathsf{Dv}$ and \cref{thm:refinement3}), we have $T_\mathsf{PS}^\infty(T_\mathsf{NM}(\chi^\mathsf{C}))\preceq \bm\Delta$. By considering all $\lambda\in\mathbb R$, we obtain that $T_\mathsf{PS}^\infty(T_\mathsf{NM}(\chi^\mathsf{C}))\preceq \gP^\mM$. Finally, noting that $T_\mathsf{PS}^\infty\circ T_\mathsf{NM}$ is order-preserving and $T_\uparrow(\chi^0)\preceq \chi^\mathsf{C}$, we have $(T_\mathsf{PS}^\infty\circ T_\mathsf{NM}\circ T_\uparrow)(\chi^0)\preceq \gP^\mM$ for all $\chi^0\in\gM_1$.
\end{proof}

We are now ready to prove \cref{thm:epwl_pswl_formal}.
\begin{proof}[Proof of \cref{thm:epwl_pswl_formal}]
    Note that $T_\mathsf{SP} \circ T_\mathsf{NM}\circ T_\uparrow$ is a color refinement, and thus $T_{\mathsf{EP},\mM}^\infty\circ T_\mathsf{SP} \circ T_\mathsf{NM}\circ T_\uparrow\preceq T_{\mathsf{EP},\mM}^\infty$. To prove that $T_\mathsf{GP}\circ T_\mathsf{SP} \circ T_\mathsf{PS}^\infty \circ T_\mathsf{NM}\circ T_\uparrow\preceq T_\mathsf{GP}\circ T_{\mathsf{EP},\mM}^\infty$, it suffices to prove that $T_\mathsf{SP} \circ T_\mathsf{PS}^\infty\circ T_\mathsf{NM}\preceq T_{\mathsf{EP},\mM}^\infty\circ T_\mathsf{SP}\circ T_\mathsf{NM}$ according to \cref{thm:refinement1}. Moreover, based on \cref{thm:refinement2}, it suffices to prove that $T_{\mathsf{EP},\mM}\circ T_\mathsf{SP} \circ T_\mathsf{PS}^\infty\circ T_\mathsf{NM}\equiv T_\mathsf{SP} \circ T_\mathsf{PS}^\infty\circ T_\mathsf{NM}$.

    Let $\chi^0\in\gM_2$ be any initial color mapping and let $\chi=T_\mathsf{PS}^\infty(T_\mathsf{NM}(\chi_0))$. Pick any graphs $G,H$ and vertices $u\in V_G,x\in V_H$ such that $[T_\mathsf{SP}(\chi)]_G(u)=[T_\mathsf{SP}(\chi)]_H(x)$, i.e.,
    \begin{equation}
    \label{eq:proof_epwl_sswl_1}
        \ldblbrace\chi_G(u,v):v\in V_G\rdblbrace=\ldblbrace\chi_H(x,y):y\in V_H\rdblbrace.
    \end{equation}
    Invoking \cref{thm:ratten_corollary} obtains that
    \begin{equation}
        \ldblbrace(\chi_G(u,v),\gP^M_G(u,v)):v\in V_G\rdblbrace=\ldblbrace(\chi_H(x,y),\gP^M_H(x,y)):y\in V_H\rdblbrace.
    \end{equation}
    Since $T_\mathsf{PS}\preceq T_\mathsf{Dv}$, based on \cref{thm:refinement3} we have $T_\mathsf{Dv}\circ T_\mathsf{PS}^\infty\equiv T_\mathsf{PS}^\infty$, i.e., $\chi \equiv T_\mathsf{Dv}(\chi)$. Therefore,
    \begin{equation}
        \ldblbrace(\chi_G(v,v),\gP^M_G(u,v)):v\in V_G\rdblbrace=\ldblbrace(\chi_H(y,y),\gP^M_H(x,y)):y\in V_H\rdblbrace.
    \end{equation}
    Since $T_\mathsf{PS}\preceq T_\mathsf{Gu}$, we similarly have $\chi \equiv T_\mathsf{Gu}(\chi)$. Therefore,
    \begin{equation}
    \label{eq:proof_epwl_sswl_2}
        \ldblbrace(\ldblbrace\chi_G(v,w):w\in V_G\rdblbrace,\gP^M_G(u,v)):v\in V_G\rdblbrace=\ldblbrace(\ldblbrace\chi_H(y,z):z\in V_H\rdblbrace,\gP^M_H(x,y)):y\in V_H\rdblbrace.
    \end{equation}
    Combining with \cref{eq:proof_epwl_sswl_1,eq:proof_epwl_sswl_2}, we have $T_{\mathsf{EP},\mM}(T_\mathsf{SP}(\chi))\equiv T_\mathsf{SP}(\chi)$. Finally, noting that $\chi^0$ is arbitrary, we conclude that $T_{\mathsf{EP},\mM}\circ T_\mathsf{SP}\circ T_\mathsf{PS}^\infty\circ T_\mathsf{NM}\equiv T_\mathsf{SP}\circ T_\mathsf{PS}^\infty\circ T_\mathsf{NM}$.
\end{proof}

\subsection{Proof of \cref{thm:distance}}
\label{sec:proof_distance}

This subsection aims to prove \cref{thm:distance}. We will first restate \cref{thm:distance} using the color refinement terminologies defined in \cref{sec:proof_preliminary}. Similarly to \cref{sec:proof_epwl_sswl}, we define the following color transformations:
\begin{itemize}[topsep=0pt,leftmargin=20pt]
    \setlength{\itemsep}{0pt}
    \item \textbf{EPWL color refinement.} Define $T_{\mathsf{EP},\mM}:\gM_1\to\gM_1$ such that for any color mapping $\chi\in\gM_1$ and rooted graph $G^u$,
    \begin{equation}
        [T_{\mathsf{EP},\mM}(\chi)]_G(u)=\hash(\chi_G(u),\ldblbrace (\chi_G(v),\gP^\mM_G(u,v)):v\in V_G\rdblbrace).
    \end{equation}
    % \item \textbf{Augmented EPWL color refinement.} Define $\tilde T_{\mathsf{EP},\mM}:\gM_1\to\gM_1$ such that for any color mapping $\chi\in\gM_1$ and rooted graph $G^u$,
    % \begin{equation}
    %     [\tilde T_{\mathsf{EP},\mM}(\chi)]_G(u)=\hash(\chi_G(u),\ldblbrace (\chi_G(u),\chi_G(v),\gP^\mM_G(u,v)):v\in V_G\rdblbrace).
    % \end{equation}
    \item \textbf{GD-WL color refinement.} Define $T_{\mathsf{GD},\mM}:\gM_1\to\gM_1$ such that for any color mapping $\chi\in\gM_1$ and rooted graph $G^u$,
    \begin{equation}
        [T_{\mathsf{GD},\mM}(\chi)]_G(u)=\hash(\chi_G(u),\ldblbrace (\chi_G(v),\mM_G(u,v)):v\in V_G\rdblbrace).
    \end{equation}
    \item \textbf{Global pooling.} Define $T_\mathsf{GP}:\gM_1\to\gM_0$ such that for any color mapping $\chi\in\gM_1$ and graph $G$,
    \begin{equation}
        [T_\mathsf{GP}(\chi)](G)=\hash(\ldblbrace \chi_G(u):u\in V_G\rdblbrace).
    \end{equation}
\end{itemize}
\cref{thm:distance} is equivalent to the following:
\begin{theorem}
\label{thm:distance_formal}
    For any graph distance matrix $\mM$ listed in \cref{sec:distance_gnn}, $T_\mathsf{GP}\circ T_{\mathsf{EP},\hat\mL}^\infty \preceq T_\mathsf{GP}\circ T_{\mathsf{GD},\mM}^\infty$.
\end{theorem}
% We will first present several basic facts that will be used to give a proof of \cref{thm:distance_formal}.
% \begin{proposition}
%     For any symmetric graph matrix $\mM$, $T_{\mathsf{EP},\mM}\equiv \tilde T_{\mathsf{EP},\mM}$.
% \end{proposition}
% \begin{proof}
%     Pick any $\chi^0\in\gM_1$ and pick any rooted graphs $G^u,H^x$. Let $\chi=T_{\mathsf{EP},\mM}(\chi^0)$ and $\tilde \chi=\tilde T_{\mathsf{EP},\mM}(\chi^0)$. Then,
%     \begin{align*}
%         \chi_G(u)=\chi_H(v)&\iff \chi^0_G(u)\!=\!\chi^0_H(v)\land\ldblbrace (\chi^0_G(v),\gP^\mM_G(u,v)):v\in V_G\rdblbrace\!=\!\ldblbrace (\chi^0_H(y),\gP^\mM_H(x,y)):y\in V_H\rdblbrace\\
%         &\iff\chi^0_G(u)\!=\!\chi^0_H(v)\land\ldblbrace (\chi^0_G(u),\chi^0_G(v),\gP^\mM_G(u,v)):v\!\in\! V_G\rdblbrace\!=\!\ldblbrace (\chi^0_H(x),\chi^0_H(y),\gP^\mM_H(x,y)):y\!\in\! V_H\rdblbrace\\
%         &\iff \tilde\chi_G(u)=\tilde\chi_H(v).
%     \end{align*}
% \end{proof}
In this subsection, let $\chi^\mathsf{C}\in\gM_1$ be the constant color mapping and let $\chi=T_{\mathsf{EP},\hat\mL}^\infty(\chi^\mathsf{C})$. Define a color mapping $\bar\chi\in\gM_2$ such that $\bar\chi_G(u,v)=\left(\chi_G(u),\chi_G(v),\gP^{\hat\mL}_G(u,v)\right)$ for all graph $G$ and vertices $u,v\in V_G$. The following proposition will play a central role in our subsequent proofs:
\begin{proposition}
\label{thm:epwl_distance_key_lemma}
    For any graph matrix $\mM$, if $\bar\chi\preceq \mM$, then $T_\mathsf{GP}\circ T_{\mathsf{EP},\hat\mL}^\infty\preceq T_\mathsf{GP}\circ T_{\mathsf{GD},\mM}^\infty$.
\end{proposition}
\begin{proof}
    Based on \cref{thm:refinement1,thm:refinement2}, it suffices to prove that $T_{\mathsf{GD},\mM}\circ T_{\mathsf{EP},\hat\mL}^\infty\equiv T_{\mathsf{EP},\hat\mL}^\infty$.
    % Since $T_{\mathsf{EP},\hat\mL}\equiv \tilde T_{\mathsf{EP},\hat\mL}$, it suffices to prove that $T_{\mathsf{GD},\mM}\circ \hat T_{\mathsf{EP},\hat\mL}^\infty\equiv \hat T_{\mathsf{EP},\hat\mL}^\infty$.
    Pick any $\chi^0\in\gM_1$ and let $\chi'=T_{\mathsf{EP},\hat\mL}^\infty(\chi^0)$. Since $\chi'\equiv T_{\mathsf{EP},\hat\mL}(\chi')$, we have
    \begin{align*}
        \chi'_G(u)=\chi'_H(y)&\implies\chi'_G(u)=\chi'_H(y)\land\ldblbrace (\chi'_G(v),\gP^{\hat\mL}_G(u,v)):v\in V_G\rdblbrace=\ldblbrace (\chi'_H(y),\gP^{\hat\mL}_H(x,y)):y\in V_H\rdblbrace\\
        &\implies\ldblbrace (\chi'_G(u),\chi'_G(v),\gP^{\hat\mL}_G(u,v)):v\in V_G\rdblbrace=\ldblbrace (\chi'_H(x),\chi'_H(y),\gP^{\hat\mL}_H(x,y)):y\in V_H\rdblbrace
    \end{align*}
    for all graphs $G,H\in\gG$ and vertices $u\in V_G,x\in V_H$. Moreover, since $\chi'\preceq \chi$ and $\bar\chi\preceq\mM$, we have
    \begin{equation*}
        \chi'_G(u)=\chi'_H(y)\implies\ldblbrace (\chi'_G(v),\mM_G(u,v)):v\in V_G\rdblbrace=\ldblbrace (\chi'_H(y),\mM_H(x,y)):y\in V_H\rdblbrace
    \end{equation*}
    for all graphs $G,H\in\gG$ and vertices $u\in V_G,x\in V_H$. Therefore, $\chi'\equiv T_{\mathsf{GD},\mM}(\chi')$. Finally, by noting that $\chi'=T_{\mathsf{EP},\hat\mL}^\infty(\chi^0)$ and $\chi^0$ is arbitrary, we obtain that $T_{\mathsf{GD},\mM}\circ T_{\mathsf{EP},\hat\mL}^\infty\equiv T_{\mathsf{EP},\hat\mL}^\infty$, as desired.
\end{proof}

We next present several basic facts about EP-WL stable colors $\chi$.
\begin{proposition}
\label{thm:epwl_distance_basic}
    For any graphs $G,H\in\gG$ and vertices $u\in V_G,x\in V_H$, if $\chi_G(u)=\chi_H(x)$, then
    \begin{enumerate}[label=\alph*),topsep=0pt,leftmargin=20pt]
    \setlength{\itemsep}{0pt}
        \item $\deg_G(u)=\deg_H(x)$;
        \item $\gP^{\hat\mL}_G(u,u)=\gP^{\hat\mL}_H(x,x)$.
    \end{enumerate}
\end{proposition}
\begin{proof}
    Since $T_{\mathsf{EP},\mM}\preceq T_\mathsf{WL}$ (\cref{thm:epwl_atp}), we have $\chi\preceq T_\mathsf{WL}^\infty(\chi^\mathsf{C})$. Item (a) follows from the fact that the 1-WL stable color can encode the vertex degree, i.e., $T_\mathsf{WL}^\infty(\chi^\mathsf{C})\preceq \deg$. To prove item (b), note that $\chi\preceq \chi^1$, where $\chi^1$ is defined in \cref{thm:epwl_first_layer_encode_Puu}. Therefore, item (b) readily follows from \cref{thm:epwl_first_layer_encode_Puu}.
\end{proof}

Below, we will separately consider each of the distances listed in \cref{sec:distance_gnn}. 

\begin{lemma}
\label{thm:epwl_spd}
    Let $\mM$ be the shortest path distance matrix. Then, $\bar\chi\preceq \mM$.
\end{lemma}
\begin{proof}
    We will prove a stronger result that $\gP^{\hat\mL}\preceq \mM$. Note that there is a relation between $\mM$ and the normalized adjacency matrix $\hat\mA:=\mD^{-1/2}\mA\mD^{-1/2}=\mI-\hat\mL$, which can be expressed as follows:
    \begin{equation}
        \mM_G(u,v)=\min\{i\in\mathbb N:\hat\mA^i_G(u,v)>0\}\qquad \forall G\in\gG,u,v\in V_G.
    \end{equation}
    The above relation can be interpreted as follows: $\hat\mA^i_G(u,v)=\sum_{p_i(u,v)}\omega(p_i(u,v))$ where $p_i(u,v)$ ranges over all walks of length $i$ from $u$ to $v$, and $\omega(p_i(u,v)):=(\deg_G(u)\deg_G(v))^{-1/2}\left(\Pi_{j=1}^{i-1} \deg_G(w_i)\right)^{-1}$ for a walk $p_i(u,v)=(u,w_1,\cdots,w_{i-1},v)$, which is always positive.

    Based on the property of eigen-decomposition, we have
    \begin{equation}
        \hat\mA^i_G(u,v)=\sum_{(\lambda,\mP_G(u,v))\in \gP^{\hat\mA}_G}\lambda^i\mP_G(u,v)=\sum_{(\lambda,\mP_G(u,v))\in \gP^{\hat\mL}_G(u,v)}(1-\lambda)^i\mP_G(u,v).
    \end{equation}
    This implies that $\mM_G(u,v)$ can be purely determined by $\gP^{\hat\mL}_G(u,v)$, i.e., $\gP^{\hat\mL}\preceq \mM$.
\end{proof}

\begin{lemma}
\label{thm:epwl_rd}
    Let $\mM$ be the resistance distance matrix. Then, $\bar\chi\preceq \mM$.
\end{lemma}
\begin{proof}
    We first consider the case of connected graphs. We leverage a celebrated result established in \citet{klein1993resistance}, which builds connections between resistance distance and graph Laplacian. Specifically, for any connected graph $G\in\gG$ and vertices $u,v\in V_G$, 
    \begin{equation}
    \label{eq:resistance_distance}
        \mM_G(u,v)=\mL^\dag(u,u)+\mL^\dag(v,v)-2\mL^\dag(u,v),
    \end{equation}
    where $\dag$ denotes the matrix Moore–Penrose inversion. Substituting $\mL=\mD^{1/2}\hat\mL\mD^{1/2}$ into \cref{eq:resistance_distance}, we obtain
    \begin{equation}
    \label{eq:resistance_distance2}
        \mM_G(u,v)=(\deg_G(u))^{-1}\hat\mL^\dag(u,u)+(\deg_G(v))^{-1}\hat\mL^\dag(v,v)-2(\deg_G(u)\deg_G(v))^{-1/2}\hat\mL^\dag(u,v),
    \end{equation}
    Moreover, we have $\hat\mL^\dag=\sum_{i:\lambda_i\neq 0}\lambda_i^{-1}\mP_i$ where $\sum_{i}\lambda_i\mP_i$ is the eigen-decomposition of $\hat\mL$. Combining all these relations, one can see that $\mM_G(u,v)$ is purely determined by the tuple $(\deg_G(u),\deg_G(v),\gP^{\hat\mL}_G(u,u),\gP^{\hat\mL}_G(v,v),\gP^{\hat\mL}_G(u,v))$. Based on \cref{thm:epwl_distance_basic}, $\chi_G(u)$ determines $\deg_G(u)$ and $\gP^{\hat\mL}_G(u,u)$, and $\chi_G(v)$ determines $\deg_G(v)$ and $\gP^{\hat\mL}_G(v,v)$. Consequently, $\bar\chi\preceq \mM$.

    We next consider the general case where the graph $G$ is disconnected. In this case, $\mL$ is a block diagonal matrix. It is straightforward to see that 
    \begin{equation}
    \label{eq:resistance_distance_general}
        \mM_G(u,v)=\left\{\begin{array}{ll}
            \mL^\dag(u,u)+\mL^\dag(v,v)-2\mL^\dag(u,v) & \text{if $u$ and $v$ are in the same connected component,} \\
            \infty & \text{otherwise.}
        \end{array}\right.
    \end{equation}
    Since we have proved that $\gP^M$ can encode the shortest path distance (\cref{thm:epwl_spd}), $\gP^{\hat\mL}_G(u,v)$ can encode whether $u$ and $v$ are in the same connected component. So we still have $\bar\chi\preceq \mM$.
\end{proof}

\begin{lemma}
\label{thm:epwl_htd}
    Let $\mM$ be the hitting-time distance matrix. Then, $\bar\chi\preceq \mM$.
\end{lemma}
\begin{proof}
    Without loss of generality, we only consider connected graphs as in \cref{thm:epwl_rd}. According to the definition of hitting-time distance, for any connected graph $G\in\gG$ and vertices $u,v\in V_G$,  we have the following recursive relation:
    \begin{equation}
    \label{eq:htd}
        \mM_G(u,v)=\left\{\begin{array}{ll}
            0 & \text{if }u=v, \\
            1+\frac 1 {\deg_G(u)}\sum_{w\in\gN_G(u)}\mM_G(u,v) & \text{otherwise}.
        \end{array}\right.
    \end{equation}
    Now define a new graph matrix
    \begin{equation}
    \label{eq:proof_htd_0}
        \tilde\mH=\mathbf 1\mathbf 1^\top+\mD^{-1}\mA(\tilde\mH-\diag (\tilde\mH))
    \end{equation}
    where $\diag(\tilde \mH)$ is the diagonal matrix obtained from $\tilde \mH$ by zeroing out all non-diagonal elements of $\tilde \mH$. It follows that $\mM=\tilde\mH-\diag (\tilde\mH)$.

    We first calculate $\diag(\tilde{\mathbf H})$. Left-multiplying \cref{eq:proof_htd_0} by $\mathbf 1^\top\mA$ yields that
    \begin{equation*}
        \mathbf 1^\top\mA\tilde\mH=\mathbf 1^\top\mA\mathbf 1\mathbf 1^\top+\mathbf 1^\top\mA\mD^{-1}\mA(\tilde {\mathbf H}-\diag (\tilde{\mathbf H}))=\mathbf 1^\top\mA\mathbf 1\mathbf 1^\top+\mathbf 1^\top\mA(\tilde {\mathbf H}-\diag (\tilde{\mathbf H})),
    \end{equation*}
    where we use the fact that $\mathbf 1^\top \mA\mD^{-1}=\mathbf 1^\top$. Therefore, $\mathbf 1^\top\mA\diag (\tilde{\mH})=\mathbf 1^\top\mA \mathbf 1\mathbf 1^\top$, namely, $\tilde \mH_G(u,u)=2|E_G|(\deg_G(u))^{-1}$.

    We next compute the full matrix $\mM$. Based on \cref{eq:proof_htd_0}, we have
    \begin{equation}
    \label{eq:proof_htd_1}
        (\mI-\mD^{-1}\mA)\mM=\mathbf 1\mathbf 1^\top-\diag (\tilde\mH).
    \end{equation}
    Left-multiplying \cref{eq:proof_htd_1} by $\mD$ leads to the fundamental equation\
    \begin{equation}
    \label{eq:proof_htd_2}
        \mL\mM=\mD\mathbf 1\mathbf 1^\top-2|E_G|\mI.
    \end{equation}
    When the graph $G$ is connected, the eigenspace of $\mL$ associated with eigenvalue $\lambda=0$ only has one dimension, and any eigenvector $\vb$ satisfying $\mL\vb=\mathbf 0$ has the form $\vb=c\mathbf 1$ for some $c$. This implies that all solutions to \cref{eq:proof_htd_2} has the form
    \begin{equation}
        \mM=\mL^\dag\mD\mJ+\mJ\mS-2|E_G|\mL^\dag
    \end{equation}
    where $\mS$ is a diagonal matrix and $\mJ=\mathbf 1\mathbf 1^\top$. Noting that $\diag(\mM)=\mO$, we have $\mO=\diag(\mL^\dag\mD\mJ)+\diag(\mJ\mS)-2|E_G|\diag(\mL^\dag)$. Since $\diag(\mJ\mS)=\mS$ for any diagonal matrix $\mS$ and $\mJ\diag(\mL^\dag\mD\mJ)=\mJ\mD\mL^\dag$, we final obtain
    \begin{equation}
    \label{eq:proof_htd_3}
    \begin{aligned}
        \mM&=\mL^\dag\mD\mJ+\mJ\left(2|E_G|\diag(\mL^\dag)-\diag(\mL^\dag\mD\mJ)\right)-2|E_G|\mL^\dag\\
        &=\mL^\dag\mD\mJ-\mJ\mD\mL^\dag+2|E_G|\mJ\diag(\mL^\dag)-2|E_G|\mL^\dag.
    \end{aligned}
    \end{equation}
    Below, we will prove that $\bar\chi\preceq \mM$. Noting that $\mL^\dag=\mD^{-1/2}\hat\mL^\dag\mD^{-1/2}$, we have
    \begin{equation}
    \label{eq:proof_htd_4}
        \mM=\mD^{-1/2}\hat\mL^\dag\mD^{1/2}\mJ-\mJ\mD^{1/2}\hat\mL^\dag\mD^{-1/2}+2|E_G|\mJ\mD^{-1/2}\diag(\hat\mL^\dag)\mD^{-1/2}-2|E_G|\mD^{-1/2}\hat\mL^\dag\mD^{-1/2}.
    \end{equation}
    Equivalently, for any graph $G\in\gG$ and vertices $u,v\in V_G$,
    \begin{equation}
    \label{eq:proof_htd_5}
    \begin{aligned}
        \mM_G(u,v)&=\mD^{-1/2}_G(u,u)\sum_{w\in V_G}\hat\mL^\dag_G(u,w)\mD^{1/2}_G(w,w)-\mD^{-1/2}_G(v,v)\sum_{w\in V_G}\hat\mL^\dag_G(w,v)\mD^{1/2}_G(w,w)\\
        &\quad+2|E_G|\hat\mL^\dag_G(v,v)\mD^{-1}_G(v,v)-2|E_G|\mD^{-1/2}_G(u,u)\hat\mL^\dag_G(u,v)\mD^{-1/2}_G(v,v).
    \end{aligned}
    \end{equation}
    From the above equation, one can see that $\mM_G(u,v)$ is fully determined by the following tuple:
    \begin{equation*}
        \left(|E_G|,\deg_G(u),\deg_G(v),\hat\mL^\dag_G(v,v),\hat\mL^\dag_G(u,v),\ldblbrace(\hat\mL^\dag_G(u,w),\deg_G(w)):w\in V_G\rdblbrace,\ldblbrace(\hat\mL^\dag_G(w,v),\deg_G(w)):w\in V_G\rdblbrace\right).
    \end{equation*}
    According to the eigen-decomposition, $\hat\mL^\dag=\sum_{i:\lambda_i\neq 0}\lambda_i^{-1}\mP_i$ where $\sum_{i}\lambda_i\mP_i$ is the eigen-decomposition of $\hat\mL$. Therefore, $\hat\mL^\dag\preceq\gP^{\hat\mL}$. Besides, based on \cref{thm:epwl_distance_basic} we have that $\deg_G(u)$ is determined by $\chi_G(u)$, $\deg_G(v)$ is determined by $\chi_G(v)$, and $\gP^{\hat\mL}_G(v,v)$ is determined by $\chi_G(v)$. It follows that $|E_G|$ is determined by $\chi_G(u)$, because by $\chi\equiv T_{\mathsf{EP},\hat\mL}(\chi)$ we have for all graphs $G,H\in\gG$ and vertices $u\in V_G,x\in V_H$,
    \begin{align*}
        \chi_G(u)=\chi_H(x)&\implies \ldblbrace\chi_G(v):v\in V_G\rdblbrace=\ldblbrace\chi_H(y):y\in V_H\rdblbrace\\
        &\implies\ldblbrace\deg_G(v):v\in V_G\rdblbrace=\ldblbrace\deg_H(y):y\in V_H\rdblbrace\\
        &\implies |E_G|=|E_H|.
    \end{align*}
    We next prove that $\ldblbrace(\hat\mL^\dag_G(u,w),\deg_G(w)):w\in V_G\rdblbrace$ is determined by $\chi_G(u)$. Again by using $\chi\equiv T_{\mathsf{EP},\hat\mL}(\chi)$, we have for all graphs $G,H\in\gG$ and vertices $u\in V_G,x\in V_H$,
    \begin{align*}
        \chi_G(u)=\chi_H(x)&\implies \ldblbrace(\chi_G(v),\gP^{\hat\mL}_G(u,v)):v\in V_G\rdblbrace=\ldblbrace(\chi_H(y),\gP^{\hat\mL}_H(x,y)):y\in V_H\rdblbrace\\
        &\implies\ldblbrace(\deg_G(v),\hat\mL^\dag_G(u,v)):v\in V_G\rdblbrace=\ldblbrace(\deg_H(y),\hat\mL^\dag_H(x,y)):y\in V_H\rdblbrace.
    \end{align*}
    Using the same analysis and noting that $\mM^\dag_G(w,v)=\mM^\dag_G(v,w)$ for any symmetric matrix $\mM$, we can prove that $\ldblbrace(\hat\mL^\dag_G(w,v),\deg_G(w)):w\in V_G\rdblbrace$ is determined by $\chi_G(v)$. Combining all these relations leads to the conclusion that $\bar\chi\preceq \mM$.
\end{proof}

\begin{corollary}
\label{thm:epwl_ctd}
    Let $\mM$ be the commute-time distance matrix. Then, $\bar\chi\preceq \mM$.
\end{corollary}
\begin{proof}
    This is a simple consequence of \cref{thm:epwl_htd}. Denoting by $\mH$ the hitting-time distance matrix, we have $\mM=\mH+\mH^\top$. Since $\bar\chi\preceq \mH$, $\bar\chi\preceq \mH^\top$ (because $\gP^{\hat\mL}_G(u,v)=\gP^{\hat\mL}_G(v,u)$ for any graph $G$ and vertices $u,v\in V_G$). Therefore, $\bar\chi\preceq \mM$.
\end{proof}

\begin{lemma}
\label{thm:epwl_prd}
    Let $\mM$ be the PageRank distance matrix associated with weight sequence $\gamma_0,\gamma_1,\cdots$. Then, $\bar\chi\preceq \mM$.
\end{lemma}
\begin{proof}
    By definition of PageRank distance, $\mM=\sum_{k=0}^\infty \gamma_k(\mD^{-1}\mA)^k$. Let $\sum_{i}\lambda_i\mP_i$ be the eigen-decomposition of $\hat\mL$. Then, $\hat\mA=\mD^{-1/2}\mA\mD^{-1/2}=\mI-\hat\mL=\sum_{i}(1-\lambda_i)\mP_i$. Therefore,
    \begin{equation}
        (\mD^{-1}\mA)^k=\mD^{-1/2}\hat\mA^k\mD^{1/2}=\sum_{i}(1-\lambda_i)^k\mD^{-1/2}\mP_i\mD^{1/2}.
    \end{equation}
    Plugging the above equation into the definition of PageRank distance, we have for any graph $G\in\gG$ and vertices $u,v\in V_G$,
    \begin{equation}
        \mM_G(u,v)=\sum_{i}\left(\sum_{k=0}^\infty\gamma_k(1-\lambda_i)^k\right)\mP_i(u,v)(\deg_G(u))^{-1/2}(\deg_G(v))^{1/2}.
    \end{equation}
    One can see that $\mM_G(u,v)$ is purely determined by the tuple $(\deg_G(u),\deg_G(v),\gP^{\hat\mL}_G(u,v))$. Based on \cref{thm:epwl_distance_basic}, $\chi_G(u)$ determines $\deg_G(u)$ and $\chi_G(v)$ determines $\deg_G(v)$. Consequently, $\bar\chi\preceq \mM$.
\end{proof}

We next study the (normalized) diffusion distance. Consider the continuous graph diffusion process defined as follows. Given time $t\ge 0$, let $\vp^t_G(u)$ be the probability ``mass'' of particles at position $u$. The particles will move following the differential equation given below:
\begin{equation}
\label{eq:diffusion_differential_equation}
    \frac {\mathrm d}{\mathrm d t}\vp^t=\mT\vp^t,
\end{equation}
where $\mT$ is the transition matrix. For example, when $\mT=\mA\mD^{-1}-\mI$, the differential equation essentially characterizes a random walk diffusion process. In this paper, we consider the normalized diffusion distance, which corresponds to $\mT=\hat\mA-\mI=-\hat\mL$. Given hyperparameter $\tau\ge 0$, denote by $(\vp\vert_{G^u})^\tau$ be the probability ``mass'' vector at time $\tau$ with the initial configuration $(\vp\vert_{G^u})^0(u)=1$ and $(\vp\vert_{G^u})^0(v)=0$ for all $v\neq u$. Then, the diffusion distance matrix $\mM$ is defined as
\begin{equation}
    \mM_G(u,v)=\left\|(\vp\vert_{G^u})^\tau-(\vp\vert_{G^v})^\tau\right\|_2.
\end{equation}

\begin{lemma}
\label{thm:epwl_difussion_distance}
    Let $\mM$ be the normalized diffusion distance defined above. Then, $\bar\chi\preceq \mM$.
\end{lemma}
\begin{proof}
    Since \cref{eq:diffusion_differential_equation} is a linear differential equation, we can solve it and obtain $(\vp\vert_{G^u})^\tau=\exp(\tau\mT)(\vp\vert_{G^u})^0$. Therefore, 
    \begin{equation}
        \mM_G(u,v)=\left\|\exp(\tau\mT)\left((\vp\vert_{G^u})^0-(\vp\vert_{G^v})^0\right)\right\|_2,
    \end{equation}
    where $\exp(\mT)$ is the matrix exponential of $\mT$. Equivalently, 
    \begin{equation}
        \mM_G^2(u,v)=\left((\vp\vert_{G^u})^0-(\vp\vert_{G^v})^0\right)^\top\exp(2\tau\mT)\left((\vp\vert_{G^u})^0-(\vp\vert_{G^v})^0\right).
    \end{equation}
    Let $\sum_{i}\lambda_i\mP_i$ be the eigen-decomposition of $\hat\mL$. Then, $\exp(2\tau\mT)=\exp(-2\tau\hat\mL)=\sum_{i}\exp(-2\tau\lambda_i)\mP_i$. Therefore,
    \begin{equation}
        \mM_G^2(u,v)=\sum_{i}\exp(-2\tau\lambda_i)(\mP_i(u,u)+\mP_i(v,v)-2\mP_i(u,v)).
    \end{equation}
    This implies that $\mM_G(u,v)$ is purely determined by the tuple $(\gP^{\hat\mL}_G(u,u),\gP^{\hat\mL}_G(v,v),\gP^{\hat\mL}_G(u,v))$. We conclude that $\bar\chi\preceq \mM$ by using \cref{thm:epwl_distance_basic}.
\end{proof}

We finally study the biharmonic distance.
\begin{lemma}
\label{thm:epwl_biharmonic_distance}
    Let $\mM$ be the biharmonic distance defined in \citet{lipman2010biharmonic}. Then, $\bar\chi\preceq \mM$.
\end{lemma}
\begin{proof}
    Without loss of generality, we only consider connected graphs. According to \citet{wei2021biharmonic}, the biharmonic distance for a connected graph can be equivalently written as
    \begin{equation}
        \mM_G(u,v)=(\mL^2)^\dag(u,u)+(\mL^2)^\dag(v,v)-2(\mL^2)^\dag(u,v).
    \end{equation}
    The subsequent proof is almost the same as in \cref{thm:epwl_rd} and we omit it for clarity.
\end{proof}

\subsection{Proof of \cref{thm:spectral_ign}}
\label{sec:proof_spectral_ign}

This section aims to prove \cref{thm:spectral_ign}. Below, we will decompose the proof into three parts. First, we will describe the color refinement algorithm corresponding to Spectral IGN, which is equivalent to Spectral IGN in terms of distinguishing non-isomorphic graphs. Then, we will prove that this color refinement algorithm is more expressive than EPWL using the color refinement terminologies
defined in \cref{sec:proof_preliminary}. Finally, we will prove the other direction, i.e., EPWL is more expressive than the color refinement algorithm of Spectral IGN, thus concluding the proof of \cref{thm:spectral_ign}.

\textbf{Color refinement algorithm for Spectral IGN.} To define the algorithm, we first need to extend several concepts defined in \cref{sec:proof_preliminary} to incorporate eigenvalues. Formally, let $\Lambda^\mM$ be the graph spectrum invariant representing the set of eigenvalues for graph matrix $\mM$, i.e., $\Lambda^\mM(G):=\{\lambda:\lambda\text{ is an eigenvalue of }\mM_G\}$ for $G\in\gG$. Define
\begin{equation}
    \gG_k^{\gP^\mM}:=\{(G^\vu,\lambda):G^\vu\in\gG_k,\lambda\in\Lambda^\mM(G)\}.
\end{equation}
We can then define color mappings over $\gG_k^{\gP^\mM}$. Formally, a function $\chi$ defined over domain $\gG_k^{\gP^\mM}$ is called a color mapping if $\chi(G^\vu,\lambda)=\chi(H^\vv,\mu)$ holds for all $(G^\vu,\lambda),(H^\vv,\mu)\in\gG_k^{\gP^\mM}$ satisfying $G^\vu\simeq H^\vv$ and $\lambda=\mu$. Without ambiguity, we will use the notation $\chi_G(\lambda,\vu)$ to refer to $\chi(G^\vu,\lambda)$ for $(G^\vu,\lambda)\in\gG_k^{\gP^\mM}$. Define $\gM_k^{\gP^\mM}$ to be the family of all color mappings over $\gG_k^{\gP^\mM}$. We can similarly define equivalence relation ``$\equiv$'' and partial relation ``$\preceq$'' between color mappings in $\gM_k^{\gP^\mM}$. In addition, the color transformation can also be extended in a similar manner.

Throughout this section, we use the notation $\chi^{\gP^\mM}\in\gM_2^{\gP^\mM}$ to represent the initial color mapping in Spectral IGN, which is defined as $\chi^{\gP^\mM}_G(\lambda,u,v)=(\lambda,\mP^\mM_\lambda(u,v))$ for all $(G^{uv},\lambda)\in\gG_2^{\gP^\mM}$, where $\mP^\mM_\lambda$ is the projection matrix associated with eigenvalue $\lambda$. We then define the following color transformations:
\begin{itemize}[topsep=0pt,leftmargin=20pt]
    \setlength{\itemsep}{0pt}
    \item \textbf{2-IGN color refinement.} Define $T_\mathsf{IGN}:\gM_2\to\gM_2$ such that for any color mapping $\chi\in\gM_2$ and $G^{uv}\in\gG_2$,
    \begin{equation}
    \label{eq:2-ign_wl}
    \begin{aligned}
        [T_\mathsf{IGN}(\chi)]_G(u,v)=\hash(&\chi_G(u,v),\chi_G(u,u),\chi_G(v,v),\chi_G(v,u),\delta_{uv}(\chi_G(u,u)),\\
        &\ldblbrace\chi_G(u,w):w\in V_G\rdblbrace,\ldblbrace\chi_G(w,u):w\in V_G\rdblbrace,\\
        &\ldblbrace\chi_G(v,w):w\in V_G\rdblbrace,\ldblbrace\chi_G(w,v):w\in V_G\rdblbrace,\\
        &\ldblbrace\chi_G(w,w):w\in V_G\rdblbrace,\ldblbrace\chi_G(w,x):w,x\in V_G\rdblbrace,\\
        &\delta_{uv}(\ldblbrace\chi_G(u,w):w\in V_G\rdblbrace),\delta_{uv}(\ldblbrace\chi_G(w,u):w\in V_G\rdblbrace),\\
        &\delta_{uv}(\ldblbrace\chi_G(w,w):w\in V_G\rdblbrace),\delta_{uv}(\ldblbrace\chi_G(w,x):w,x\in V_G\rdblbrace)).
    \end{aligned}
    \end{equation}
    Here, the function $\delta_{uv}$ satisfies that $\delta_{uv}(c)=c$ if $u=v$ and $\delta_{uv}(c)=0$ otherwise (0 is a special element that differs from all $\chi_G(u,v)$). One can see that \cref{eq:2-ign_wl} has 15 aggregations inside the hash function, which matches the number of orthogonal bases for a 2-IGN layer \citep{maron2019invariant}.
    \item \textbf{Spectral pooling.} Define $T_\mathsf{SP}:\gM_2^{\gP^\mM}\to\gM_2$ such that for any color mapping $\chi\in\gM_2^{\gP^\mM}$ and $G^{uv}\in\gG_2$,
    \begin{equation}
    \label{eq:spectral_pooling}
        [T_\mathsf{SP}(\chi)]_G(u,v)=\hash(\ldblbrace\chi_G(\lambda,u,v):\lambda\in\Lambda^\mM(G)\rdblbrace).
    \end{equation}
    \item \textbf{Spectral IGN color refinement.} Define $T_\mathsf{SIGN}:\gM_2^{\gP^\mM}\to\gM_2^{\gP^\mM}$ such that for any color mapping $\chi\in\gM_2^{\gP^\mM}$ and $(G^{uv},\lambda)\in\gG_2^{\gP^\mM}$,
    \begin{equation}
    \label{eq:spectral_ign_wl}
        [T_\mathsf{SIGN}(\chi)]_G(\lambda,u,v)=\hash([T_\mathsf{IGN}(\chi(\lambda,\cdot,\cdot))]_G(u,v),[T_\mathsf{IGN}(T_\mathsf{SP}(\chi))]_G(u,v)).
    \end{equation}
    \item \textbf{Spectral IGN final pooling.} Define $T_\mathsf{FP}:\gM_2^{\gP^\mM}\to\gM_0$ such that for any color mapping $\chi\in\gM_2^{\gP^\mM}$ and $G\in\gG$,
    \begin{equation}
    \label{eq:final_pooling}
        [T_\mathsf{FP}(\chi)](G)=\hash(\ldblbrace \chi_G(\lambda, u,v):\lambda\in\Lambda^\mM(G),u,v\in V_G\rdblbrace).
    \end{equation}
    \item \textbf{Joint pooling.} Define $T_\mathsf{JP}:\gM_2\to\gM_0$ such that for any color mapping $\chi\in\gM_2$ and $G\in\gG$,
    \begin{equation}
    \label{eq:joint_pooling}
        [T_\mathsf{JP}(\chi)](G)=\hash(\ldblbrace \chi_G(u,v):u,v\in V_G\rdblbrace).
    \end{equation}
    \item \textbf{2-dimensional Pooling.} Define $T_\mathsf{P2}:\gM_2\to\gM_1$ such that for any color mapping $\chi\in\gM_2$ and $G\in\gG_1$,
    \begin{equation}
    \label{eq:2dim_pooling}
        [T_\mathsf{P2}(\chi)]_G(u)=\hash(\ldblbrace \chi_G(u,v):v\in V_G\rdblbrace).
    \end{equation}
    \item \textbf{EPWL color refinement.} Define $T_{\mathsf{EP},\mM}:\gM_1\to\gM_1$ such that for any color mapping $\chi\in\gM_1$ and $G^u\in\gG_1$,
    \begin{equation}
        [T_{\mathsf{EP},\mM}(\chi)]_G(u)=\hash(\chi_G(u),\ldblbrace (\chi_G(v),\gP^\mM_G(u,v)):v\in V_G\rdblbrace).
    \end{equation}
    \item \textbf{Global pooling.} Define $T_\mathsf{GP}:\gM_1\to\gM_0$ such that for any color mapping $\chi\in\gM_1$ and graph $G$,
    \begin{equation}
    \label{eq:global_pooling}
        [T_\mathsf{GP}(\chi)](G)=\hash(\ldblbrace \chi_G(u):u\in V_G\rdblbrace).
    \end{equation}
\end{itemize}

We now related the expressive power of Spectral IGN to the corresponding color refinement algorithm. The proof is straightforward following standard techniques, see e.g., \citet{zhang2023complete}.
\begin{proposition}
\label{thm:spectral_ign_wl}
    The expressive power of Spectral IGN is bounded by the color mapping $(T_\mathsf{FP}\circ T_\mathsf{SIGN}^\infty)(\chi^{\gP^\mM})$ in distinguishing non-isomorphic graphs. Moreover, with sufficient layers and proper network parameters, Spectral IGN can be as expressive as the above color mapping in distinguishing non-isomorphic graphs.
\end{proposition}

Equivalently, the pooling $T_\mathsf{FP}$ can be decomposed into three pooling transformations $T_\mathsf{GP}\circ T_\mathsf{P2}\circ T_\mathsf{SP}$, as stated below:
\begin{lemma}
\label{thm:spectral_ign_lemma1}
    $T_\mathsf{GP}\circ T_\mathsf{P2}\circ T_\mathsf{SP}\circ T_\mathsf{SIGN}^\infty\equiv T_\mathsf{FP}\circ T_\mathsf{SIGN}^\infty$.
\end{lemma}
\begin{proof}
    First, it is clear that $T_\mathsf{GP}\circ T_\mathsf{P2}\circ T_\mathsf{SP}\circ T_\mathsf{SIGN}^\infty\preceq T_\mathsf{FP}\circ T_\mathsf{SIGN}^\infty$. Thus, it suffices to prove that $T_\mathsf{FP}\circ T_\mathsf{SIGN}^\infty\preceq T_\mathsf{GP}\circ T_\mathsf{P2}\circ T_\mathsf{SP}\circ T_\mathsf{SIGN}^\infty$. Pick any initial color mapping $\chi^0\in\gM_2^{\gP^\mM}$ and let $\chi=T_\mathsf{SIGN}^\infty(\chi^0)$. Note that $\chi\equiv T_\mathsf{SIGN}(\chi)$. We will prove that $(T_\mathsf{FP})(\chi)\preceq  (T_\mathsf{GP}\circ T_\mathsf{P2}\circ T_\mathsf{SP})(\chi)$. Pick any graphs $G,H\in\gG$. We have
    \begin{align*}
        &\quad [T_\mathsf{FP}(\chi)](G)=[T_\mathsf{FP}(\chi)](H)\\
        &\implies \ldblbrace\chi_G(\lambda,u,v):\lambda\in\Lambda^\mM(G),u,v\in V_G\rdblbrace = \ldblbrace\chi_H(\mu,x,y):\mu\in\Lambda^\mM(H),x,y\in V_H\rdblbrace\\
        &\implies \ldblbrace\chi_G(\lambda,u,u):\lambda\in\Lambda^\mM(G),u\in V_G\rdblbrace = \ldblbrace\chi_H(\mu,x,x):\mu\in\Lambda^\mM(H),x\in V_H\rdblbrace\\
        &\implies \ldblbrace\ldblbrace\chi_G(\lambda,u,v):v\in V_G\rdblbrace:\lambda\in\Lambda^\mM(G),u\in V_G\rdblbrace = \ldblbrace\ldblbrace\chi_H(\mu,x,y):y\in V_H\rdblbrace:\mu\in\Lambda^\mM(H),x\in V_H\rdblbrace\\
        &\implies \ldblbrace\ldblbrace\ldblbrace\chi_G(\lambda',u,v):\lambda'\in\Lambda^\mM(G)\rdblbrace:v\in V_G\rdblbrace:\lambda\in\Lambda^\mM(G),u\in V_G\rdblbrace\\
        &\qquad= \ldblbrace\ldblbrace\ldblbrace\chi_H(\mu',x,y):\mu'\in\Lambda^\mM(H)\rdblbrace:y\in V_H\rdblbrace:\mu\in\Lambda^\mM(H),x\in V_H\rdblbrace\\
        &\implies \ldblbrace\ldblbrace\ldblbrace\chi_G(\lambda',u,v):\lambda'\in\Lambda^\mM(G)\rdblbrace:v\in V_G\rdblbrace:u\in V_G\rdblbrace\\
        &\qquad = \ldblbrace\ldblbrace\ldblbrace\chi_H(\mu',x,y):\mu'\in\Lambda^\mM(H)\rdblbrace:y\in V_H\rdblbrace:x\in V_H\rdblbrace\\
        &\implies [(T_\mathsf{GP}\circ T_\mathsf{P2}\circ T_\mathsf{SP})(\chi)](G)=[(T_\mathsf{GP}\circ T_\mathsf{P2}\circ T_\mathsf{SP})(\chi)](H),
    \end{align*}
    where the second, third, and fourth steps in the above derivation are based on \cref{eq:2-ign_wl,eq:spectral_pooling,eq:spectral_ign_wl}. We have obtained the desired result.
\end{proof}

In the subsequent proof, we will show that $(T_\mathsf{GP}\circ T_\mathsf{P2}\circ T_\mathsf{SP}\circ T_\mathsf{SIGN}^\infty)(\chi^{\gP^\mM})\equiv (T_\mathsf{GP}\circ T_{\mathsf{EP},\mM})(\chi^0)$, where $\chi^0\in\gM_1$ is the constant mapping.

\begin{lemma}
\label{thm:spectral_ign_lemma>}
    Let $\chi^0\in\gM_1$ be the constant mapping. Then, $(T_\mathsf{GP}\circ T_\mathsf{P2}\circ T_\mathsf{SP}\circ T_\mathsf{SIGN}^\infty)(\chi^{\gP^\mM})\preceq (T_\mathsf{GP}\circ T_{\mathsf{EP},\mM}^\infty)(\chi^0)$.
\end{lemma}
\begin{proof}
    Define an auxiliary color transformation $T_\times:\gM_1\to\gM_2$ such that for any color mapping $\chi\in\gM_1$ and rooted graph $G^{uv}\in\gG_2$,
    \begin{equation}
        [T_\times(\chi)]_G(u,v)=\hash(\chi_G(u),\chi_G(v)).
    \end{equation}
    We first prove that $T_\mathsf{GP}\equiv T_\mathsf{GP}\circ T_\mathsf{P2}\circ T_\times$. Pick any color mapping $\chi\in\gM_1$ and graphs $G,H\in\gG$,
    \begin{align*}
        &\qquad [T_\mathsf{GP}(\chi)](G)=[T_\mathsf{GP}(\chi)](H)\\
        &\iff \ldblbrace\chi_G(u):u\in V_G\rdblbrace=\ldblbrace\chi_H(x):x\in V_H\rdblbrace\\
        &\iff \ldblbrace(\chi_G(u),\ldblbrace\chi_G(v):v\in V_G\rdblbrace):u\in V_G\rdblbrace=\ldblbrace(\chi_H(x),\ldblbrace \chi_H(y):y\in V_H\rdblbrace):x\in V_H\rdblbrace\\
        &\iff \ldblbrace\ldblbrace(\chi_G(u),\chi_G(v)):v\in V_G\rdblbrace:u\in V_G\rdblbrace=\ldblbrace\ldblbrace (\chi_H(x),\chi_H(y)):y\in V_H\rdblbrace:x\in V_H\rdblbrace\\
        &\iff [(T_\mathsf{GP}\circ T_\mathsf{P2}\circ T_\times)(\chi)](G)=[( T_\mathsf{GP}\circ T_\mathsf{P2}\circ T_\times)(\chi)](H).
    \end{align*}
    Based on the equivalence, it suffices to prove that $(T_\mathsf{GP}\circ T_\mathsf{P2}\circ T_\mathsf{SP}\circ T_\mathsf{SIGN}^\infty)(\chi^{\gP^\mM})\preceq (T_\mathsf{GP}\circ T_\mathsf{P2}\circ T_\times\circ T_{\mathsf{EP},\mM}^\infty)(\chi^0)$. Since $T_\mathsf{GP}\circ T_\mathsf{P2}$ is order-preserving by definition, it suffices to prove that $(T_\mathsf{SP}\circ T_\mathsf{SIGN}^\infty)(\chi^{\gP^\mM})\preceq (T_\times\circ T_{\mathsf{EP},\mM}^\infty)(\chi^0)$.

    We will prove the following stronger result: for all $t\in\mathbb N$, $(T_\mathsf{SP}\circ T_\mathsf{SIGN}^t)(\chi^{\gP^\mM})\preceq (T_\times\circ T_{\mathsf{EP},\mM}^t)(\chi^0)$. The proof is based on induction. For the base case of $t=0$, since $(T_\times\circ T_{\mathsf{EP},\mM}^t)(\chi^0)$ is a constant mapping, the result clearly holds. Now assume that the result holds for $t=t'$ and consider the case of $t=t'+1$. Denote $\chi=T_\mathsf{SIGN}^{t'}(\chi^{\gP^\mM})$, $\hat\chi=T_{\mathsf{EP},\mM}^t(\chi^0)$, and note that $T_\mathsf{SP}(\chi)\preceq T_\mathsf{SP}(\chi^{\gP^\mM})\equiv \gP^\mM$. Pick any graphs $G,H\in\gG$ and vertices $u,v\in V_G$, $x,y\in V_H$. Based on the induction hypothesis, $[T_\mathsf{SP}(\chi)]_G(u,v)=[T_\mathsf{SP}(\chi)]_H(x,y)$ implies that $\hat\chi_G(u)=\hat\chi_H(x)$ and $\hat\chi_G(v)=\hat\chi_H(y)$. We have
    \begin{align*}
        &\quad [T_\mathsf{SP}(T_\mathsf{SIGN}(\chi))]_G(u,v)=[T_\mathsf{SP}(T_\mathsf{SIGN}(\chi))]_H(x,y)\\
        &\implies\ldblbrace [T_\mathsf{IGN}(T_\mathsf{SP}(\chi))]_G(u,v):\lambda\in\Lambda^\mM(G)\rdblbrace=\ldblbrace [T_\mathsf{IGN}(T_\mathsf{SP}(\chi))]_H(x,y):\mu\in\Lambda^\mM(H)\rdblbrace\\
        &\implies [T_\mathsf{IGN}(T_\mathsf{SP}(\chi))]_G(u,v)=[T_\mathsf{IGN}(T_\mathsf{SP}(\chi))]_H(x,y)\\
        &\implies [T_\mathsf{SP}(\chi)]_G(u,v)=[T_\mathsf{SP}(\chi)]_H(x,y)\land \ldblbrace[T_\mathsf{SP}(\chi)]_G(u,w):w\in V_G\rdblbrace=\ldblbrace[T_\mathsf{SP}(\chi)]_H(x,z):z\in V_H\rdblbrace\\
        &\qquad\land \ldblbrace[T_\mathsf{SP}(\chi)]_G(v,w):w\in V_G\rdblbrace=\ldblbrace[T_\mathsf{SP}(\chi)]_H(y,z):z\in V_H\rdblbrace\\
        &\implies \hat\chi_G(u)=\hat\chi_H(x)\land \ldblbrace(\hat\chi_G(w),\gP^\mM_G(u,w)):w\in V_G\rdblbrace= \ldblbrace(\hat\chi_H(z),\gP^\mM_H(x,z)):z\in V_H\rdblbrace\\
        &\qquad\land \hat\chi_G(v)=\hat\chi_H(y)\land \ldblbrace(\hat\chi_G(w),\gP^\mM_G(v,w)):w\in V_G\rdblbrace= \ldblbrace(\hat\chi_H(z),\gP^\mM_H(y,z)):z\in V_H\rdblbrace\\
        &\implies [T_{\mathsf{EP},\mM}(\hat\chi)]_G(u)=[T_{\mathsf{EP},\mM}(\hat\chi)]_H(x)\land [T_{\mathsf{EP},\mM}(\hat\chi)]_G(v)=[T_{\mathsf{EP},\mM}(\hat\chi)]_H(y)\\
        &\implies [(T_\times\circ T_{\mathsf{EP},\mM})(\hat\chi)]_G(u,v)=[(T_\times\circ T_{\mathsf{EP},\mM})(\hat\chi)]_H(x,y)
    \end{align*}
    where in the first step we use the definition of Spectral IGN (\cref{eq:spectral_ign_wl}) and spectral pooling (\cref{eq:spectral_pooling}); in the third step we use the definition of 2-IGN (\cref{eq:2-ign_wl}); in the fourth step we use the induction hypothesis and the fact that $T_\mathsf{SP}(\chi)\preceq \gP^\mM$. This concludes the induction step.
\end{proof}

\begin{lemma}
\label{thm:spectral_ign_aux_lemma}
    Define an auxiliary color transformation $T_{\times,\gP^\mM}:\gM_1\to\gM_2$ such that for any color mapping $\chi\in\gM_1$ and rooted graph $G^{uv}\in\gG_2$,
    \begin{equation}
        [T_{\times,\gP^\mM}(\chi)]_G(u,v)=\hash(\chi_G(u),\chi_G(v),\gP^\mM_G(u,v)).
    \end{equation}
    Then, $(T_{\times,\gP^\mM}\circ T_{\mathsf{EP},\mM}^\infty)(\chi^0)\preceq (T_\mathsf{SP}\circ T_\mathsf{SIGN}^\infty)(\chi^{\gP^\mM})$.
\end{lemma}
\begin{proof}
\begingroup
\allowdisplaybreaks
    We will prove the following stronger result: for any $t\ge 0$, $(T_{\times,\gP^\mM}\circ T_{\mathsf{EP},\mM}^{2t})(\chi^0)\preceq (T_\mathsf{SP}\circ T_\mathsf{SIGN}^t)(\chi^{\gP^\mM})$. The proof is based on induction. For the base case of $t=0$, since $(T_{\times,\gP^\mM}\circ T_{\mathsf{EP},\mM}^0)(\chi^0)=T_{\times,\gP^\mM}(\chi^0)\equiv T_\mathsf{SP}\circ \chi^{\gP^\mM}$, the result clearly holds. Now assume that the result holds for $t=t'$ and consider the case of $t=t'+1$. Denote $\chi=T_{\mathsf{EP},\mM}^{2t'}(\chi^0)$ and $\hat\chi=T_\mathsf{SIGN}^{t'}(\chi^{\gP^\mM})$. Pick any graphs $G,H\in\gG$ and vertices $u,v\in V_G$, $x,y\in V_H$. Based on the induction hypothesis, $[T_{\times,\gP^\mM}(\chi)]_G(u,v)=[T_{\times,\gP^\mM}(\chi)]_H(x,y)$ implies that $[T_\mathsf{SP}(\hat\chi)]_G(u,v)=[T_\mathsf{SP}(\hat\chi)]_H(x,y)$. We have
    \begin{align*}
        &\quad[(T_{\times,\gP^\mM}\circ T_{\mathsf{EP},\mM}^2)(\chi)]_G(u,v)=[(T_{\times,\gP^\mM}\circ T_{\mathsf{EP},\mM}^2)(\chi)]_H(x,y)\\
        &\implies [T_{\mathsf{EP},\mM}^2(\chi)]_G(u)=[T_{\mathsf{EP},\mM}^2(\chi)]_H(x)\land [T_{\mathsf{EP},\mM}^2(\chi)]_G(v)=[T_{\mathsf{EP},\mM}^2(\chi)]_H(y)\land \gP^\mM_G(u,v)=\gP^\mM_H(x,y)\\
        &\implies \chi_G(u)=\chi_H(x)\land \chi_G(v)=\chi_H(y) \land {\gP^\mM}_G(u,v)={\gP^\mM}_H(x,y) \\
        &\qquad\land \gP^\mM_G(u,u)=\gP^\mM_H(x,x)\land \gP^\mM_G(v,v)=\gP^\mM_H(y,y)\land \mathbb I[u=v]= \mathbb I[x=y]\\
        &\qquad \land\ldblbrace ([T_{\mathsf{EP},\mM}(\chi)]_G(w),\gP^\mM_G(u,w)):w\in V_G\rdblbrace=\ldblbrace ([T_{\mathsf{EP},\mM}(\chi)]_H(z),\gP^\mM_H(x,z)):z\in V_G\rdblbrace\\
        &\qquad \land\ldblbrace ([T_{\mathsf{EP},\mM}(\chi)]_G(w),\gP^\mM_G(v,w)):w\in V_G\rdblbrace=\ldblbrace ([T_{\mathsf{EP},\mM}(\chi)]_H(z),\gP^\mM_H(y,z)):z\in V_G\rdblbrace\\
        &\implies \chi_G(u)=\chi_H(x)\land \chi_G(v)=\chi_H(y) \land \gP^\mM_G(u,v)=\gP^\mM_H(x,y) \\
        &\qquad\land \gP^\mM_G(u,u)=\gP^\mM_H(x,x)\land \gP^\mM_G(v,v)=\gP^\mM_H(y,y)\land \mathbb I[u=v]= \mathbb I[x=y]\\
        &\qquad\land \ldblbrace (\chi_G(u),\chi_G(w),\gP^\mM_G(u,w)):w\in V_G\rdblbrace=\ldblbrace (\chi_H(x),\chi_H(z),\gP^\mM_H(x,z)):z\in V_G\rdblbrace\\
        &\qquad \land\ldblbrace (\chi_G(v),\chi_G(w),\gP^\mM_G(v,w)):w\in V_G\rdblbrace=\ldblbrace (\chi_H(y),\chi_H(z),\gP^\mM_H(y,z)):z\in V_G\rdblbrace\\
        &\qquad\land \ldblbrace (\chi_G(w),\gP^\mM_G(w,w)):w\in V_G\rdblbrace=\ldblbrace (\chi_H(z),\gP^\mM_H(z,z)):z\in V_G\rdblbrace\\
        &\qquad \land \ldblbrace (\chi_G(w),\chi_G(w'),\gP^\mM_G(w,w')):w,w'\in V_G\rdblbrace=\ldblbrace (\chi_H(z),\chi_H(z'),\gP^\mM_H(z,z')):z,z'\in V_G\rdblbrace\\
        &\implies [T_\mathsf{SP}(\hat\chi)]_G(u,v)=[T_\mathsf{SP}(\hat\chi)]_H(x,y)\land[T_\mathsf{SP}(\hat\chi)]_G(u,u)=[T_\mathsf{SP}(\hat\chi)]_H(x,x)\\
        &\qquad\land [T_\mathsf{SP}(\hat\chi)]_G(v,v)=[T_\mathsf{SP}(\hat\chi)]_H(y,y)\land[T_\mathsf{SP}(\hat\chi)]_G(v,u)=[T_\mathsf{SP}(\hat\chi)]_H(y,x)\land \mathbb I[u=v]= \mathbb I[x=y]\\
        &\qquad\land \ldblbrace[T_\mathsf{SP}(\hat\chi)]_G(u,w):w\in V_G\rdblbrace=\ldblbrace[T_\mathsf{SP}(\hat\chi)]_H(x,z):z\in V_H\rdblbrace\\
        &\qquad\land \ldblbrace[T_\mathsf{SP}(\hat\chi)]_G(v,w):w\in V_G\rdblbrace=\ldblbrace[T_\mathsf{SP}(\hat\chi)]_H(y,z):z\in V_H\rdblbrace\\
        &\qquad\land \ldblbrace[T_\mathsf{SP}(\hat\chi)]_G(w,u):w\in V_G\rdblbrace=\ldblbrace[T_\mathsf{SP}(\hat\chi)]_H(z,x):z\in V_H\rdblbrace\\
        &\qquad\land \ldblbrace[T_\mathsf{SP}(\hat\chi)]_G(w,v):w\in V_G\rdblbrace=\ldblbrace[T_\mathsf{SP}(\hat\chi)]_H(z,y):z\in V_H\rdblbrace\\
        &\qquad\land \ldblbrace[T_\mathsf{SP}(\hat\chi)]_G(w,w):w\in V_G\rdblbrace=\ldblbrace[T_\mathsf{SP}(\hat\chi)]_H(z,z):z\in V_H\rdblbrace\\
        &\qquad\land \ldblbrace[T_\mathsf{SP}(\hat\chi)]_G(w,w'):w,w'\in V_G\rdblbrace=\ldblbrace[T_\mathsf{SP}(\hat\chi)]_H(z,z'):z,z'\in V_H\rdblbrace\\
        &\implies [(T_\mathsf{IGN}\circ T_\mathsf{SP})(\hat\chi)]_G(u,v)=[(T_\mathsf{IGN}\circ T_\mathsf{SP})(\hat\chi)]_H(x,y)
    \end{align*}
    where in the second step we use the definition of $T_{\mathsf{EP},\mM}$ and also \cref{thm:epwl_atp,thm:epwl_first_layer_encode_Puu}; in the third step we use the definition of $T_{\mathsf{EP},\mM}$ again; in the fourth step we use the induction hypothesis and the fact that $\gP^\gM_G(u,v)=\gP^\gM_G(v,u)$ for all $G\in\gG$ and $u,v\in V_G$; in the last step we use the definition of 2-IGN (\cref{eq:2-ign_wl}).
    
    We next prove that $[(T_\mathsf{IGN}\circ T_\mathsf{SP})(\hat\chi)]_G(u,v)=[(T_\mathsf{IGN}\circ T_\mathsf{SP})(\hat\chi)]_H(x,y)\implies [(T_\mathsf{SP}\circ T_\mathsf{SIGN})(\hat\chi)]_G(u,v)=[(T_\mathsf{SP}\circ T_\mathsf{SIGN})(\hat\chi)]_H(x,y)$. This is because
    \begin{align*}
        &\quad[(T_\mathsf{IGN}\circ T_\mathsf{SP})(\hat\chi)]_G(u,v)=[(T_\mathsf{IGN}\circ T_\mathsf{SP})(\hat\chi)]_H(x,y)\\
        &\implies \Lambda^\mM(G)=\Lambda^\mM(H)\land [T_\mathsf{IGN}(\hat\chi(\lambda,\cdot,\cdot))]_G(u,v)=[T_\mathsf{IGN}(\hat\chi(\lambda,\cdot,\cdot))]_H(x,y) \quad \forall \lambda\in\Lambda^\mM(G)\\
        &\implies [(T_\mathsf{SP}\circ T_\mathsf{SIGN})(\hat\chi)]_G(u,v)=[(T_\mathsf{SP}\circ T_\mathsf{SIGN})(\hat\chi)]_H(x,y),
    \end{align*}
    where in the first step we use the following observations: $(\mathrm{i})$ $\hat\chi\preceq\chi^{\gP^\mM}$, which implies that $\hat\chi_G(\lambda,u,v)=\hat\chi_H(\mu,x,y)\implies \lambda=\mu$; $(\mathrm{ii})$  $(T_\mathsf{IGN}\circ T_\mathsf{SP})(\hat\chi)\preceq T_\mathsf{SP}\circ \chi^{\gP^\mM}\equiv \gP^\mM$ and thus $\Lambda^\mM(G)=\Lambda^\mM(H)$. We thus conclude the induction step.
\endgroup
\end{proof}

\begin{lemma}
\label{thm:spectral_ign_lemma<}
    Let $\chi^0\in\gM_1$ be the constant mapping. Then, $(T_\mathsf{GP}\circ T_{\mathsf{EP},\mM}^\infty)(\chi^0)\preceq (T_\mathsf{GP}\circ T_\mathsf{P2}\circ T_\mathsf{SP}\circ T_\mathsf{SIGN}^\infty)(\chi^{\gP^\mM})$.
\end{lemma}
\begin{proof}
    Based on \cref{thm:spectral_ign_aux_lemma}, it suffices to prove that $(T_\mathsf{GP}\circ T_{\mathsf{EP},\mM}^\infty)(\chi^0)\equiv (T_\mathsf{GP}\circ  T_\mathsf{P2}\circ T_{\times,\gP^\mM}\circ T_{\mathsf{EP},\mM}^\infty)(\chi^0)$, where $T_{\times,\gP^\mM}$ is defined in \cref{thm:spectral_ign_aux_lemma}. Denote $\chi=T_{\mathsf{EP},\mM}^\infty(\chi^0)$. Let $G,H\in\gG$ be any graphs such that  $[T_\mathsf{GP}(\chi)](G)=[T_\mathsf{GP}(\chi)](H)$. By definition of $T_\mathsf{GP}$,
    \begin{equation}
        \ldblbrace\chi_G(u):u\in V_G\rdblbrace=\ldblbrace\chi_H(x):x\in V_H\rdblbrace.
    \end{equation}
    Since $\chi\equiv T_{\mathsf{EP},\mM}\circ \chi$,
    \begin{equation}
        \ldblbrace(\chi_G(u),\ldblbrace(\chi_G(v),\gP^\mM_G(u,v)):v\in V_G\rdblbrace): u\in V_G\rdblbrace=\ldblbrace(\chi_H(x),\ldblbrace (\chi_H(y),\gP^\mM_H(x,y)):y\in V_H\rdblbrace):x\in V_H\rdblbrace.
    \end{equation}
    Equivalently,
    \begin{equation}
        \ldblbrace\ldblbrace(\chi_G(u),\chi_G(v),\gP^\mM_G(u,v)):v\in V_G\rdblbrace: u\in V_G\rdblbrace=\ldblbrace\ldblbrace(\chi_H(x),\chi_H(y),\gP^\mM_H(x,y)):y\in V_H\rdblbrace:x\in V_H\rdblbrace.
    \end{equation}
    This implies that $[(T_\mathsf{GP}\circ  T_\mathsf{P2}\circ T_{\times,\gP^\mM})(\chi)](G)=[(T_\mathsf{GP}\circ  T_\mathsf{P2}\circ T_{\times,\gP^\mM})(\chi)](H)$, concluding the proof.
\end{proof}

Combining \cref{thm:spectral_ign_wl,thm:spectral_ign_lemma1,thm:spectral_ign_lemma>,thm:spectral_ign_lemma<}, we conclude the proof of \cref{thm:spectral_ign}.

\subsection{Proof of \cref{thm:basisnet_basic,thm:spe}}
\label{sec:proof_basisnet}

This sections aims to prove \cref{thm:basisnet_basic,thm:spe}. We will first give a brief introduction of BasisNet. The BasisNet architecture is composed of two parts: an eigenspace encoder $\Phi:\mathbb R^{m\times n\times n}\to \mathbb R^{n\times d}$ and a top graph encoder $\rho:\mathbb R^{n\times d}\to \mathbb R^{d'}$. Here, we consider the standard setting where $\rho$ is a messgage-passing GNN that takes node features as inputs and outputs a graph representation invariant to node permutation. We assume that the expressive power of $\rho$ is bounded by the classic 1-WL test.

We next describe the design of the eigenspace encoder $\Phi$. Given graph matrix $\mM$ and graph $G$, let $\mM_G=\sum_{i=1}^m \lambda_i\mP_i$ be the eigen-decomposition of $\mM_G$, where $\lambda_1<\cdots<\lambda_m$ are eigenvalues of $\mM_G$. For each eigenspace, BasisNet processes the projection matrix $\mP_i$ using a 2-IGN $\mathsf{IGN}^{(d_i)}:\mathbb R^{n\times n}\to\mathbb R^{n}$, where $d_i$ is the multiplicity of eigenvalue $\lambda_i$. The output of $\Phi$ is then defined as
\begin{equation}
    \Phi(G)=[\mathsf{IGN}^{(d_1)}(\mP_1),\cdots,\mathsf{IGN}^{(d_m)}(\mP_m),\mathbf 0,\cdots,\mathbf 0]\in\mathbb R^{n\times d},
\end{equation}
where $[\ ]$ denotes the concatenation. When the number of eigenspaces is less than the output dimension $d$, zero-padding is applied. Note that BasisNet processes different projections using different IGNs if their multiplicities differ.

\textbf{Color refinement algorithms for Siamese IGN and BasisNet.} Similar to Spectral IGN, we can write the corresponding color refinement algorithms for the two GNN architectures. We first define the initial color mapping $\chi^{\mathsf{Basis},\gP^\mM}\in\gM_2^{\gP^\mM}$ as follows: $\chi^{\mathsf{Basis},\gP^\mM}_G(\lambda,u,v)=(d_\lambda,\mP^\mM_\lambda(u,v))$ for all $(G^{uv},\lambda)\in\gG_2^{\gP^\mM}$, where $\mP^\mM_\lambda$ is the projection matrix associated with eigenvalue $\lambda$ and $d_\lambda$ is the multiplicity of eigenvalue $\lambda$. Here, we encode the multiplicity in $\chi^{\mathsf{Basis},\gP^\mM}$ because BasisNet uses different IGNs for different eigenvalue multiplicities. We then define several color transformations:
\begin{itemize}[topsep=0pt,leftmargin=20pt]
    \setlength{\itemsep}{0pt}
    \item \textbf{Siamese IGN color refinement.} Define $T_\mathsf{Siam}:\gM_2^{\gP^\mM}\to\gM_2^{\gP^\mM}$ such that for any color mapping $\chi\in\gM_2^{\gP^\mM}$ and $(G^{uv},\lambda)\in\gG_2^{\gP^\mM}$,
    \begin{equation}
    \label{eq:siamese_ign_wl}
        [T_\mathsf{Siam}(\chi)]_G(\lambda,u,v)=[T_\mathsf{IGN}(\chi(\lambda,\cdot,\cdot))]_G(u,v),
    \end{equation}
    where $T_\mathsf{IGN}$ is defined in \cref{eq:2-ign_wl}.
    \item \textbf{BasisNet pooling.} Define $T_\mathsf{BP}:\gM_2^{\gP^\mM}\to\gM_1^{\gP^\mM}$ such that for any color mapping $\chi\in\gM_2^{\gP^\mM}$ and $(G^{u},\lambda)\in\gG_1^{\gP^\mM}$,
    \begin{equation}
    \label{eq:basisnet_pooling_wl}
    \begin{aligned}
        [T_\mathsf{BP}(\chi)]_G(\lambda,u)=\hash(&\chi_G(\lambda,u,u),\ldblbrace\chi_G(\lambda,u,v):v\in V_G\rdblbrace,\ldblbrace\chi_G(\lambda,v,u):v\in V_G\rdblbrace,\\
        &\ldblbrace\chi_G(\lambda,v,v):v\in V_G\rdblbrace,\ldblbrace\chi_G(\lambda,v,w):v,w\in V_G\rdblbrace).
    \end{aligned}
    \end{equation}
    One can see that \cref{eq:basisnet_pooling_wl} has 5 aggregations inside the hash function, which matches the number of orthogonal bases in \citet{maron2019invariant}.
    \item \textbf{Spectral pooling.} Define $T_\mathsf{SP1}:\gM_1^{\gP^\mM}\to\gM_1$ such that for any color mapping $\chi\in\gM_1^{\gP^\mM}$ and $G^{u}\in\gG_1$,
    \begin{equation}
    \begin{aligned}
        [T_\mathsf{SP1}(\chi)]_G(u)=\hash(\ldblbrace\chi_G(\lambda,u):\lambda\in \Lambda^\mM(G)\rdblbrace).
    \end{aligned}
    \end{equation}
    Similarly, define $T_\mathsf{SP2}:\gM_2^{\gP^\mM}\to\gM_2$ such that for any color mapping $\chi\in\gM_2^{\gP^\mM}$ and $G^{uv}\in\gG_2$,
    \begin{equation}
    \begin{aligned}
        [T_\mathsf{SP2}(\chi)]_G(u,v)=\hash(\ldblbrace\chi_G(\lambda,u,v):\lambda\in \Lambda^\mM(G)\rdblbrace).
    \end{aligned}
    \end{equation}
    \item \textbf{Joint pooling.} This has been defined in \cref{eq:joint_pooling}.
    \item \textbf{Diagonal pooling.} Define $T_\mathsf{D}:\gM_2\to\gM_1$ such that for any color mapping $\chi\in\gM_2$ and rooted graph $G^{u}$,
    \begin{equation}
    \label{eq:diagonal_pooling}
        [T_\mathsf{D}(\chi)]_G(u)=\chi_G(u,u).
    \end{equation}
    \item \textbf{1-WL refinement.} This has been defined in \cref{eq:wl}.
    \item \textbf{Global pooling.} This has been defined in \cref{eq:global_pooling}.
\end{itemize}
We are ready to define the color mappings corresponding to the whole algorithms:
\begin{itemize}[topsep=0pt,leftmargin=20pt]
    \setlength{\itemsep}{0pt}
    \item \textbf{Weak Spectral IGN}: the color mapping is defined as $(T_\mathsf{JP}\circ T_\mathsf{SP2}\circ T_\mathsf{Siam}^\infty)(\chi^{\gP^\mM})$.
    \item \textbf{BasisNet}: the color mapping is defined as $(T_\mathsf{GP}\circ T_\mathsf{WL}\circ T_\mathsf{SP1}\circ T_\mathsf{BP}\circ T_\mathsf{Siam}^\infty)(\chi^{\mathsf{Basis},\gP^\mM})$.
\end{itemize}
Similar to the previous analysis, we can prove that the above two color mappings upper bound the expressive power of the corresponding GNN models. Below, it suffices to prove the following key lemma:
\begin{lemma}
    For any graph matrix $\mM$, $(T_\mathsf{JP}\circ T_\mathsf{SP2}\circ T_\mathsf{Siam}^\infty)(\chi^{\gP^\mM})\preceq (T_\mathsf{GP}\circ T_\mathsf{WL}\circ T_\mathsf{SP1}\circ T_\mathsf{BP}\circ T_\mathsf{Siam}^\infty)(\chi^{\mathsf{Basis},\gP^\mM})$.
\end{lemma}
\begin{proof}
    The proof will be decomposed into a series of steps. We first prove that $T_\mathsf{Siam}^\infty(\chi^{\gP^\mM})\preceq T_\mathsf{Siam}^\infty(\chi^{\mathsf{Basis},\gP^\mM})$. It suffices to prove that $T_\mathsf{Siam}(\chi^{\gP^\mM})\preceq \chi^{\mathsf{Basis},\gP^\mM}$. Pick any graphs $G,H\in\gG$, eigenvalues $\lambda\in\Lambda^\mM(G), \mu\in\Lambda^\mM(H)$, and vertices $u,v\in V_G$, $x,y\in V_H$. Then, by definition of $T_\mathsf{Siam}$, $[T_\mathsf{Siam}(\chi^{\gP^\mM})]_G(\lambda,u,v)=[T_\mathsf{Siam}(\chi^{\gP^\mM})]_H(\mu,x,y)$ implies that
    \begin{equation}
        \chi^{\gP^\mM}_G(\lambda,u,v)=\chi^{\gP^\mM}_H(\mu,x,y)\land \ldblbrace\chi^{\gP^\mM}_G(\lambda,w,w):w\in V_G\rdblbrace=\ldblbrace\chi^{\gP^\mM}_H(\mu,z,z):z\in V_H\rdblbrace.
    \end{equation}
    Therefore, $\tr([\mP^\mM_\lambda]_G)=\tr([\mP^\mM_\mu]_H)$, where $\tr(\cdot)$ denotes the matrix trace. Noting that $\tr([\mP^\mM_\lambda]_G)$ is exactly the multiplicity of eigenvalue $\lambda$ for graph matrix $\mM_G$, we have $\chi^{\mathsf{Basis},\gP^\mM}_G(\lambda,u,v)=\chi^{\mathsf{Basis},\gP^\mM}_H(\mu,x,y)$.

    We then prove that $T_\mathsf{JP}\circ T_\mathsf{SP2}\circ T_\mathsf{Siam}^\infty\preceq T_\mathsf{GP}\circ T_\mathsf{WL}\circ T_\mathsf{D}\circ T_\mathsf{SP2}\circ T_\mathsf{Siam}^\infty$. Pick any initial color mapping $\chi^0\in\gM_2^{\gP^\mM}$ and let $\chi=T_\mathsf{Siam}^\infty(\chi^0)$. Note that $\chi\equiv T_\mathsf{Siam}(\chi)$. We will prove that $(T_\mathsf{JP}\circ T_\mathsf{SP2})(\chi)\preceq (T_\mathsf{GP}\circ T_\mathsf{WL}\circ T_\mathsf{D}\circ T_\mathsf{SP2})(\chi)$. Pick any graphs $G,H\in\gG$. We have
    \begin{align*}
        &\quad [(T_\mathsf{JP}\circ T_\mathsf{SP2})(\chi)](G)=[(T_\mathsf{JP}\circ T_\mathsf{SP2})(\chi)](H)\\
        &\implies \ldblbrace\ldblbrace\chi_G(\lambda,u,v):\lambda\in\Lambda^\mM(G)\rdblbrace: u,v\in V_G\rdblbrace = \ldblbrace\ldblbrace\chi_H(\mu,x,y):\mu\in\Lambda^\mM(H)\rdblbrace:x,y\in V_H\rdblbrace\\
        &\implies \ldblbrace(\ldblbrace\chi_G(\lambda,u,v):\lambda\in\Lambda^\mM(G)\rdblbrace,\atp_G(u,v)): u,v\in V_G\rdblbrace\\
        &\qquad = \ldblbrace(\ldblbrace\chi_H(\mu,x,y):\mu\in\Lambda^\mM(H)\rdblbrace,\atp_H(x,y)):x,y\in V_H\rdblbrace\\
        &\implies \ldblbrace(\ldblbrace\chi_G(\lambda,v,v):\lambda\in\Lambda^\mM(G)\rdblbrace,\atp_G(u,v)):u,v\in V_G\rdblbrace\\
        &\qquad= \ldblbrace(\ldblbrace\chi_H(\mu,y,y):\mu\in\Lambda^\mM(H)\rdblbrace,\atp_H(x,y)):x,y\in V_H\rdblbrace\\
        &\implies [(T_\mathsf{GP}\circ T_\mathsf{WL}\circ T_\mathsf{D}\circ T_\mathsf{SP2})(\chi)](G)=[(T_\mathsf{GP}\circ T_\mathsf{WL}\circ T_\mathsf{D}\circ T_\mathsf{SP2})(\chi)](H),
    \end{align*}
    where the second step is based on \cref{thm:epwl_atp}, and the third step is based on the definition of $T_\mathsf{Siam}$. This proves that $T_\mathsf{JP}\circ T_\mathsf{SP2}\circ T_\mathsf{Siam}^\infty\preceq T_\mathsf{GP}\circ T_\mathsf{WL}\circ T_\mathsf{D}\circ T_\mathsf{SP2}\circ T_\mathsf{Siam}^\infty$.

    We next prove that $T_\mathsf{D}\circ T_\mathsf{SP2}\circ T_\mathsf{Siam}^\infty\preceq T_\mathsf{SP1}\circ T_\mathsf{BP}\circ T_\mathsf{Siam}^\infty$. Pick any graphs $G,H\in\gG$ and vertices $u\in V_G$, $x\in V_H$. We have
    \begin{align*}
        &\quad [(T_\mathsf{D}\circ T_\mathsf{SP2})(\chi)](G)=[(T_\mathsf{D}\circ T_\mathsf{SP2})(\chi)](H)\\
        &\implies \ldblbrace\chi_G(\lambda,u,u):\lambda\in\Lambda^\mM(G)\rdblbrace= \ldblbrace\chi_H(\mu,x,x):\mu\in\Lambda^\mM(H)\rdblbrace\\
        &\implies \ldblbrace\ldblbrace\chi_G(\lambda,u,v):v\in V_G\rdblbrace:\lambda\in\Lambda^\mM(G)\rdblbrace = \ldblbrace\ldblbrace\chi_H(\mu,x,y):y\in V_H\rdblbrace:\mu\in\Lambda^\mM(H)\rdblbrace\\
        &\implies [(T_\mathsf{SP1}\circ T_\mathsf{BP})(\chi)](G)=[(T_\mathsf{SP1}\circ T_\mathsf{BP})(\chi)](H),
    \end{align*}
    where the second step is based on the definition of $T_\mathsf{Siam}$. This proves that $T_\mathsf{D}\circ T_\mathsf{SP2}\circ T_\mathsf{Siam}^\infty\preceq T_\mathsf{SP1}\circ T_\mathsf{BP}\circ T_\mathsf{Siam}^\infty$.

    We conclude the proof by combining the above relations with \cref{thm:refinement1,thm:refinement2}.
\end{proof}

We next turn to the proof of \cref{thm:spe}, which is almost the same as the case of BasisNet. Below, we will define the equivalent color refinement algorithm for SPE. The initial color mapping associated with SPE is simply $\gP^\mM$. Then, the architecture refines $\gP^\mM$ by using color transformation $T_\mathsf{IGN}$ defined in \cref{eq:2-ign_wl}. The remaining procedure is the same as BasisNet. Combined these together, the color mapping corresponding to the whole algorithm can be written as $(T_\mathsf{GP}\circ T_\mathsf{WL}^\infty\circ T_\mathsf{P2}\circ T_\mathsf{IGN}^\infty)(\gP^\mM)$. Then, it suffices to prove the following two equivalence relations: $(\mathrm{i})$ $ T_\mathsf{IGN}^\infty(\gP^\mM)\equiv (T_\mathsf{SP2}\circ T_\mathsf{SIGN}^\infty)(\chi^{\gP^\mM})$; $(\mathrm{ii})$ $T_\mathsf{WL}\circ T_\mathsf{P2}\circ T_\mathsf{SP2}\circ T_\mathsf{SIGN}^\infty\equiv T_\mathsf{P2}\circ T_\mathsf{SP2}\circ T_\mathsf{SIGN}^\infty$. The proof procedure is alomst the same as in \cref{sec:proof_spectral_ign,sec:proof_basisnet} and we omit it here.

\subsection{Discussions with other architectures}
\label{sec:proof_other_architectures}

\textbf{Graphormer \citep{ying2021transformers}, Graphormer-GD \citep{zhang2023rethinking}, and GraphiT \citep{mialon2021graphit}.} \citet{zhang2023rethinking} has shown that the expressive power of these architectures is inherently bounded by GD-WL with different distances. Here, Graphormer uses SPD, Graphormer-GD uses both SPD and RD, while the distance used in GraphiT has the form $d_G(u,v)=\sum_{i=1}^m \phi(\lambda_i)\mP_i(u,v)$, where $\mP_i$ is the projection matrix associated with eigenvalue $\lambda$ for graph matrix $\hat\mL$, and $\phi:\mathbb R\to\mathbb R$ is a general function. Therefore, based on \cref{thm:distance} and the proof, it is straightforward to see that all these architectures are bounded by EPWL.
 
\textbf{PEG \citep{wang2022equivariant}}. Given graph $G$, PEG maintains a feature vector $\vh^{(l)}(u)\in\mathbb R^d$ for each node $u\in V_G$ in each layer $l$, and the feature is updated by the following formula:
\begin{equation}
    \vh^{(l+1)}(u)=\psi\left(\sum_{v\in N_G(u)}\phi(\|\vz_G(u)-\vz_G(v)\|)\mW\vh^{(l)}(v)\right),
\end{equation}
where $\phi:\mathbb R\to\mathbb R$ and $\psi:\mathbb R^d\to\mathbb R^d$ are arbitrary functions, $\mW\in\mathbb R^{d\times d}$ is a parameterized weight matrix, and $\vz_G(u)\in\mathbb R^k$ is the positional encoding corresponding to the top $k$ eigenvectors at node $u$. Here, we assume that the number $k$ is chosen such that $\|\vz_G(u)-\vz_G(v)\|$ is unique for all graphs $G$ of interest (i.e., no ambiguity problem).

We will show that the expressive power of PEG is bounded by EPWL. To obtain this result, note that $\|\vz_G(u)-\vz_G(v)\|^2=\sum_{i=1}^k (\vz_{G,i}(u))^2+(\vz_{G,i}(v))^2-2\vz_{G,i}(u)\vz_{G,i}(v)$. Since $\|\vz_G(u)-\vz_G(v)\|$ is unique, the span of top $k$ eigenvectors must be equivalent to the direct sum of the eigenspaces corresponding to top $k'$ eigenvalues $\lambda_1>\lambda_2>\cdots>\lambda_{k'}$ for some $k'\le k$. It follows that $\|\vz_G(u)-\vz_G(v)\|^2=\sum_{i=1}^{k'}\mP_i(u,u)+\mP_i(v,v)-2\mP_i(u,v)$ where $\mP_i$ is the projection onto the eigenspace corresponding to eigenvalue $\lambda_i$. Therefore, the expressive power of PEG is bounded by the color refinement algorithm $(T_\mathsf{GP}\circ T_\mathsf{PEG}^\infty)(\chi^0)$ with $\chi^0$ the initial color mapping and $T_\mathsf{PEG}$ the color transformation defined below:
\begin{equation}
    [T_\mathsf{PEG}(\chi)]_G(u)=\hash\left(\ldblbrace(\chi_G(v),\gP^\mM_G(u,u),\gP^\mM_G(v,v),\gP^\mM_G(u,v)):v\in V_G\rdblbrace\right).
\end{equation}
We will prove that $(T_\mathsf{GP}\circ T_{\mathsf{EP},\mM}^\infty)(\chi^0)\preceq (T_\mathsf{GP}\circ T_\mathsf{PEG}^\infty)(\chi^0)$. Based on \cref{thm:refinement1,thm:refinement2}, it suffices to prove that $T_{\mathsf{EP},\mM}^\infty\preceq T_\mathsf{PEG}\circ T_{\mathsf{EP},\mM}^\infty$. Denote $\chi=T_{\mathsf{EP},\mM}^\infty(\chi^0)$ and note that $\chi\equiv T_{\mathsf{EP},\mM}^\infty(\chi)$. Pick any graphs $G,H\in\gG$ and vertices $u\in V_G$, $x\in V_H$. We have
\begin{align*}
    &\quad\chi_G(u)=\chi_H(x)\\
    &\implies \gP^\mM_G(u,u)=\gP^\mM_H(x,x)\land \ldblbrace(\chi_G(v),\gP^\mM_G(u,v)):v\in V_G\rdblbrace=\ldblbrace(\chi_H(y),\gP^\mM_H(x,y)):y\in V_H\rdblbrace\\
    &\implies \ldblbrace(\chi_G(v),\gP^\mM_G(u,u),\gP^\mM_G(v,v),\gP^\mM_G(u,v)):v\in V_G\rdblbrace=\ldblbrace(\chi_H(y),\gP^\mM_H(x,x),\gP^\mM_H(y,y),\gP^\mM_H(x,y)):y\in V_H\rdblbrace\\
    &\implies [T_\mathsf{PEG}(\chi)]_G(u)=[T_\mathsf{PEG}(\chi)]_H(x),
\end{align*}
where in the first and second steps we use \cref{thm:epwl_distance_basic}. This concludes the proof that the expressive power of PEG is bounded by EPWL.

\textbf{GIRT \citep{ma2023graph}}. Given graph $G$, GIRT maintains a feature vector for both vertices and vertex pairs. Denote by $\vh^{(l)}(u)\in\mathbb R^d$ the feature of node $u\in V_G$ in layer $l$, and denote by $\vh^{(l)}(u,v)\in\mathbb R^{d'}$ the feature of node pair $(u,v)\in V_G^2$ in layer $l$. The features are updated by the following formula:
\begin{align}
    &\vh^{(l+1)}_G(u,v)=\sigma\left(\rho\left(\left(\mW_{\mathsf{Q}}\vh_G(u)+\mW_{\mathsf{K}}\vh_G(v)\right)\odot\mW_\mathsf{Ew}\vh^{(l)}_G(u,v)\right)+\mW_{\mathsf{Eb}} \vh^{(l)}_G(u,v)\right),\\
    &\alpha^{(l+1)}_G(u,v)=\text{Softmax}_{j\in V_G}(\mW_\mathsf{A}\vh^{(l+1)}_G(u,v)), \\
    &\vh^{(l+1)}_G(u)=\sum_{v\in V_G}\alpha^{(l+1)}_G(u,v)\cdot(\mW_\mathsf{V}\vh^{(l)}_G(v)+\mathbf{W}_{\mathsf{Ev}}\vh^{(l+1)}_G(u,v)),
\end{align}
where the initial feature is defined as
\begin{align*}
    &\vh^{(0)}_G(u,v)=[(\mD^{-1}\mA)^0_G(u,v),(\mD^{-1}\mA)^1_G(u,v),\cdots,(\mD^{-1}\mA)^{K}_G(u,v)],\\
    &\vh^{(0)}_G(u)=\vh^{(0)}_G(u,u).
\end{align*}
One can easily write the corresponding color refinement algorithm that upper bounds of the expressive power of GIRT. Formally, it can be expressed as $(T_\mathsf{GP}\circ T_\mathsf{D}\circ T_\mathsf{GIRT}^\infty)(\chi^\mathsf{GIRT})$, where $T_\mathsf{D}$ is defined in \cref{eq:diagonal_pooling}, the initial color mapping $\chi^\mathsf{GIRT}$ is simply the multi-dimensional PageRank distance, and $T_\mathsf{GIRT}:\gM_2\to\gM_2$ is the color refinement defined below:
\begin{equation}
    [T_\mathsf{GIRT}(\chi)]_G(u,v)=\left\{\begin{array}{ll}
        \hash(\chi_G(u,v),\chi_G(u,u),\chi_G(v,v)) & \text{if }u\neq v, \\
        \hash(\chi_G(u,u), \ldblbrace(\chi_G(u,v),\chi_G(v,v)):v\in V_G\rdblbrace &\text{if }u=v. 
    \end{array}\right.
\end{equation}
We will prove that $(T_\mathsf{GP}\circ T_{\mathsf{EP},\mM}^\infty)(\chi^0)\preceq (T_\mathsf{GP}\circ T_\mathsf{D}\circ T_\mathsf{GIRT}^\infty)(\chi^\mathsf{GIRT})$. Define color transformation $T_{\times,\gP^\mM}:\gM_1\to\gM_2$ such that for any color mapping $\chi\in\gM_1$ and rooted graph $G^{uv}\in\gG_2$,
\begin{equation}
    [T_{\times,\gP^\mM}(\chi)]_G(u,v)=\hash(\chi_G(u),\chi_G(v),\gP^\mM_G(u,v),\gP^\mM_G(u,u),\gP^\mM_G(v,v)).
\end{equation}
Note that $T_\mathsf{GP}\circ T_{\mathsf{EP},\mM}^\infty\equiv T_\mathsf{GP}\circ T_\mathsf{D}\circ T_{\times,\gP^\mM} \circ T_{\mathsf{EP},\mM}^\infty$ by \cref{thm:epwl_distance_basic}. Also, $T_{\times,\gP^\mM}(\chi^0)\preceq \chi^\mathsf{GIRT}$ due to \cref{thm:epwl_rd}. Therefore, it suffices to prove that $T_{\times,\gP^\mM}\circ T_{\mathsf{EP},\mM}^\infty\preceq T_\mathsf{GIRT}^\infty\circ T_{\times,\gP^\mM}$. Based on \cref{thm:refinement2}, it suffices to prove that $T_\mathsf{GIRT}\circ T_{\times,\gP^\mM}\circ T_{\mathsf{EP},\mM}^\infty\equiv T_{\times,\gP^\mM}\circ T_{\mathsf{EP},\mM}^\infty$. Denote $\chi=T_{\mathsf{EP},\mM}^\infty(\chi^0)$, where $\chi^0\in\gM_1$ is any initial color mapping. Pick any graphs $G,H\in\gG$ and vertices $u,v\in V_G$, $x,y\in V_H$ such that $[T_{\times,\gP^\mM}(\chi)]_G(u,v)=[T_{\times,\gP^\mM}(\chi)]_H(x,y)$. We have:
\begin{itemize}[topsep=0pt,leftmargin=20pt]
    \setlength{\itemsep}{0pt}
    \item If $u=v$, then $x=y$. 
    \begin{align*}
        &\quad[T_{\times,\gP^\mM}(\chi)]_G(u,u)=[T_{\times,\gP^\mM}(\chi)]_H(x,x)\\
        &\implies \chi_G(u)=\chi_H(x),\gP^\mM_G(u,u)=\gP^\mM_H(x,x)\\
        &\implies\ldblbrace(\chi_G(u),\chi_G(v),\gP^\mM_G(u,v)):v\in V_G\rdblbrace=\ldblbrace(\chi_H(x),\chi_H(y),\gP^\mM_H(x,y)):y\in V_H\rdblbrace\\
        &\implies\ldblbrace(\chi_G(u),\gP^\mM_G(u,u),\chi_G(v),\gP^\mM_G(v,v),\gP^\mM_G(u,v)):v\in V_G\rdblbrace\\
        &\qquad=\ldblbrace(\chi_H(x),\gP^\mM_H(x,x),\chi_H(y),\gP^\mM_H(y,y),\gP^\mM_H(x,y)):y\in V_H\rdblbrace\\
        &\implies\ldblbrace ([T_{\times,\gP^\mM}(\chi)]_G(u,v),[T_{\times,\gP^\mM}(\chi)]_G(v,v)):v\in V_G\rdblbrace\\
        &\qquad =\ldblbrace ([T_{\times,\gP^\mM}(\chi)]_H(x,y),[T_{\times,\gP^\mM}(\chi)]_H(y,y)):y\in V_H\rdblbrace\\
        &\implies[(T_\mathsf{GIRT}\circ T_{\times,\gP^\mM})(\chi)]_G(u,u)=[(T_\mathsf{GIRT}\circ T_{\times,\gP^\mM})(\chi)]_G(x,x).
    \end{align*}
    \item If $u\neq v$, then $x\neq y$.
    \begin{align*}
        &\quad[T_{\times,\gP^\mM}(\chi)]_G(u,v)=[T_{\times,\gP^\mM}(\chi)]_H(x,y)\\
        &\implies \chi_G(u)=\chi_H(x),\chi_G(v)=\chi_H(y),\gP^\mM_G(u,v)=\gP^\mM_H(x,y),\gP^\mM_G(u,u)=\gP^\mM_H(x,x),\gP^\mM_G(v,v)=\gP^\mM_H(y,y)\\
        &\implies [(T_\mathsf{GIRT}\circ T_{\times,\gP^\mM})(\chi)]_G(u,v)=[(T_\mathsf{GIRT}\circ T_{\times,\gP^\mM})(\chi)]_G(x,y).
    \end{align*}
\end{itemize}
Here, in the above derivations we use \cref{thm:epwl_distance_basic}. We have concluded the proof that the expressive power of GIRT is bounded by EPWL.

\textbf{Spectral PPGN and Spectral $k$-IGN.} Based on \citet{maron2019provably}, PPGN can mimic the 2-FWL test \citep{cai1992optimal}, and Spectral $k$-IGN
can mimic the $k$-WL test \citep{grohe2017descriptive}. Let $T_{\mathsf{WL}(k)}:\gM_k\to\gM_k$ and $T_{\mathsf{FWL}(k)}:\gM_k\to\gM_k$ be the color refinements associated with $k$-WL and $k$-FWL, respectively, and let $\chi^0_k\in\gM_k$ be the initial color mapping in $k$-WL and $k$-FWL. The color refinement algorithms corresponding to $k$-WL and $k$-FWL can then be described as $(T_{\mathsf{JP}(k)}\circ T_{\mathsf{WL}(k)}^\infty)(\chi^0_k)$ and $(T_{\mathsf{JP}(k)}\circ T_{\mathsf{FWL}(k)}^\infty)(\chi^0_k)$, respectively, where $T_{\mathsf{JP}(k)}$ is defined in \cref{eq:k-joint_pooling}. \citet{maron2019provably} proved that with sufficiently layers, the features of vertex $k$-tuples computed by $k$-IGN is finer than $T_{\mathsf{WL}(k)}^\infty(\chi^0_k)$, and the features of vertex pairs computed by PPGN is finer than $T_{\mathsf{FWL}(2)}^\infty(\chi^0_2)$. Later, \citet{azizian2021expressive} proved that the features of vertex pairs computed by PPGN is also bounded by (and thus as fine as) $T_{\mathsf{FWL}(2)}^\infty(\chi^0_2)$ (see Lemma 12 in their paper). Finally, \citet{geerts2022expressiveness} proved that the features of vertex $k$-tuples computed by $k$-IGN is bounded by (and thus as fine as) $T_{\mathsf{WL}(k)}^\infty(\chi^0_k)$ (see Lemma E.1 in their paper). 

We now define the color refinement algorithms for Spectral PPGN and Spectral $k$-IGN, which are as expressive as the corresponding GNN architectures based on the results of \citet{maron2019provably,azizian2021expressive,geerts2022expressiveness}. First define the following color transformations:
\begin{itemize}[topsep=0pt,leftmargin=20pt]
    \setlength{\itemsep}{0pt}
    \item \textbf{Spectral $k$-IGN color refinement.} Define $T_{\mathsf{SIGN}(k)}:\gM_k^{\gP^\mM}\to\gM_k^{\gP^\mM}$ such that for any color mapping $\chi\in\gM_k^{\gP^\mM}$ and $(G^{\vu},\lambda)\in\gG_k^{\gP^\mM}$,
    \begin{equation}
        [T_{\mathsf{SIGN}(k)}(\chi)]_G(\lambda,\vu)=\hash([T_{\mathsf{WL}(k)}(\chi(\lambda,\cdots))]_G(\vu),[T_{\mathsf{WL}(k)}(T_{\mathsf{SP}(k)}(\chi))]_G(\vu)).
    \end{equation}
    \item \textbf{Spectral PPGN color refinement.} Define $T_{\mathsf{PPGN}}:\gM_2^{\gP^\mM}\to\gM_2^{\gP^\mM}$ such that for any color mapping $\chi\in\gM_2^{\gP^\mM}$ and $(G^{\vu},\lambda)\in\gG_2^{\gP^\mM}$,
    \begin{equation}
        [T_{\mathsf{SPPGN}}(\chi)]_G(\lambda,\vu)=\hash([T_{\mathsf{FWL}(2)}(\chi(\lambda,\cdots))]_G(\vu),[T_{\mathsf{FWL}(2)}(T_{\mathsf{SP}(2)}(\chi))]_G(\vu)).
    \end{equation}
    \item \textbf{Spectral pooling.} Define $T_{\mathsf{SP}(k)}:\gM_k^{\gP^\mM}\to\gM_k$ such that for any color mapping $\chi\in\gM_k^{\gP^\mM}$ and $G^{\vu}\in\gG_k$,
    \begin{equation}
    \label{eq:spectral_pooling_k}
        [T_{\mathsf{SP}(k)}(\chi)]_G(\vu)=\hash(\ldblbrace\chi_G(\lambda,\vu):\lambda\in\Lambda^\mM(G)\rdblbrace).
    \end{equation}
    \item \textbf{Joint pooling.} Define $T_{\mathsf{JP}(k)}:\gM_k\to\gM_0$ such that for any color mapping $\chi\in\gM_k$ and $G\in\gG$,
    \begin{equation}
    \label{eq:k-joint_pooling}
        [T_{\mathsf{JP}(k)}(\chi)](G)=\hash(\ldblbrace\chi_G(\vu):\vu\in V_G^k \rdblbrace).
    \end{equation}
\end{itemize}
Define the initial color mapping $\chi_k^{\gP^\mM}\in\gM_k^{\gP^\mM}$ such that for any graph $G$, vertices $\vu\in V_G^k$, and $\lambda\in \Lambda^\mM(G)$,
\begin{equation}
    [\chi_k^{\gP^\mM}]_G(\lambda,\vu)=\hash(\lambda,[\mP^\mM_\lambda]_G(u_1,u_1),\cdots,[\mP^\mM_\lambda]_G(u_1,u_k),\cdots,[\mP^\mM_\lambda]_G(u_k,u_1),\cdots,[\mP^\mM_\lambda]_G(u_k,u_k)),
\end{equation}
where $\mP^\mM_\lambda$ is the projection onto eigenspace associated with eigenvalue $\lambda$ for graph matrix $\mM$. The color refinement algorithm for Spectral PPGN is then defined as $(T_{\mathsf{JP}(2)}\circ T_{\mathsf{SP}(2)}\circ T_{\mathsf{SPPGN}}^\infty)(\chi_2^{\gP^\mM})$. Similarly, the color refinement algorithm for Spectral $k$-IGN is then defined as $(T_{\mathsf{JP}(k)}\circ T_{\mathsf{SP}(k)}\circ T_{\mathsf{SIGN}(k)}^\infty)(\chi_2^{\gP^\mM})$. We aim to prove that following two results:
\begin{proposition}
\label{thm:spectral_ppgn_2fwl}
    $(T_{\mathsf{JP}(2)}\circ T_{\mathsf{SP}(2)}\circ T_{\mathsf{SPPGN}}^\infty)(\chi_2^{\gP^\mM})\equiv (T_{\mathsf{JP}(2)}\circ  T_{\mathsf{FWL}(2)}^\infty)(\chi^0_2)$.
\end{proposition}
\begin{proposition}
\label{thm:spectral_kign_kwl}
    $(T_{\mathsf{JP}(k)}\circ T_{\mathsf{SP}(k)}\circ T_{\mathsf{SIGN}(k)}^\infty)(\chi_k^{\gP^\mM})\equiv (T_{\mathsf{JP}(k)}\circ  T_{\mathsf{WL}(k)}^\infty)(\chi^0_k)$.
\end{proposition}
We will only prove \cref{thm:spectral_ppgn_2fwl}, as the proof of \cref{thm:spectral_kign_kwl} is almost the same.
\begin{proof}[Proof of \cref{thm:spectral_ppgn_2fwl}]
    We first prove that $(T_{\mathsf{JP}(2)}\circ T_{\mathsf{SP}(2)}\circ T_{\mathsf{SPPGN}}^\infty)(\chi_2^{\gP^\mM})\preceq (T_{\mathsf{JP}(2)}\circ  T_{\mathsf{FWL}(2)}^\infty)(\chi^0_2)$. Since $T_{\mathsf{SP}(2)}(\chi_2^{\gP^\mM}) \preceq \chi^0_2$ (\cref{thm:epwl_atp}), it suffices to prove that $T_{\mathsf{SP}(2)}\circ T_{\mathsf{SPPGN}}^\infty\preceq T_{\mathsf{FWL}(2)}^\infty\circ T_{\mathsf{SP}(2)}$. Based on \cref{thm:refinement2}, it suffices to prove that $T_{\mathsf{FWL}(2)}\circ T_{\mathsf{SP}(2)}\circ T_{\mathsf{SPPGN}}^\infty\equiv T_{\mathsf{SP}(2)}\circ T_{\mathsf{SPPGN}}^\infty$. Denote $\chi=T_{\mathsf{SPPGN}}^\infty(\hat\chi^0)$, where $\hat\chi^0\in\gM_2^{\gP^\mM}$ is any initial color mapping. Pick any graphs $G,H\in\gG$ and vertices $u,v\in V_G$, $x,y\in V_H$ such that $[T_{\mathsf{SP}(2)}(\chi)]_G(u,v)=[T_{\mathsf{SP}(2)}(\chi)]_H(x,y)$. We have:
    \begin{align*}
        &\quad[T_{\mathsf{SP}(2)}(\chi)]_G(u,v)=[T_{\mathsf{SP}(2)}(\chi)]_H(x,y)\\
        &\implies \ldblbrace\chi_G(\lambda,u,v):\lambda\in \Lambda^\mM(G)\rdblbrace = \ldblbrace\chi_H(\mu,x,y):\mu\in \Lambda^\mM(H)\rdblbrace\\
        &\implies \ldblbrace[(T_{\mathsf{FWL}(2)}\circ T_{\mathsf{SP}(2)})(\chi)]_G(u,v):\lambda\in \Lambda^\mM(G)\rdblbrace = \ldblbrace[(T_{\mathsf{FWL}(2)}\circ T_{\mathsf{SP}(2)})(\chi)]_H(x,y):\mu\in \Lambda^\mM(H)\rdblbrace\\
        &\implies [(T_{\mathsf{FWL}(2)}\circ T_{\mathsf{SP}(2)})(\chi)]_G(u,v)=[(T_{\mathsf{FWL}(2)}\circ T_{\mathsf{SP}(2)})(\chi)]_H(x,y).
    \end{align*}
    where in the second step we use the definition of $T_{\mathsf{SPPGN}}$ and the fact that $\chi\equiv T_{\mathsf{SPPGN}}(\chi)$.
    
    We next prove that $(T_{\mathsf{JP}(2)}\circ  T_{\mathsf{FWL}(2)}^\infty)(\chi^0_2)\preceq (T_{\mathsf{JP}(2)}\circ T_{\mathsf{SP}(2)}\circ T_{\mathsf{SPPGN}}^\infty)(\chi_2^{\gP^\mM})$. Based on \citet{zhang2023complete}, we have $T_{\mathsf{FWL}(2)}^\infty\preceq T_{\mathsf{PS}}^\infty\circ T_{\mathsf{NM}}$ and $T_{\mathsf{PS}}^\infty\equiv T_{\mathsf{Du}}\circ T_{\mathsf{Dv}} \circ T_{\mathsf{PS}}^\infty$, where $T_{\mathsf{PS}}$, $T_{\mathsf{NM}}, T_{\mathsf{Du}}, T_{\mathsf{Dv}}$ are defined in \cref{sec:proof_epwl_sswl}. On the other hand, we have proved in \cref{thm:ratten_corollary} that $(T_{\mathsf{PS}}^\infty\circ T_{\mathsf{NM}})(\chi^0_2)\preceq \gP^\mM$. Therefore, $(T_{\mathsf{PS}}^\infty\circ T_{\mathsf{NM}})(\chi^0_2)\preceq (T_{\mathsf{Du}}\circ T_{\mathsf{Dv}})(\gP^\mM)\preceq T_{\mathsf{SP}(2)}(\chi_2^{\gP^\mM})$. This finally implies that $T_{\mathsf{FWL}(2)}^\infty(\chi^0_2)\preceq T_{\mathsf{SP}(2)}(\chi_2^{\gP^\mM})$.

    It thus suffices to prove that $(T_{\mathsf{FWL}(2)}^\infty\circ T_{\mathsf{SP}(2)})(\chi_2^{\gP^\mM})\preceq (T_{\mathsf{SP}(2)}\circ T_{\mathsf{SPPGN}}^\infty)(\chi_2^{\gP^\mM})$. We will prove the following stronger result: for any $t\ge 0$, $(T_{\mathsf{FWL}(2)}^\infty\circ T_{\mathsf{SP}(2)})(\chi_2^{\gP^\mM}))\preceq (T_{\mathsf{SP}(2)}\circ T_{\mathsf{SPPGN}}^t)(\chi_2^{\gP^\mM})$. The proof is based on induction. For the base case of $t=0$, the result clearly holds. Now assume that the result holds for $t=t'$ and consider the case of $t=t'+1$. Denote $\chi=T_{\mathsf{SPPGN}}^t(\chi_2^{\gP^\mM})$. Pick any graphs $G,H\in\gG$ and vertices $u,v\in V_G$, $x,y\in V_H$. We have
    \begin{align*}
        &\quad[(T_{\mathsf{FWL}(2)}^\infty\circ T_{\mathsf{SP}(2)})(\chi_2^{\gP^\mM}))]_G(u,v)=[(T_{\mathsf{FWL}(2)}^\infty\circ T_{\mathsf{SP}(2)})(\chi_2^{\gP^\mM}))]_H(x,y)\\
        &\implies[(T_{\mathsf{FWL}(2)}\circ T_{\mathsf{FWL}(2)}^\infty\circ T_{\mathsf{SP}(2)})(\chi_2^{\gP^\mM}))]_G(u,v)=[(T_{\mathsf{FWL}(2)}\circ T_{\mathsf{FWL}(2)}^\infty\circ T_{\mathsf{SP}(2)})(\chi_2^{\gP^\mM}))]_H(x,y)\\
        &\implies[(T_{\mathsf{FWL}(2)}\circ T_{\mathsf{SP}(2)})(\chi)]_G(u,v)=[(T_{\mathsf{FWL}(2)}\circ T_{\mathsf{SP}(2)})(\chi)]_H(x,y)\\
        &\implies \Lambda^\mM(G)=\Lambda^\mM(H) \land [(T_{\mathsf{FWL}(2)}(\chi(\lambda,\cdots))]_G(u,v)=[(T_{\mathsf{FWL}(2)}(\chi(\lambda,\cdots))]_H(x,y)\quad \forall \lambda\in \Lambda^\mM(G).
    \end{align*}
    Combining the last two steps implies that $[(T_{\mathsf{SP}(2)}\circ T_{\mathsf{SPPGN}})(\chi)]_G(u,v)=[(T_{\mathsf{SP}(2)}\circ T_{\mathsf{SPPGN}})(\chi)]_H(x,y)$. We thus conclude the induction step.
\end{proof}

\subsection{Counterexamples}
\label{sec:proof_counterexample}

In this subsection, we aim to reveal the expressivity gaps between different GNN models. This is achieved by constructing a pair of counterexample graphs $G,H\in\gG$ such that one GNN can distinguish while the other cannot. Here, we will leverage an important theoretical tool to construct counterexamples, known as F{\"u}rer graphs \citet{furer2001weisfeiler}. An in-depth introduction of F{\"u}rer graphs can be found in \citet{zhang2023complete}.

\begin{definition}[F{\"u}rer graphs]
    Given any connected graph $F=(V_F,E_F)$, the F{\"u}rer graph $G(F)=(V_{G(F)},E_{G(F)})$ is constructed as follows:
    \begin{align*}
        &V_{G(F)}=\{(x,X):x\in V_F,X\subset N_F(x),|X|\bmod 2 = 0\},\\
        &E_{G(F)}=\{\{(x,X),(y,Y)\}\subset V_G:\{x,y\}\in E_F,(x\in Y\leftrightarrow y\in X)\}.
    \end{align*}
    Here, $x\in Y\leftrightarrow y\in X$ holds when either ($x\in Y$ and $y\in X$) or ($x\notin Y$ and $y\notin X$) holds. For each $x\in V_F$, denote the set
    \begin{equation}
        \meta_F(x):=\{(x,X):X\subset N_F(x),|X|\bmod 2 = 0\},
    \end{equation}
    which is called the meta vertices of $G(F)$ associated to $x$. Note that $V_{G(F)}=\bigcup_{x\in V_F}\meta_F(x)$.
\end{definition}

We next define an operation called ``twist'':
\begin{definition}[Twisted F{\"u}rer graphs]
    Let $G(F)=(V_{G(F)},E_{G(F)})$ be the F{\"u}rer graph of $F=(V_F,E_F)$, and let $\{x,y\}\in E_F$ be an edge of $F$. The \emph{twisted} F{\"u}rer graph of $G(F)$ for edge $\{x,y\}$, is constructed as follows: $\twist(G(F),\{x,y\}):=(V_{G(F)},E_{\twist(G(F),\{x,y\})})$, where
    \begin{align*}
        E_{\twist(G(F),\{x,y\})}:=E_{G(F)}\triangle\{\{\xi,\eta\}:\xi\in\meta_F(x),\eta\in\meta_F(y)\},
    \end{align*}
    and $\triangle$ is the symmetric difference operator, i.e., $A\triangle B=(A\backslash B)\cup(B\backslash A)$. For an edge set $S=\{e_1,\cdots,e_k\}\subset E_F$, we further define
    \begin{align}
    \label{eq:twist}
        \twist(G(F), S):=\twist(\cdots\twist(G(F),e_1)\cdots,e_k).
    \end{align}
    Note that \cref{eq:twist} is well-defined as the resulting graph does not depend on the order of edges $e_1,\cdots,e_k$ for twisting.
\end{definition}

The following result is well-known \citep[see e.g., ][Corollary I.5 and Lemma I.7]{zhang2023complete}):
\begin{theorem}
    For any graph $F$ and any set $S_1,S_2\subset E_F$, $\twist(G(F), S_1)\simeq \twist(G(F), S_2)$ iff $|S_1|\bmod 2 = |S_2|\bmod 2$.
\end{theorem}

Below, we will prove \cref{thm:counterexamples} using F{\"u}rer graphs. Note that it can be easily checked via a computer program whether a pair of graphs can be distinguished by a given color refinement algorithm. However, an in-depth understanding of why a given color refinement algorithm can/cannot distinguish these (twisted) F{\"u}rer graphs is beyond the scope of this paper and is left for future work.
\begin{proposition}
\label{thm:counterexamples}
    The following hold:
    \begin{enumerate}[label=\alph*),topsep=0pt,leftmargin=20pt]
    \setlength{\itemsep}{0pt}
        \item There exists a pair of graphs $G,H$ such that SWL cannot distinguish them but Siamese IGN, Spectral IGN, and EPWL with any graph matrix $\mM\in\{\mA,\mL,\hat\mL\}$ can distinguish them;
        \item For any $\mM\in\{\mA,\mL,\hat\mL\}$, there exists a pair of graphs $G,H$ such that Weak Spectral IGN cannot distinguish them but Spectral IGN can distinguish them;
        \item There exists a pair of graphs $G,H$ such that Weak Spectral IGN cannot distinguish them with any $\mM\in\{\mA,\mL,\hat\mL\}$, but GD-WL with any distance listed in \cref{sec:distance_gnn} can distinguish them;
        \item There exists a pair of graphs $G,H$ such that Spectral IGN and EPWL with any graph matrix $\mM\in\{\mA,\mL,\hat\mL\}$ cannot distinguish them, but SWL can distinguish them;
        \item There exists a pair of graphs $G,H$ such that Siamese IGN with any graph matrix $\mM\in\{\mA,\mL,\hat\mL\}$ cannot distinguish them, but Weak Spectral IGN with any graph matrix $\mM\in\{\mA,\mL,\hat\mL\}$ can distinguish them.
    \end{enumerate}
\end{proposition}
\begin{figure}[t]
  % \vspace{-20pt}
  \begin{center}
  \begin{tabular}{ccccc}
      \includegraphics[height=0.1\textwidth]{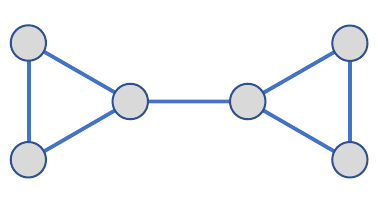} & \includegraphics[height=0.1\textwidth]{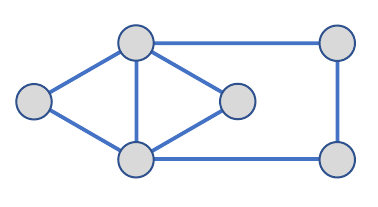} & \includegraphics[height=0.1\textwidth]{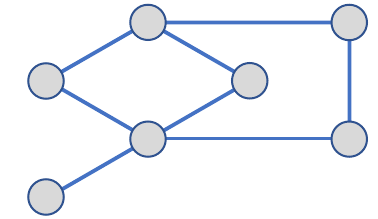} & \includegraphics[height=0.1\textwidth]{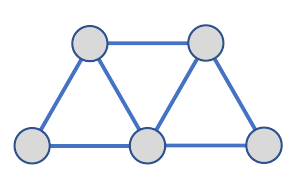} & \includegraphics[height=0.1\textwidth]{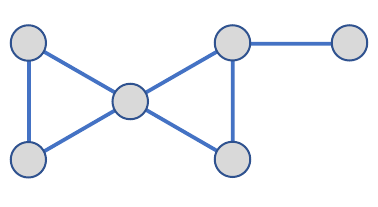} \\
      (a) & (b) & (c) & (d) & (e)
  \end{tabular}
  \end{center}
  \caption{Illustrations of base graphs used to construct F{\"u}rer graph and twisted F{\"u}rer graph for proving \cref{thm:counterexamples}.}
  \label{fig:counterexample}
  % \vspace{-20pt}
\end{figure}
\begin{proof}
    For \cref{thm:counterexamples}(a,b,c,d,e), the counterexample graphs are the F{\"u}rer graph and twisted F{\"u}rer graph for base graph $F$ defined in \cref{fig:counterexample}(a,b,c,d,e), respectively.
\end{proof}

\end{document}